%% file: main.tex
\documentclass[a4paper,10pt,reqno]{amsart}
\usepackage{times}
\usepackage{graphicx} 
\usepackage[shortlabels]{enumitem}
\usepackage{microtype}
\usepackage{subcaption}
\usepackage{booktabs} 
\usepackage{amsthm}
\usepackage{hyperref}
\usepackage{float}
\usepackage{url}
\usepackage{amsmath}
\usepackage{amssymb}
\usepackage{mathtools}
\usepackage{physics} 
\usepackage{xcolor}
\usepackage{bbm}
\usepackage{soul}
\usepackage{multirow}
\usepackage{adjustbox}
\usepackage[algo2e,ruled,linesnumbered]{algorithm2e}
\usepackage[capitalize]{cleveref}
\usepackage{cases} 
\usepackage{natbib}
\usepackage{tabularray}
\usepackage[foot]{amsaddr}

\newtheorem{corollary}{Corollary}
\newtheorem{definition}{Definition}
\newtheorem{proposition}{Proposition}
\newtheorem{lemma}{Lemma}
\newtheorem*{lemma*}{Lemma}
\newtheorem{remark}{Remark}

\input{math_commands}

\def\addcontentsline#1#2#3{}

\title{Accelerated regularized Wasserstein proximal sampling algorithms}
\author[Tan]{Hong Ye Tan${}^*$} 
\email{hyt35@math.ucla.edu}
\address[A1, A2]{Department of Mathematics, University of California, Los Angeles, CA 90095.}
\author[Osher]{Stanley Osher} 
\email{sjo@math.ucla.edu}
\author[Li]{Wuchen Li}
\email{wuchen@mailbox.sc.edu}
\address[A3]{Department of Mathematics, University of South Carolina, Columbia, SC 29208.}

\begin{document}

\begin{abstract}
    We consider sampling from a Gibbs distribution by evolving a finite number of particles using a particular score estimator rather than Brownian motion. To accelerate the particles, we consider a second-order score-based ODE, similar to Nesterov acceleration. In contrast to traditional kernel density score estimation, we use the recently proposed regularized Wasserstein proximal method, yielding the Accelerated Regularized Wasserstein Proximal method (ARWP). We provide a detailed analysis of continuous- and discrete-time non-asymptotic and asymptotic mixing rates for Gaussian initial and target distributions, using techniques from Euclidean acceleration and accelerated information gradients. Compared with the kinetic Langevin sampling algorithm, the proposed algorithm exhibits a higher contraction rate in the asymptotic time regime. Numerical experiments are conducted across various low-dimensional experiments, including multi-modal Gaussian mixtures and ill-conditioned Rosenbrock distributions. ARWP exhibits structured and convergent particles, accelerated discrete-time mixing, and faster tail exploration than the non-accelerated regularized Wasserstein proximal method and kinetic Langevin methods. Additionally, ARWP particles exhibit better generalization properties for some non-log-concave Bayesian neural network tasks.
\end{abstract}
\maketitle
\section{Introduction}
Let $V:\R^d \rightarrow \R$ be a known $\mathcal{C}^1$ potential function that typically satisfies some growth condition at infinity. The problem is to design an algorithm for sampling from target Gibbs distributions with densities 
\begin{equation*}
    \pi(x) \propto \exp(-\beta V(x)),
\end{equation*}
where $\beta>0$ is some diffusion/temperature parameter. Such tasks occur frequently in data science, such as uncertainty quantification and physical modeling \cite{akiyama2019first}, or more recently in generative modeling using diffusion models \cite{du2019implicit}.

Traditional methods, such as Markov chain Monte Carlo (MCMC) methods, apply Markov chains with an invariant distribution $\pi$. MCMC methods usually arise from discretizations of stochastic differential equations (SDEs),
which evolves a density according to a Fokker--Planck equation
\begin{equation}\label{eq:FokkerPlanck}
    \frac{\partial \rho}{\partial t} = \nabla \cdot (\nabla V(x)\rho) + \beta^{-1} \Delta \rho,\quad \rho(x,0) = \rho_0(x).
\end{equation}
Standard examples include the unadjusted Langevin algorithm (ULA) and the Metropolis-adjusted Langevin algorithm (MALA) \cite{rossky1978brownian,durmus2019high}. These algorithms arise from particular discretizations or approximations of the (overdamped) Langevin equation
\begin{equation}\label{eq:Langevin}
    \dd{X}_t = -\nabla V(X) \dd{t} + \sqrt{2\beta^{-1}} \dd{W}.
\end{equation}
It can be shown that the density of particles evolving under \labelcref{eq:Langevin} satisfies the Fokker--Planck equation \labelcref{eq:FokkerPlanck}. 
Traditionally, accelerating overdamped Langevin dynamics yields the underdamped Langevin equation, sometimes known as kinetic Langevin dynamics. While the invariant distribution of the overdamped Langevin equation is the target Gibbs distribution $\pi$, the invariant distribution of the underdamped dynamics is a separable joint density in a position-momentum phase space. Some methods arising from discretizing the underdamped Langevin dynamics include the variational acceleration flow \cite{chen2025accelerating}, inertial Langevin algorithm \cite{falk2025inertial}, kinetic Langevin Monte Carlo \cite{dalalyan2020sampling,cheng2018underdamped}, or some splitting methods, such as the OBA or BAOAB methods \cite{leimkuhler2013rational,leimkuhler2024contraction}. Generalizations of kinetic Langevin dynamics include the Hessian-free high resolution dynamics \cite{li2022hessian}, and primal-dual damping stochastic dynamics \cite{zuo2025gradient}. Convergence results for the kinetic Langevin equation are well studied, including non-asymptotic Wasserstein-2 contraction in continuous time and under different time discretizations \cite{dalalyan2020sampling,leimkuhler2024contraction,cheng2018underdamped,ma2021there}. Recently, \cite{cao2023explicit} demonstrates $L^2$ convergence under a Poincar\'e inequality, and provides an optimal friction coefficient. Related flows for sampling from phase-space include Hamiltonian Monte Carlo methods \cite{chen2020fast}.

As opposed to using a discretized SDE, another sampling paradigm approximates the Fokker--Planck equation by evolving a finite collection of particles through the score-based ODE
\begin{equation}\label{eq:scoreBasedUpdate}
    \frac{\dd{X}_t}{\dd{t}} = -\nabla V(X) - \beta^{-1} \nabla \log \rho_t(X),
\end{equation}
where $\rho_t$ is the density of $X_t$ at time $t$. From the continuity equation, the density of particles evolving according to this ODE \labelcref{eq:scoreBasedUpdate} also follows the Fokker--Planck equation \labelcref{eq:FokkerPlanck}. However, the score function $\nabla \log \rho_t$ is often intractable and requires estimation from empirical distributions. To derive algorithms that work with finitely many particles, \labelcref{eq:scoreBasedUpdate} needs to be modified with kernel functions. Examples of score-based sampling methods using kernels include Stein variational gradient descent (SVGD) \cite{liu2016stein}, which performs steepest descent with respect to the KL divergence with respect to a Wasserstein-type metric structure on the space of distributions induced by the Stein operator with a kernel function. Another example is the blob method \cite{carrillo2019blob,craig2016blob}, which considers Wasserstein gradient flows with particular kernel regularizations of the energy functional.

To accelerate the score-based flow \labelcref{eq:scoreBasedUpdate}, one may consider adding a momentum variable, similar to Nesterov acceleration methods for classical optimization problems \cite{nesterov1983method,su2016differential}. 
One possible acceleration to sample from a distribution comes from a particular Hamiltonian evolution \cite{wang2022accelerated,taghvaei2019accelerated,chen2025accelerating,carrillo2019convergence}. In the particular case of the Wasserstein-$2$ metric, the accelerated flow to minimize the KL divergence can be written as coupled ODEs in the density $\rho_t$ and its momentum variable. 
By adding a damping term into the momentum equation, \cite{wang2022accelerated,taghvaei2019accelerated} derive the \emph{accelerated information gradient flow}. 
The particle evolutions take the following form, which can be viewed as a second-order dynamics of the original score-based ODE \labelcref{eq:scoreBasedUpdate}: 
\begin{equation} \label{eq:AIG}
    \frac{\dd{}}{\dd{t}} \begin{bmatrix}
        X \\
        P
    \end{bmatrix} =  \begin{bmatrix}
        P \\
        -aP-\nabla V(X) - \nabla \log \rho_t(X)
    \end{bmatrix}.
\end{equation}
We study an equation-level modification of equation \labelcref{eq:AIG}. In particular, following \cite{tan2024noise}, we approximate the intractable score $\nabla \log \rho_t$ with a tractable approximation $\nabla \log \wprox \rho_t$, where $\wprox$ is the \emph{regularized Wasserstein proximal operator} (RWPO), defined in \Cref{def:RegWassProx} in the following section as the solution of a mean field control problem. This allows us to compute the score approximation without relying on selecting an appropriate kernel, such as the Gaussian kernel density estimator used in \cite{taghvaei2019accelerated,wang2022accelerated} to approximate $\log\rho_t$. The algorithm takes the form of a particular discretization of the particle evolution
\begin{equation} \label{eq:wAIG}
    \frac{\dd{}}{\dd{t}} \begin{bmatrix}
        X \\
        P
    \end{bmatrix} =  \begin{bmatrix}
        P \\
        -aP-\nabla V(X) - \nabla \log \wprox \rho_t(X)
    \end{bmatrix}.
\end{equation}



\subsection{Contributions}
Focusing on the accelerated flow \labelcref{eq:wAIG}, this work is organized as follows:
\begin{itemize}
    \item \Cref{sec:RegWassProx} details the regularized Wasserstein proximal operator of \cite{li2023kernel}, and its links to regularizing the Fokker--Planck equation \labelcref{eq:FokkerPlanck}. Moreover, we recall the space-varying kernel representation of this operator, which will be used in future computations.
    \item \Cref{sec:ARWP} introduces the \emph{accelerated regularized Wasserstein proximal (ARWP) method}. By performing a particular symplectic time discretization of the flow \labelcref{eq:wAIG}, we derive a discrete-time interacting particle algorithm to evolve the positions and velocities using the RWPO.
    \item Using the Gaussian closure property of the RWPO, \Cref{sec:QuadraticConvergence} analyzes the convergence of the ARWP method in the case of Gaussian initial and target distributions. Using closed-form updates, we provide an asymptotic continuous and discrete-time mixing analysis through linearization. Moreover, a detailed Lyapunov analysis demonstrates convergence in continuous time, where the damping parameter is sufficiently large to satisfy standard assumptions (up to modification by the regularization parameter).
    \item \Cref{sec:experiments} verifies the convergence analysis, and tests the ARWP method against various baselines. This includes a Rosenbrock distribution to identify tail exploration, a multimodal Gaussian mixture to test mixing times, and a Bayesian neural network example for simulations in higher dimensions.
\end{itemize}
Additional background on accelerated probability flows, proofs of \Cref{sec:QuadraticConvergence}, and additional qualitative/quantitative results, including hyperparameter ablations and choices, are also provided in the appendix.

\section{Regularized Wasserstein Proximal}\label{sec:RegWassProx}
We begin with the definition of the Wasserstein-2 distance and the Wasserstein proximal. 
\begin{definition}[{\cite{santambrogio2015optimal,ambrosio2005gradient}}]
    Let $\gP_2(\R^d)$ be the set of probability densities with finite second moment. For $\mu,\nu \in \gP_2(\R^d)$, the \emph{Wasserstein-2} distance $\gW_2(\mu, \nu)$ is
    \begin{equation}
        \gW_2(\mu,\nu) = \inf_{\pi \in \Gamma(\mu,\nu)} \int \|x-y\|^2 \dd\pi({x}, {y}),
    \end{equation}
    where the infimum is taken over couplings $\pi \in \Gamma(\mu,\nu)$, i.e. probability measures on $\R^d \times \R^d$ satisfying
    \begin{equation}
        \int_{\R^d} \pi(x,y) \dd{y} = \mu(x),\, \int_{\R^d} \pi(x,y) \dd{x} = \nu(y).
    \end{equation}

    Consider a probability density $\rho_0 \in \gP_2(\R^d)$ and $V \in \mathcal{C}^1(\R^d)$ be a lower-bounded potential function. For a scalar $T>0$, the \emph{Wasserstein proximal} of $\rho_0$ is defined as 
    \begin{equation}\label{eq:wassproxdef}
        \wprox_{T, V}(\rho_0) \coloneqq \argmin_{q\in \mathcal{P}_2(\R^d)} \int_{\R^d} V(x)q(x) \dd{x} + \frac{\gW(\rho_0, q)^2}{2T}.
    \end{equation}
\end{definition}

A recent line of work considers a principled score estimator called the \emph{backwards regularized Wasserstein proximal} (BRWP) method \cite{tan2024noise}. This method utilizes the fact that the time-discretized score-based particle update \labelcref{eq:scoreBasedUpdate} corresponds to a time evolution of a modified Fokker--Planck equation, which can be computed using a kernel formula. This has been extended to utilize tensor train score approximations \cite{han2024tensor} and splitting methods \cite{han2025splitting}. A recent followup work incorporates a preconditioning matrix into the underlying Fokker--Planck approximation, resulting in a modified mean-field control problem and different kernel \cite{tan2025preconditioned}. In this section, we recall the definition of the regularized Wasserstein proximal, which is used in the proposed scheme from discretizing \labelcref{eq:wAIG}.

The regularized Wasserstein proximal was first proposed as an approximation to a mean field control problem, obtained from the Wasserstein proximal through the Benamou--Brenier formulation \cite{li2023kernel,benamou2000computational}. In particular, the variational form \labelcref{eq:wassproxdef} has two equivalent formulations. One of them is the mean field control (MFC) formulation, where the Wasserstein proximal is given by the terminal time solution $\rho_T$ to the following MFC:
\begin{subequations}
\begin{gather}
    \inf_{\rho,v,q} \int_0^T \int_{\R^d} \frac{1}{2}\|v(t,x)\|^2 \rho(t,x) \dd{x} \dd{t} + \int_{\R^d} V(x)q(x) \dd{x}, \label{eq:WProxOriginalMFC} \\
    \partial_t \rho(t,x) + \nabla\cdot (\rho(t,x) v(t,x)) = 0, \quad \rho(0,x) = \rho_0(x),\, \rho(T,x) = q(x). \label{eq:WProxOriginalMFCDiffusive}
\end{gather}
\end{subequations}
The minimization \labelcref{eq:WProxOriginalMFC} is taken jointly with respect to a time-varying density function $\rho$, a velocity field $v$, and the terminal density function $q$. The conditions in \labelcref{eq:WProxOriginalMFCDiffusive} are a continuity equation and some boundary conditions, coupling the density and the velocity field. The Benamou--Brenier formulation then states that the solution to this minimization problem is equivalently given by the terminal time solution $\rho(T,x)$ of the following set of coupled PDEs,
\begin{equation}
    \begin{cases}
      \partial_t \rho(t,x) + \nabla_x \cdot\left(\rho(t,x) \nabla_x \Phi(t,x)\right) = 0, \\
      \partial_t \Phi(t,x) + \frac{1}{2} \|\nabla_x \Phi(t,x)\|^2 = 0, \\
      \rho(0,x) = \rho_0(x),\quad \Phi(T,x) = -V(x).
    \end{cases}
\end{equation}
We can now define the regularized Wasserstein proximal by adding some appropriate diffusive terms to the MFC, with an equivalent definition in terms of a terminal time solution to some coupled PDEs. 
\begin{definition}\label{def:RegWassProx}
For a probability distribution $\rho_0 \in \gP_2(\R^d)$ and a diffusion parameter $\beta>0$, the \emph{regularized Wasserstein proximal} $\wprox_{T,V}(\rho_0)$ is given by the solution to the following MFC problem: 
\begin{subequations}
\begin{gather}
    \inf_{\rho,v,q} \int_0^T \int_{\R^d} \frac{1}{2}\|v(t,x)\|^2 \rho(t,x) \dd{x} \dd{t} + \int_{\R^d} V(x)q(x) \dd{x}, \label{eq:NWProxOriginalMFC} \\
    \partial_t \rho(t,x) + \nabla\cdot (\rho(t,x) v(t,x)) = \beta^{-1} \Delta \rho(t,x), \notag \\ \rho(0,x) = \rho_0(x),\, \rho(T,x) = q(x), \label{eq:NWProxC}
\end{gather}
\end{subequations}
The RWPO is defined as follows:
\begin{equation*}
    \wprox_{T,V}\rho_0 \coloneqq \rho_T.
\end{equation*}
The above MFC problem \labelcref{eq:NWProxOriginalMFC} is a modification of \labelcref{eq:WProxOriginalMFC}, where the constraint \labelcref{eq:NWProxC}
is with a Laplacian term in the continuity equation \labelcref{eq:WProxOriginalMFCDiffusive}.

\end{definition}
Solving the optimality conditions yields that the regularized Wasserstein proximal satisfies a similar regularized Benamou--Brenier formulation, given by Laplacian regularization in both the forward-time Fokker--Planck and backward-time Hamilton--Jacobi equations,

\begin{subequations}\label{eqs:regPDE}
    \begin{numcases}{}
      \partial_t \rho(t,x) + \nabla_x \cdot\left(\rho(t,x) \nabla_x \Phi(t,x)\right) = \beta^{-1} \Delta_x \rho(t,x), \label{eq:regPDE_a} \\
      \partial_t \Phi(t,x) + \frac{1}{2} \|\nabla_x \Phi(t,x)\|^2 = -\beta^{-1} \Delta_x \Phi(t,x), \label{eq:regPDE_b}\\
      \rho(0,x) = \rho_0(x),\quad \Phi(T,x) = -V(x).\label{eq:regPDE_c}
    \end{numcases}
\end{subequations}
Using a particular Cole--Hopf formula connects these coupled PDEs with a set of forward-backward coupled heat equations, which has an exact solution based on a kernel formula:
\begin{subequations}\label{eqs:kernelBRWP}
    \begin{gather}
    \rho_T(x) = \rho(T,x) = \int_{\R^d} K(x,y) \rho_0(y) \dd{y}, \label{eq:rhoT}\\
    K(x,y) = \frac{\exp(-\frac{\beta}{2} (V(x) + \frac{\|x-y\|^2}{2T}))}{\int_{\R^d} \exp(-\frac{\beta}{2} (V(z) + \frac{\|z-y\|^2}{2T})) \dd{z}}.\label{eq:kernelDef}
\end{gather}
\end{subequations}

Motivated by the approximation of the Wasserstein proximal and the iterative component of the JKO scheme \cite{jordan1998variational}, \cite{tan2024noise} propose the \textit{Backwards Regularized Wasserstein Proximal} (BRWP) method. This is a semi-implicit discretization of the regularized Fokker--Planck equation. In particular, by the continuity equation, the Fokker--Planck equation \labelcref{eq:regPDE_a}
\begin{equation}\label{eq:RegWassProx}
\partial_t \rho(t,x) + \nabla \cdot\left(\rho(t,x)\nabla \Phi(t, x)- \beta^{-1}  \rho \nabla (\log \rho)(t,x)\right)=0,
\end{equation}
corresponds to particles evolving as
\begin{equation}\label{eq:brwpParticleCts}
    \frac{\dd X}{\dd t} = \nabla \Phi(t,X) - \beta^{-1} \nabla  \log \rho(t, X).
\end{equation}
By discretizing the particle-based updates \labelcref{eq:brwpParticleCts} using the backward Euler method, the dual function $\Phi$ simply becomes an update in $V$ due to the boundary conditions of the MFC \labelcref{eq:regPDE_c}. 
Moreover, the score term $\nabla \log \rho(T,x)$ is precisely given by the score of the regularized Wasserstein proximal, which is computable for an empirical distribution given by a collection of particles. The BRWP iterations can be written as
\begin{equation}\label{eq:brwp}
    X_{k+1} = X_k -\eta\Big(\nabla V(X_k) +\beta^{-1} \nabla \log \wprox_{T,V} \rho_k(X_k)\Big).
\end{equation}
This allows for an adaptive kernel-based modification of the score-based update \labelcref{eq:scoreBasedUpdate}, replacing an arbitrary choice of kernel function (such as Gaussian or Mat\'ern), with a choice of regularization parameter $T$.

\section{Accelerated Regularized Wasserstein Proximal Method}\label{sec:ARWP}

The BRWP method replaces the log density $\log \rho$ with the log density of the regularized Wasserstein proximal $\log \wprox_{T,V}\rho$, and provides a discrete time update. We propose the \textit{accelerated regularized Wasserstein proximal} (ARWP) method in continuous and discrete time, which arises from using the approximation $\log \rho \approx \log\wprox_{T,V}\rho$ within \labelcref{eq:AIG}. For a particle with position $X$ and momentum $P$, we recall the iteration \labelcref{eq:wAIG}: 
\begin{subequations}\label{eq:arwp_cts}
    \begin{numcases}{}
        \frac{\dd{X}}{\dd{t}} = P, \label{eq:arwp_cts_dX}\\
        \frac{\dd{P}}{\dd{t}} = -aP - \nabla V(X) - \beta^{-1} \nabla \log \wprox_{T,V}\rho(t,X), \label{eq:arwp_cts_dP}
    \end{numcases}
\end{subequations}
where $\rho(t,\cdot)$ denotes the distribution of particles at time $t$. The discrete time update for a finite set of particles is given as follows. In the finite particle setting with positions and momentums $\{(\rvx_i^{(k)}, \rvp_i^{(k)})\}_{i=1}^N$, the semi-implicit scheme is given by setting $\rho$ to be the empirical distribution at each iteration. The discrete-time ARWP algorithm for a step-size $\eta>0$ and possibly-varying damping parameters $a_k>0$ is defined by updating the particle positions and momenta using the symplectic (semi-implicit) Euler discretization:
\begin{equation}\label{eqs:symplecticARWP}
\begin{cases}
    \rvp_i^{(k+1)} = (1-\eta a_k)\rvp_i^{(k)} - \eta \nabla V(\rvx_i^{(k)}) - \eta \beta^{-1} \nabla \log \wprox_{T,V} \rho_k(\rvx_i^{(k)}),  \\
    \rvx_{i}^{(k+1)} = \rvx_i^{(k)} + \eta \rvp_i^{(k+1)}.
\end{cases}
\end{equation}

For each point, we can compute the score function of the RWPO of the empirical distribution $\rho_k = \frac{1}{N}\sum_{i=1}^N \delta(\rvx_i^{(k)})$ using the kernel formula \labelcref{eqs:kernelBRWP} \cite{tan2024noise}, where $\delta(\rvx)$ denotes the Dirac delta at the point $\rvx$. Temporarily dropping the iteration $k$ subscripts and superscripts, the RWPO of the empirical distribution can be computed at each point as follows:
\begin{equation*}
    \wprox_{T,V}\rho(\rvx_i) = \frac{1}{N}\sum_{j=1}^N K(\rvx_i, \rvx_j) 
    = \frac{1}{N}\sum_{j=1}^N \frac{\exp \left[-\frac{\beta}{2}\left(V(\rvx_i) + \frac{\|\rvx_i - \rvx_j\|^2}{2T}\right)\right]}{\gZ(\rvx_j)},
\end{equation*}
\begin{equation*}
    \nabla \wprox_{T,V}\rho(\rvx_i) = \frac{1}{N}\sum_{j=1}^N \frac{\left(-\frac{\beta}{2}\left(\nabla V(\rvx_i) + \frac{\rvx_i - \rvx_j}{T}\right)\right)\exp \left[-\frac{\beta}{2}\left(V(\rvx_i) + \frac{\|\rvx_i - \rvx_j\|^2}{2T}\right)\right]}{\gZ(\rvx_j)},
\end{equation*}
\begin{equation}\label{eq:normalizingConstant}
    \gZ(\rvx_j) \coloneqq \int_{\R^d} e^{-\frac{\beta}{2} (V(z) + \frac{\|z-\rvx_j\|^2}{2T})} \dd{z}.
\end{equation}
Using these expressions, the score function of the RWPO of $\rho=\rho_k$ can be computed with the simple identity
\begin{equation*}
    \nabla \log \wprox_{T,V}\rho(\rvx_i) = \frac{\nabla \wprox_{T,V}\rho}{\wprox_{T,V}\rho}(\rvx_i).
\end{equation*}
\subsection{Computational Considerations}
A parallelization similar to \cite{tan2025preconditioned} may be employed by concatenating the position and momentum variables into a matrix, and utilizing the structure of $\wprox_{T,V}\rho$ as a sum of exponentials. Recall that the softmax function of a vector $v \in \R^N$ is defined as 
\begin{equation*}
    \softmax(v) = \left(\frac{\exp(v_i)}{\sum_{j=1}^N \exp(v_j)}\right)_{i=1,...,N},
\end{equation*}
satisfying $\sum_j \softmax(v)_j=1$. The score approximation $\nabla \log \wprox_{T,V} \rho_k$ may then be written in terms of a softmax matrix:
\begin{equation}\label{eq:wproxSoftmaxFormulation}
    \nabla \log \wprox_{T, V}\rho_{k}(\rvx_i) = -\frac{\beta \nabla V(\rvx_i)}{2}- \frac{\beta}{2 T} \rvx_i + \frac{\beta}{2 T} \sum_{j=1}^N \softmax(W_{i,\cdot})_j \rvx_j,
\end{equation}
where $W_{i,j}$ is an interaction matrix defined as 
\begin{equation}\label{eq:interactionMatrix}
    W_{i,j} \coloneqq -\beta \frac{\|\rvx_i - \rvx_j\|^2}{4T} - \log \gZ(\rvx_j).
\end{equation}
Notice that the reformulation \labelcref{eq:wproxSoftmaxFormulation} additionally contains a $\nabla V(\rvx_i)$ term within the score, which can be combined with the $\nabla V$ term within the momentum update \labelcref{eq:arwp_cts_dP}. In particular, the discrete-time momentum update \labelcref{eqs:symplecticARWP} can be rewritten as 
\begin{align}
    \rvp_i^{(k+1)} &= (1-a_k\eta) \rvp_i^{(k)} + \eta \left(- \nabla V(\rvx_i^{(k)}) - \beta^{-1} \nabla \log \wprox_{T, V}\rho_{k}(\rvx_i^{(k)})\right) \notag\\
    &= (1-a_k\eta) \rvp_i^{(k)} - \frac{\eta}{2}  \nabla V(\rvx_i^{(k)}) + \frac{\eta}{2T}\left( \sum_{j=1}^N \softmax(W_{i,\cdot}^{(k)})_j (\rvx_i^{(k)}-\rvx_j^{(k)})\right).\label{eq:PUpdateSoftmaxForm}
\end{align}

To compute the normalization constant \labelcref{eq:normalizingConstant}, we can use a Monte Carlo integral applied to 
\begin{equation}
    \gZ(\rvx_j) = (4 \pi T\beta^{-1})^{d/2}\E_{z \sim \gN(\rvx_j, 2T\beta^{-1})} \left[\exp\left(- \frac{\beta V(z)}{2}\right)\right].
\end{equation}
We note that the constant $(4 \pi T\beta^{-1})^{d/2}$ is canceled out by the logarithm and softmax operators, and can be ignored during computation. 
\begin{remark}
    An alternative to Monte Carlo integration in high dimensions for small $T$ is to use the Laplace approximation \cite{tan2025preconditioned,tibshirani2025laplace}. This reads 
    \begin{align*}
        \gZ(\rvx_j) &= \int_{\R^d} e^{-\frac{\beta}{2} (V(z) + \frac{\|z-\rvx_j\|^2}{2T})} \dd{z}\\
        &\approx e^{-\frac{\beta}{2} V(\rvx_j)} C(\beta, T),
    \end{align*}
    where $C(\beta,T)$ is a constant independent of $\rvx_i$ which also cancels out under the logarithm and the softmax operations. 
\end{remark}

To finish the parallelization, we combine the position and momentum vectors into matrices
\begin{align*}
    \tX = \begin{bmatrix}
        \rvx_1 & ... & \rvx_N
    \end{bmatrix} \in \R^{d \times N},\quad \tP = \begin{bmatrix}
        \rvp_1 & ... & \rvp_N
    \end{bmatrix} \in \R^{d \times N}.
\end{align*}
The ARWP update \labelcref{eqs:symplecticARWP} can be combined with \labelcref{eq:PUpdateSoftmaxForm} to be written in matrix form:
\begin{align*}
    \tP^{(k+1)} &= (1-a_k \eta) \tP^{(k)} - \frac{\eta}{2} \nabla V(\tX^{(k)}) + \frac{\eta}{2T}\left(\tX^{(k)} - \tX^{(k)} \softmax(W^{(k)})^\top\right),\\
    \tX^{(k+1)} &= \tX^{(k)} + \eta \tP^{(k+1)},
\end{align*}
where $W^{(k)}$ is the interaction matrix \labelcref{eq:interactionMatrix} at iteration $k$. This is summarized in \Cref{alg:ARWP}. 

\SetKwComment{Comment}{/* }{ */}
\begin{algorithm2e}[ht]
\caption{ARWP: Accelerated Regularized Wasserstein Proximal Method}\label{alg:ARWP}
\KwData{Initial points $\rvx_1^{(1)},...,\rvx_N^{(1)} \in \R^d$, potential $V: \R^d \rightarrow \R$, regularization parameter $T>0$, diffusion $\beta>0$, step-size $\eta>0$, iteration count $K$, damping parameters $a_k>0$.}
\KwResult{$\tX^{(K)} = \begin{bmatrix}
    \rvx_1^{(K)} &...& \rvx_N^{(K)}
\end{bmatrix}$ sampling from $\exp(-\beta V)$.}
Initialize $\tX^{(1)} = \begin{bmatrix}
    \rvx_1^{(1)} & ... & \rvx_N^{(1)} 
\end{bmatrix} \in \R^{d \times N}$, initialize $\tP^{(1)} = \mathbf{0}_{d \times N}$\;

\For{$k=1,...,K$}
{
Approximate normalizing constants $\gZ(\rvx_i^{(k)}),\, i=1,...,N$ using Monte Carlo/Laplace method\;
Compute interaction matrix $W_{i,j} = -\beta \frac{\|\rvx_i-\rvx_j\|^2}{4 T} - \log \gZ(\rvx_j)$\label{algstep:intermat}\;
Compute row-wise softmax interaction matrix $\softmax(W)_{i,j} = \softmax(W_{i,\cdot})_j$\;

Evolve momentum matrix $\tP^{(k+1)} = (1-a_k \eta)\tP^{(k)} - \frac{\eta}{2}  \nabla V(\tX^{(k)}) + \frac{\eta}{2T}\left(\tX^{(k)} - \tX^{(k)} \softmax(W^{(k)})^\top\right)$ \;

Evolve particle positions $\tX^{(k+1)} = \tX^{(k)} + \eta \tP^{(k+1)}$\;
}
\end{algorithm2e}

\begin{remark}
    This approximation is used in contrast to the kernel density estimation of \cite{wang2022accelerated} or ``diffusion map approximation'' in \cite{taghvaei2019accelerated}. From a computational perspective, one needs to approximate the log-score from an empirical distribution. The regularized Wasserstein proximal allows for this, interpretable as a potential-aware modified kernel method.
\end{remark}

The proposed ARWP method differs from the unregularized accelerated information gradient flow in its choice of score approximation. In particular, \cite{taghvaei2019accelerated,wang2022accelerated} both consider using some variant of Gaussian kernel density estimation, for which the interplay between the bias and convergence is not characterized. In the following section, we use properties of the RWPO operator to characterize the asymptotic and non-asymptotic convergence behavior for quadratic potentials. 

ARWP does not incur a significant computational increase over the non-accelerated BRWP \labelcref{eq:brwp}. The main computational cost requirement comes from the interaction term in \Cref{algstep:intermat}, which requires constructing (rows of) an $N \times N$ matrix, which is also required in BRWP. Since updating each particle requires only one call of $\nabla V$, having to track an additional momentum parameter per particle incurs only a constant memory factor increase, which can range from double to negligible depending on the level of parallelization employed.

\section{Convergence for Quadratic Potentials}\label{sec:QuadraticConvergence}
This section analyzes the convergence of the ARWP method, in the case of Gaussian distributions. \cite{tan2024noise} utilizes a closed-form update for the RWPO for Gaussian distributions, demonstrating that the update in the BRWP method with quadratic potential stays in Gaussian distributions. A similar argument shows that the ARWP method updates Gaussian distributions to Gaussian distributions. In other words, if the target distribution is Gaussian and the initial distribution is Gaussian, then the discrete-time particle updates \labelcref{eq:arwp_cts}'s density function $\rho_k$ also follows a Gaussian distribution. 

Fix a covariance matrix $\Lambda \in \mathrm{Sym}_{++}(\R^d)$, and define the quadratic potential function $V(x)= x^\top \Lambda^{-1} x/2$. We additionally assume that all matrices commute so that we may work in a common eigenbasis, and fix $\beta=1$ without loss of generality.

We show the closure within Gaussian distributions by considering the particle-wise update. Suppose that the distribution at iteration $k$ is Gaussian. We have the following lemmas characterizing the effect of the regularized Wasserstein proximal on a covariance matrix.
\begin{lemma}{{\cite{tan2024noise,tan2025preconditioned}}}
    For a covariance matrix $\Sigma$, if $T < \lambda_{\min}(\Lambda)$, then the regularized Wasserstein proximal of the Gaussian distribution $\gN(0, \Sigma)$ is also a Gaussian distribution $\gN(0, \tilde\Sigma)$, whose covariance takes the form:
    \begin{gather*}
        \wprox_{T,V}\gN(0, \Sigma) = \gN(0, \tilde\Sigma),\\
        \tilde\Sigma \coloneqq 2\beta^{-1} T\left(I + T\Lambda^{-1} \right)^{-1} + \left(I + T\Lambda^{-1} \right)^{-1} \Sigma \left(I +T\Lambda^{-1}\right)^{-1}.
    \end{gather*}
    Moreover, the inverse operator of the regularized Wasserstein proximal satisfies
    \begin{equation}
        \wprox_{T,V}^{-1}(\gN(0, \Sigma)) = \gN(0, (1+T\Lambda^{-1})\Sigma (1-T\Lambda^{-1})).
    \end{equation}
\end{lemma}

As a shorthand, we will use the tilde notation to denote the regularized Wasserstein proximal of a covariance matrix throughout. It is shown in \cite{tan2024noise} that the terminal distribution under the BRWP update is $\gN(0, \Sigma_\infty)$, where $\Sigma_\infty$ is such that $\wprox_{T,V}(\gN(0, \Sigma_\infty)) = \gN(0, \Lambda)$, i.e. $\tilde\Sigma_\infty = \Lambda$. Moreover, the bias of the stationary distribution under BRWP (and ARWP) depends only on $T$, in particular, is independent of the step-size $\eta$.

We will also find it convenient to define the following two matrix expressions $K_\pm$:
\begin{equation}\label{eq:KDef}
    K_+ \coloneqq I + T\Lambda^{-1},\quad K_- \coloneqq I - T\Lambda^{-1}.
\end{equation}

In this section, we analyze the accelerated backward-regularized Wasserstein proximal method for the special case of Gaussian distributions under different approximations of \labelcref{eq:arwp_cts}. 
\begin{itemize}
    \item \Cref{ssec:arwpCtsTime} converts the ARWP update \labelcref{eq:arwp_cts} into a pair of coupled ODEs in covariance and a dual term.
    
    \item \Cref{ssec:ctsTimeConvRate} treats the linearized continuous time case, showing the corresponding coupled ODEs observe an asymptotic $\mathcal{O}\left(t\exp(-\sqrt{2\lambda^{-1} \frac{1-T\lambda^{-1}}{1+T\lambda^{-1}}}t)\right)$ error convergence in each eigendirection, where $\lambda$ is the corresponding eigenvalue of the target covariance matrix $\Lambda$. 
    
    \item While the convergence rate of the corresponding ODE is slightly slower than the unregularized case $T=0$, \Cref{ssec:DiscreteTimeLinearization} shows that the step-size can be taken to be larger, resulting in a discrete-time iteration speedup by a constant factor of $\frac{1+\sqrt{2}}{2}$. This gives the asymptotic mixing rate of the proposed ARWP method.

    \item We compare the discrete-time rates with the kinetic Langevin algorithm in \Cref{ssec:kineticComparison}, demonstrating a faster rate arising from the regularization. 

    \item In the non-asymptotic case, we show in \Cref{ssec:NoLinNoAsympCts} that the continuous-time coupled ODEs converge linearly to the target distribution, and show that the damping condition $a>\lambda^{-1/2}$ is sufficient for convergence. This is done using a particular Lyapunov analysis, splitting the ODEs into underdamped and overdamped cases. This allows for the analysis in \Cref{ssec:ctsTimeConvRate} to apply over a large time.
\end{itemize}

\subsection{Continuous Time Covariance Update of ARWP}\label{ssec:arwpCtsTime}
Since the regularized Wasserstein proximal of a Gaussian distribution is a Gaussian distribution, the $\mathrm{d}P/\mathrm{d}{t}$ in \labelcref{eq:arwp_cts} is linear in $X$. Therefore, we may use the ansatz $P_t = G_tX_t$, where $G: \R_{\ge 0}  \times \R^d \rightarrow \R^d$ is a time varying linear map. Let $\Sigma_t$ denote the covariance of $X_t$. After a change of variables, the update \labelcref{eq:arwp_cts_dP} satisifies
\begin{equation*}
    \frac{\dd{P}}{\dd{t}} = -aP - \Lambda^{-1}X + \tilde\Sigma_t^{-1} X = \frac{\dd G}{\dd{t}}X + G \frac{\dd{X}}{\dd{t}} = \frac{\dd G}{\dd{t}}X + G^2X.
\end{equation*}
Moreover, the update \labelcref{eq:arwp_cts_dX} turns $\mathrm{d}X/\mathrm{d}{t} = P = GX$ into $\dot \Sigma_t = G_t \Sigma_t + \Sigma_t G_t$. Rearranging yields the following coupled ODE system
\begin{equation}\label{eqs:ctsCovUpdate}
\begin{cases}
    \dot \Sigma_t = G_t \Sigma_t + \Sigma_t G_t,\\
    \dot G_t = -aG_t - G_t^2 - \Lambda^{-1} + \tilde\Sigma_t^{-1},
\end{cases}
\end{equation}
where $\tilde\Sigma_t$ is the covariance of the regularized Wasserstein proximal applied to $\gN(0, \Sigma_t)$. Observe that $(\Sigma, G) = (\tilde\Sigma_\infty, \mathbf{0}_{d\times d})$ is a stationary point of \labelcref{eqs:ctsCovUpdate}, where $\Sigma_\infty = (1+T\Lambda^{-1})\Lambda (1-T\Lambda^{-1})$ is such that $\tilde\Sigma_\infty = \Lambda$. The following sections find convergence rates to this point.

\subsection{Continuous Time Asymptotic Convergence Rate}\label{ssec:ctsTimeConvRate}
To check the convergence behavior of \labelcref{eqs:ctsCovUpdate} near zero, we can linearize near the terminal state. As we assume that all covariances commute, let us work in one dimension, where our expressions are written in lower case. Then, the continuous time update in 1D is given by 
\begin{equation}\label{eqs:ctsCovUpdate1D}
\begin{cases}
    \dot \sigma_t = 2g_t \sigma_t,\\
    \dot g_t = -ag_t - g_t^2 - \lambda^{-1} + \tilde\sigma_t^{-1}.
\end{cases}
\end{equation}
Linearizing about the stationary point $(\sigma_\infty, 0)$, where $\tilde\sigma_\infty = \lambda$, let us consider the ansatz $\sigma_t = \sigma_\infty + \varepsilon_t$. The first order approximation to the $g_t$ update becomes
\begin{equation*}
    \lambda^{-1} - \tilde\sigma^{-1}_t = \lambda^{-2} (1+T\lambda^{-1})^{-2} \varepsilon_t+ \mathcal{O}(\varepsilon_t^2).
\end{equation*}
The linearized system (in phase space) near the stationary point becomes

\begin{equation}\label{eqs:epsilonDiscretizations}
\begin{cases}
    \dot\varepsilon_t = 2g_t (1+T\lambda^{-1})(1-T\lambda^{-1}) \lambda,\\
    \dot g_t = -ag_t - \lambda^{-2} (1+T\lambda^{-1})^{-2}\varepsilon.
\end{cases}
\end{equation}
The corresponding second order equation is
\begin{equation}
    \ddot{\varepsilon} = -a\dot\varepsilon - 2\lambda^{-1} k_+^{-1} k_- \varepsilon,
\end{equation}
where $k_\pm$ are defined as the one-dimensional counterparts of \labelcref{eq:KDef}. Using the ansatz $\varepsilon_t = e^{-rt}$, the rate satisfies
\begin{equation}\label{eq:ctsTimeLinearizedRate}
    r_{\pm} = \frac{a \pm \sqrt{a^2 - 8 \lambda^{-1} k_+^{-1} k_-}}{2}.
\end{equation}
The convergence rate is then given by the smaller root in case both roots are real, or the real part in case both roots are complex. In one dimension, the convergence rate is fastest when $a =  \sqrt{8\lambda^{-1} k_+^{-1} k_-}$. In this case, the rate is given by 
\begin{equation*}
    \|(\sigma_t, g_t) - (\sigma_\infty, 0)\|_2 = \mathcal{O}\left(t\exp(-\sqrt{2\lambda^{-1}k_+^{-1} k_-}t)\right).
\end{equation*}

The continuous time rate with regularization is thus slightly slower than the unregularized case $T=0$ by a factor of $\sqrt{k_+^{-1} k_-} = \sqrt{(1-T\lambda^{-1})/(1+T\lambda^{-1})}<1$. This analysis can be extended into multiple dimensions to find the asymptotic convergence rate of covariance in the trace norm.

\begin{proposition}\label{prop:ctsTimeLinearized}
    Let $V(x) = x^\top \Lambda^{-1}x/2$, where the smallest and largest eigenvalues of $\Lambda$ are $\lambda_{\min}, \lambda_{\max}$ respectively. Let $a>0$ be some damping parameter, $T \in [0,\lambda_{\min})$ be the regularization parameter, and suppose the initial distribution is $\gN(0, \Sigma_0)$. If $\Sigma_t$ is the continuous-time evolution of ARWP and $\Sigma_t$ converges to its stationary distribution's covariance $\Sigma_\infty = \wprox^{-1}_{T,V}(\Lambda)$, then the asymptotic convergence rate is 
    \begin{equation}
        \Tr(\Sigma_t - \Sigma_\infty) = \mathcal{O}(\exp(-rt)),
    \end{equation}
    where 
    \begin{equation}\label{eq:rate}
        r = \frac{1}{2}\left[a - \sqrt{\max_{\lambda \in [\lambda_{\min}, \lambda_{\max}]}\left(a^2 - 8\lambda^{-1} \frac{1-T\lambda^{-1}}{1+T\lambda^{-1}},0\right)}\right].
    \end{equation}
\end{proposition}
\begin{proof}
    The overall convergence rate is given by the smallest rate over each component. The rate corresponding to an eigenvalue $\lambda \in [\lambda_{\min}, \lambda_{\max}]$ is given from \labelcref{eq:ctsTimeLinearizedRate} as
    \begin{equation}\label{eqs:casesLinCts}
        \begin{cases}
            \frac{a}{2},& \text{ if } a^2 \le 8\lambda^{-1}\frac{1-T\lambda^{-1}}{1+T\lambda^{-1}};\\
            \frac{a - \sqrt{a^2 - 8 \lambda^{-1} k_+^{-1} k_-}}{2},& \text{ if } a^2 \ge  8\lambda^{-1}\frac{1-T\lambda^{-1}}{1+T\lambda^{-1}}.
        \end{cases}
    \end{equation}
    This is equivalent to \labelcref{eq:rate}.
\end{proof}

\begin{remark}
    Since the regularized Wasserstein proximal performs an affine transformation on the covariance matrix of a Gaussian distribution, the asymptotic convergence rate of the covariance $\Sigma$ and the RWPO covariances $\tilde\Sigma$ are identical.
\end{remark}

While this proposition appears to indicate that a smaller value of $a$ leads to better convergence rates, a necessary condition is that the evolution converges to the stationary point. In \Cref{ssec:NoLinNoAsympCts}, we show that this holds if $a > \lambda^{-1/2}$, similar to existing convergence results for the kinetic Langevin diffusion. This corresponds to requiring a sufficient amount of damping in order for the iterations to converge.

In the next section, we consider the discrete-time analog of the linearized system \labelcref{eqs:epsilonDiscretizations}. We show the main advantage of regularizing: we can take a larger step-size if $T$ is nonzero, which increases the discrete-time asymptotic mixing rate. This arises since the condition number of the regularized system is lower than that of the unregularized system.

\subsection{Linearized Discrete-Time Convergence Rate}\label{ssec:DiscreteTimeLinearization}
In this subsection, we consider the convergence of the one-dimensional RWPO covariances $\tilde\sigma_t$ to their stationary distributions $\gN(0, \lambda)$, similarly to \cite{tan2024noise,tan2025preconditioned}. Under the change of variables
\begin{equation*}
    \tilde\sigma_t = 2Tk_+^{-1} + k_+^{-2} \sigma_t,
\end{equation*}
as well as the ansatz $\tilde\sigma_t = \lambda + \delta_t$, the linearization of \labelcref{eqs:ctsCovUpdate1D} becomes
\begin{equation}\label{eq:ctsLinearization}
\frac{\dd{}}{\dd{t}}\begin{pmatrix}
    \delta_t \\
    g_t
\end{pmatrix} = 
\begin{pmatrix}
    [2 \lambda - 4Tk_+^{-1}] g_t\\
    -ag_t - \lambda^{-2} \delta_t
\end{pmatrix}
= \underbrace{\begin{bmatrix}
    0 & 2\lambda - 4Tk_+^{-1} \\
    -\lambda^{-2} & -a
\end{bmatrix}}_{\eqqcolon A}\begin{pmatrix}
    \delta_t \\
    g_t
\end{pmatrix}.
\end{equation}
Denoting the matrix as $A = A(\lambda, T, a)$, the eigenvalues $\chi_\pm$ of $A$ are given by 
\begin{align*}
    \chi_\pm &= \frac{1}{2}[\Tr \pm \sqrt{\Tr - 4\det}]\\
    &= \frac{1}{2}\left[-a \pm \sqrt{a^2 - 4 \lambda^{-2}(2\lambda - 4TK_+^{-1})}\right].
\end{align*}
In the continuous-time case, the stability condition near $(\delta, g)=(0,0)$ is that all eigenvalues have real component less than 0. This is the case for all $a>0$.

In the discrete case with step-size $\eta>0$, the (symplectic Euler) update becomes 
\begin{equation}\label{eq:DiscreteTimeLinearized}
    \begin{pmatrix}
        \delta_{n+1}\\
        g_{n+1}
    \end{pmatrix} = [I+\eta A] \begin{pmatrix}
        \delta_{n}\\
        g_{n}
    \end{pmatrix}.
\end{equation}
The stability condition for this update is that $I+\eta A$ has to have (both) eigenvalues lying in the open disk $\{|z|<1\mid z \in \mathbb{C}\}$ \cite{iserles2009first}. Moreover, the convergence rate in this direction is given by $\mathcal{O}(\max\{|1+\eta \chi_+|, |1+\eta \chi_-|\}^n)$. In particular, we have the following two special cases, corresponding to two different critical damping parameters. The step-size is controlled by the Lipschitz constant of $V$, which is $\lambda_{\min}^{-1}$. We demonstrate that the maximal step-size can be taken to be larger than the unregularized version, which corresponds to a faster discrete-time mixing rate, i.e., the rate at which $(\tilde\Sigma_k, G_k) \rightarrow (\Lambda,0)$.

\begin{proposition}\label{prop:LinearizedNonasymptotic}
Suppose the eigenvalues of the covariance matrix $\Lambda$ are $0<\lambda_{\min} \le ... \le \lambda_{\max}$. Let the update be given by the discrete time update \labelcref{eq:DiscreteTimeLinearized} with step-size $\eta >0$ and regularization $T \in [0, (1+\sqrt{2})^{-1}\lambda_{\min}]$. 
    \begin{enumerate}[(a)]
        \item If the momentum and step-size are chosen to be 
        \begin{equation}\label{eq:optimalDampingMax}
            a = 2\sqrt{2} \lambda_{\max}^{-1/2} \sqrt{\frac{\lambda_{\max} - T}{\lambda_{\max}+T}},\quad \eta  = \frac{\lambda_{\max}^{-1/2} \sqrt{\frac{\lambda_{\max} - T}{\lambda_{\max}+T}}}{\sqrt{2}\lambda_{\min}^{-1} {\frac{\lambda_{\min} - T}{\lambda_{\min}+T}}},
        \end{equation}
        then the mixing rate of the linearized system is
        \begin{equation}\label{eq:mixingrate}
            \sqrt{1 - \kappa^{-1} \frac{\lambda_{\max}-T}{\lambda_{\max}+T} \frac{\lambda_{\min}+T}{\lambda_{\min}-T}}.
        \end{equation}
        In particular, taking $T = (1+\sqrt{2})^{-1} \lambda_{\min}$, the discrete time rate for $\kappa \gg 1$ and the linearized system is (up to first order)
        \begin{equation*}
            1 - \frac{1+\sqrt{2}}{2}\kappa^{-1}.
        \end{equation*}
        \item If the momentum and step-size are chosen to be 
        \begin{equation}\label{eq:optimalDamping}
            a = 2\sqrt{2} \lambda_{\min}^{-1/2}\sqrt{\frac{\lambda_{\min}-T}{\lambda_{\min}+T}},\quad \eta  = 2a^{-1} = \frac{1}{\sqrt{2}} \lambda_{\min}^{1/2} \sqrt{\frac{\lambda_{\min}+T}{\lambda_{\min}-T}},
        \end{equation}
        then the mixing rate is 
        \begin{equation*}
            \sqrt{1 - \kappa^{-1} \frac{\lambda_{\max}-T}{\lambda_{\max}+T}\frac{\lambda_{\min}+T}{\lambda_{\min}-T}}.
        \end{equation*}
        In particular, taking $T = (1+\sqrt{2})^{-1} \lambda_{\min}$, the discrete-time mixing rate for $\kappa \gg 1$ is (up to first order)
        \begin{equation}
            1 - \frac{1+\sqrt{2}}{2}\kappa^{-1}.
        \end{equation}
    \end{enumerate}
\end{proposition}
\begin{proof}[Sketch]
    The momentum parameters are chosen to be optimal for the largest and smallest eigenvalues of $\Lambda$, respectively, and the step-sizes are chosen to be maximal such that the method converges. Moreover, the function $x \mapsto x^{-1}\frac{x-T}{x+T}$ is maximized at $(1+\sqrt{2})T$ and is decreasing for $x>(1+\sqrt{2})T$. The rates are obtained from a worst-case analysis over all possible eigenvalues for the given momentum and step sizes. A full derivation is given in \Cref{app:linearizedDiscreteTime}. 
\end{proof}

We note that the restriction on regularization $T \in [0,(1+\sqrt{2})^{-1}\lambda_{\min}]$ is used only to provide a uniform worst-case bound, using the monotonicity of $\lambda \mapsto \lambda^{-1} \frac{\lambda-T}{\lambda+T}$, which is increasing over $[T, (1+\sqrt{2})T]$ and decreasing over $[(1+\sqrt{2})T, +\infty)$. This can be relaxed to $T \in [0,\lambda_{\min})$, by replacing all instances of $\lambda_{\min}^{-1}\frac{\lambda_{\min}-T}{\lambda_{\min}+T}$ and $\lambda_{\max}^{-1}\frac{\lambda_{\max}-T}{\lambda_{\max}+T}$ with $\max\{ \lambda^{-1} \frac{\lambda-T}{\lambda+T}\mid \lambda \in \mathrm{Spec}(\Lambda) \subset [\lambda_{\min}, \lambda_{\max}]\}$ and $\min\{ \lambda^{-1} \frac{\lambda-T}{\lambda+T}\mid \lambda \in \mathrm{Spec}(\Lambda) \subset [\lambda_{\min}, \lambda_{\max}]\}$ respectively. The constant acceleration factor arises as the condition number of eigenvalues decreases after applying the function $\lambda \mapsto \lambda^{-1} \frac{\lambda-T}{\lambda+T}$ if $T>0$.

\subsection{Comparison with Kinetic Langevin Diffusion}\label{ssec:kineticComparison}
The kinetic/underdamped Langevin diffusion is given as a second-order version of the stochastic dynamics \labelcref{eq:Langevin}. In $\R^d$, if a particle position is $X$ with momentum $P$, the kinetic Langevin update proceeds by adding a Brownian motion on the momentum parameter \cite{dalalyan2020sampling},
\begin{equation}
     \begin{bmatrix}
        \dd X \\
        \dd P
    \end{bmatrix} = \begin{bmatrix}
        P\\
        -(a P + u \nabla V(X))
    \end{bmatrix} \dd{t} + \sqrt{2au} \begin{bmatrix}
        0\\
        I
    \end{bmatrix} \dd{W},
\end{equation}
where $W$ is a $2d$-dimensional standard Brownian motion, $a>0$ is a friction coefficient, and $u>0$ is an inverse mass, which can be taken to be $u=1$ without loss of generality. This converges to the phase-space stationary distribution $\rho(x,p) \propto \exp(-V(x) - \frac{1}{2u} \|p\|^2)$. A more detailed treatment is given in \cref{app:underdampedLangevin}. While more general convergence results are given in \cite{cheng2018underdamped,dalalyan2020sampling}, \cite{zuo2025gradient} specializes into the Gaussian setting, and we can compare the asymptotic convergence rates with the proposed ARWP method.

For now, consider the kinetic Langevin update in one dimension. For a target distribution $\gN(0, \Lambda)$, the particle position and momentums $(X_t, P_t) \in \R^2$ follow a joint normal distribution 
\begin{equation*}
    (X_t,P_t) \sim \gN\left(0, \begin{pmatrix}
        \Sigma_{11}(t) & \Sigma_{12}(t)\\
        \Sigma_{12}(t) & \Sigma_{22}(t)
    \end{pmatrix}\right),
\end{equation*}
with terminal values $(\Sigma_{11}, \Sigma_{12}, \Sigma_{22}) \rightarrow (\lambda,0,1)$ From \cite[Cor. 3.3]{zuo2025gradient}, the covariance update satisfies the following linear system\footnote{The result in the reference should be applied with $a=0$ and $\mathbf{C}=I$ as denoted in their work.}
\begin{equation}\label{eq:updateMat}
    \frac{\dd{}}{\dd{t}} \begin{bmatrix}
        \Sigma_{11}\\
        \Sigma_{12}\\
        \Sigma_{22}
    \end{bmatrix} = \begin{bmatrix}
        0 & 2 & 0\\
        -\lambda^{-1} & -a & 1\\
        0 & 2\lambda^{-1} & -2a
    \end{bmatrix} \begin{bmatrix}
        \Sigma_{11}-\lambda \\
        \Sigma_{12} \\
        \Sigma_{22}-1
    \end{bmatrix}.
\end{equation}
In particular, the eigenvalues of the update matrix are given by 
\begin{equation*}
    -a,\quad -a \pm \sqrt{a^2 - 4\lambda^{-1}}.
\end{equation*}
In multiple dimensions, standard numerical analysis gives that the convergence rate is given by the largest norm of $1+\eta \chi$, where $\chi$ runs over the three eigenvalues of the update matrix in \labelcref{eq:updateMat}, and over the eigenvalues of $\Lambda$. It remains to compute the step-size that minimizes the maximum norm over all possible eigenvalues of $\Lambda$.

In the small momentum critical damping case, the optimal momentum is taken to be $a = 2\lambda_{\max}^{-1/2}$, which gives a continuous-time convergence rate of $\mathcal{O}(t^2\exp(-at))$ in covariance. To compute the optimal step-size $\eta>0$, one computes
\begin{align*}
    &\quad |1 + \eta (-a \pm \sqrt{a^2 - 4\lambda_{\min}^{-1}})|^2 < 1\\
    &\Leftrightarrow 1 - 4\eta \lambda^{-1/2}_{\max} + 4\eta^2 \lambda_{\min}^{-1}<1.
\end{align*}
The rate is minimized when the quadratic on the left is minimized, which occurs when $\eta = \lambda_{\min}\lambda_{\max}^{-1/2}/2$. The discrete time per-iteration contraction rate is then given by 
\begin{align*}
    &\quad \max_{\lambda \in[\lambda_{\min},\lambda_{\max}]} |1-\eta (-a \pm \sqrt{a^2 - 4\lambda^{-1}})|\\
    &= \sqrt{1 - 4\frac{\lambda_{\min}\lambda_{\max}^{-1/2}}{2} \lambda^{-1}_{\max} + 4\frac{\lambda_{\min}^2\lambda_{\max}^{-1}}{4} \lambda_{\min}^{-1}} = \sqrt{1-\kappa^{-1}}.
\end{align*}

This should be compared with \Cref{prop:LinearizedNonasymptotic}(a). From \labelcref{eq:mixingrate}, we have acceleration as the constant in front of $\kappa^{-1}$ is $\frac{\lambda_{\max}-T}{\lambda_{\max}+T} \frac{\lambda_{\min}+T}{\lambda_{\min}-T}>1$. The linearized system of the proposed ARWP method is therefore more well-behaved than the one in kinetic Langevin method.

A similar analysis can be performed in the high critical damping case, where $a = 2\lambda_{\min}^{-1/2}$. As in \Cref{appsec:highCritDamping}, the optimal step-size is given by $\eta  = 1/a$. The rate is similarly given by $\sqrt{1-\kappa^{-1}}$, and we conclude the same conclusion as in the previous case.

\subsection{Non-Linearized Non-Asymptotic Continuous Time Convergence}\label{ssec:NoLinNoAsympCts}

In order to apply the analysis of the previous three sections, we need to show that the system indeed converges to the stationary distribution. In this section, we demonstrate convergence for the non-linearized continuous time system \labelcref{eqs:ctsCovUpdate1D}, which converges for $a > \lambda^{-1/2}$. 

We first show convergence of the non-linearized system in continuous time. In one dimension, the nonlinear coupled ODEs governing the covariance are given by 
\begin{equation*}
\begin{cases}
    \dot{\sigma}_t = 2g_t \sigma_t \\
    \dot g_t = -a g_t - g_t^2 - \lambda^{-1} + \tilde\sigma_t^{-1}.
\end{cases}
\end{equation*}
Nonlinearities arise from the introduction of the $g_t^2$ term, as well as the inverse covariance term $\tilde\sigma_t^{-1}$ in the $g_t$ update. We may change this into a pair of coupled ODEs in $\dot{\tilde \sigma}_t$ using the change of variables
\begin{equation*}
    \tilde\sigma_t = 2Tk_+^{-1} + k_+^{-2} \sigma_t.
\end{equation*}
The change of variables becomes a forcing term in $\tilde\sigma_t$,
\begin{equation}\label{eq:tildeSigmaGtContinuous}
    \begin{cases}
        \dot{\tilde\sigma}_t = 2g_t \tilde\sigma_t - 4g_t T k_+^{-1}\\
        \dot g_t = -ag_t - g_t^2 - \lambda^{-1} + \tilde\sigma_t^{-1}
    \end{cases}
\end{equation}

By selecting a particular Lyapunov function, we may show that this coupled ODE system converges to the stationary point $(\tilde\sigma_t, g_t) \rightarrow (\lambda, 0)$. This implies that for Gaussian distributions, the accelerated regularized Wasserstein proximal method in continuous time converges to the stationary distribution. We have three different results, corresponding to the underdamped case $a \in (\lambda^{-1/2}, 2\lambda^{-1/2}]$, a ``critical'' damping case $a = 2\lambda^{-1/2}$, and an overdamped case $a \ge 2\lambda^{-1/2}$. The underdamped and critical damping cases can use the same Lyapunov function, while the overdamped case requires a modified Lyapunov function. The results are summarized in the following two propositions.

\begin{proposition}\label{prop:changeMomRate}
    Consider the quadratic potential in one dimension $V(x) = \lambda^{-1} x^2/2$ and diffusion parameter $\beta=1$, and further let $T \in [0, \lambda)$ be a regularization parameter. Consider evolving a Gaussian distribution $\gN(0, \Sigma_t)$ through the continuous-time ARWP system \labelcref{eq:tildeSigmaGtContinuous}. Define a Lyapunov function as 
    \begin{equation}\label{eq:LyapunovDefMainText}
        \gE_t \coloneqq (\tilde\sigma_t - 2Tk_+^{-1}) [(\lambda^{-1/2} - \tilde\sigma_t^{-1/2}) + g_t]^2 + 2\KL(\tilde\sigma_t, \lambda),
    \end{equation}
    where we write the KL divergence between two variances to represent the KL divergence between the corresponding zero-mean Gaussian distributions. Then, the regularized Wasserstein proximal of the distributions $\wprox_{T,V}(\gN(0, \sigma_t)) = \gN(0 ,\tilde\sigma_t)$ converges to the terminal distribution $\gN(0, \lambda)$. Furthermore, the convergence rate can be characterized as follows:
    \begin{enumerate}
        \item (Critically damped) In one dimension, let the momentum parameter be taken as $a = 2\lambda^{-1/2}$. Furthermore, assume that the covariance satisfies $\tilde\sigma^2 \ge 2Tk_+^{-1} \lambda$. Then, the Lyapunov function satisfies the Lyapunov-like decay
        \begin{equation}\label{eq:LyapunovWeakDecay}
            \dot \gE_t \le -\lambda^{-1/2}(1 - 2Tk_+^{-1}\tilde\sigma^{-1}_t) \gE_t.
        \end{equation}
        In particular, close to the terminal distribution $\tilde\sigma_t \approx \lambda$, the decay is $\gE_t = \mathcal{O}(\exp(-rt))$, where the rate is 
        \begin{equation}
            r = \left(\frac{\lambda-T}{\lambda+T}\right) \lambda^{-1/2}.
        \end{equation}
        \item (Underdamped) For $a \in (\lambda^{-1/2}, 2\lambda^{-1/2}]$, define
        \begin{equation*}
            p = p_t \coloneqq \lambda^{-1/2} + 2Tk_+^{-1}\tilde\sigma^{-3/2},\quad b_+ \coloneqq \lambda^{-1/2} + \tilde\sigma_t^{-1/2}.
        \end{equation*}
        Let $r$ be the smallest positive root of the following (time-varying) quadratic equation:
        \begin{equation*}
            p^2 - 4 \left(-p+rb_+ \frac{\tilde\sigma - 2Tk_+^{-1}}{ \tilde\sigma(2\sqrt{\lambda}b_+-1)}\right)(-(1-r)b_+)=0.
        \end{equation*}
        Then $r$ exists, and the rate is given by 
        \begin{equation*}
            \dot \gE_t \le -\frac{2rb_+(\tilde\sigma - 2Tk_+^{-1})}{{ \tilde\sigma(2\sqrt{\lambda}b_+-1)}} \gE_t.
        \end{equation*}
    \end{enumerate}
\end{proposition}
\begin{proof}[Sketch.]
    Differentiating the Lyapunov function gives a quadratic equation in $g_t$, which is upper-bounded over all possible $g_t$. The conditions arise from the requirement that the $g_t^2$ coefficient in the quadratic is negative. The full derivation is given in \Cref{app:ConvAccelRegContinuous}; the first part is given in \Cref{app:appBOneDim} and the second part in \Cref{app:appBMultiDim}. 
\end{proof}
\begin{remark}
    As seen in part 2 of the proposition, the assumption that the covariance is larger than a constant in part 1 is not strictly necessary. Moreover, the first case is a special case of the second. This proposition quantifies the observation in \cite{tan2024noise}, that the convergence rate is a bit slower if the initial covariance is too small, but accelerates again close to the terminal distribution. This slowdown does not occur if the initial covariance is larger than the terminal covariance.
\end{remark}

This shows that in the underdamped and critically damped cases $a \in (\lambda^{-1/2}, 2\lambda^{-1/2}]$, the ODE system converges to the terminal solution $(\tilde\sigma_t, g_t) \rightarrow (\lambda,0)$. We note that for $a \le \lambda^{-1/2}$, the Lyapunov function may not necessarily decrease, and may lead to oscillation behaviors, similarly to \cite{su2016differential}. A similar theoretical restriction arises in \cite{dalalyan2020sampling}, which requires that the damping be greater than $m^{-1/2}$, where $m$ is the strong convexity constant of $V$.

In the overdamped case, the Lyapunov function as defined in \labelcref{eq:LyapunovDefMainText} does not necessarily decay. To show the convergence, we need to consider a modified Lyapunov function. This is given in the following proposition.

\begin{proposition}\label{prop:overdamped}
    (Overdamped) Let $V(x)=\lambda^{-1}x^2/2$ and $T\in [0, \lambda)$ be as in the previous proposition. Suppose that the momentum damping parameter is $a\ge 2\lambda^{-1/2}$, and define $\zeta \coloneqq a\lambda^{1/2}/2$. Define a modified Lyapunov function as 
    \begin{equation}\label{eq:overdampedModLyapMainText}
        \gF_t = \zeta^{-1}(\tilde\sigma_t - 2Tk_+^{-1})[b_- + \zeta g_t]^2 + 2\zeta \KL(\tilde\sigma_t, \lambda).
    \end{equation}
    Moreover, define the (time-varying) variables
    \begin{equation}
        p = p_t \coloneqq a\zeta - \lambda^{-1/2} + 2Tk_+^{-1}\tilde\sigma_t^{-3/2},\quad b_+ \coloneqq \lambda^{-1/2} + \tilde\sigma_t^{-1/2}.
    \end{equation}
    Let $r$ be the smallest positive root of the (time-varying) quadratic equation
    \begin{equation*}
    \zeta^{-2} p^2 + 4\left(-p + rb_+  \frac{\tilde\sigma_t - 2Tk_+^{-1}}{(2\sqrt\lambda b_+-1)\tilde\sigma_t}\right)(1-r)b_+ =0.
    \end{equation*}
    Then $r$ exists, and the modified Lyapunov function \labelcref{eq:overdampedModLyapMainText} decays as 
    \begin{equation}
        \dot \gF_t \le -\frac{2rb_+(\tilde\sigma - 2Tk_+^{-1})}{\zeta(2\sqrt{\lambda}b_+-1) \tilde\sigma} \gF_t.
    \end{equation}
\end{proposition}
\begin{proof}[Sketch]
    The form of the Lyapunov function comes from inspecting the previous Lyapunov function \labelcref{eq:LyapunovDefMainText}, and transferring the modifications of \cite{su2016differential} to the linear convergence case. The definition of the rate being in terms of a quadratics' root is sufficient in order to guarantee that the Lyapunov function is decreasing.
\end{proof}

We remark that a more general sufficient condition is $\zeta \ge a\lambda^{-1/2}/2$. However, due to the presence of $\zeta^{-1}$ in the rate, it is not beneficial to take a larger $\zeta$.

While the analysis presented so far is for the one-dimensional case, in the commuting case, we can extend this to higher dimensions simply by taking the trace over each eigendirection. For a damping parameter $a$ to work for all eigenvalues, one should consider the overdamped case, i.e., extending \Cref{prop:overdamped}. This can be summarized in the following corollary, in which the Lyapunov function is defined using a weighted KL divergence.

\begin{corollary}
    (Overdamped) Let $V(x)=x^\top \Lambda x/2$ and $T\in [0, \lambda_{\min})$, and assume $\Sigma_0$ commutes with $\Lambda$. Suppose that the momentum damping parameter is $a\ge 2\lambda_{\min}^{-1/2}$, and define $Z \coloneqq a\Lambda^{1/2}/2$. Define a modified Lyapunov function as 
    \begin{equation}\label{eq:overdampedModLyapMainText2}
        \gF_t = \Tr(Z^{-1}(\Sigma_t - 2K_+^{-1}[B_- + Z G_t]^2)) + 2\sum_{i=1}^d\zeta_i \KL(\tilde\sigma_{t,i}, \lambda_i).
    \end{equation}
    Moreover, define the (time-varying) variables in each eigendirection
    \begin{equation}
        p_i = p_{t,i} \coloneqq a\zeta_i - \lambda_i^{-1/2} + 2Tk_{+,i}^{-1}\tilde\sigma_{t,i}^{-3/2},\quad b_{+,i} \coloneqq \lambda_i^{-1/2} + \tilde\sigma_{t,i}^{-1/2},\quad i=1,...,d.
    \end{equation}
    Let $r_i,\, i=1,...,d$ be the smallest positive roots of the (time-varying) quadratic equations
    \begin{equation*}
    \zeta^{-2} p_i^2 + 4\left(-p_i + rb_{+,i}  \frac{\tilde\sigma_{t,i} - 2Tk_{+,i}^{-1}}{(2\sqrt\lambda b_{+,i}-1)\tilde\sigma_{t,i}}\right)(1-r)b_{+,i} =0,\quad i=1,...,d.
    \end{equation*}
    Then $r$ exists, and the modified Lyapunov function \labelcref{eq:overdampedModLyapMainText2} decays as 
    \begin{equation}
        \dot \gF_t \le -\min_{i=1,...,d}\left(\frac{2r_ib_{+,i}(\tilde\sigma_{t,i} - 2Tk_{+,i}^{-1})}{\zeta_i(2\sqrt{\lambda}b_{+,i}-1) \tilde\sigma_{t,i}}\right) \gF_t.
    \end{equation}
\end{corollary}
\begin{proof}
    By the assumption at the start of the section, the matrices $Z,\,B_-,\,G_t,\, \Sigma_t$ all commute. Summing \labelcref{eq:overdampedModLyapMainText} over all eigendirections yields the multidimensional Lyapunov function \labelcref{eq:overdampedModLyapMainText2}. In each direction, the decay is given by \labelcref{eq:overdampedModLyapMainText}; taking the smallest decay coefficient yields the uniform decay \labelcref{eq:overdampedModLyapMainText2}.
\end{proof}

This section shows that the forced ODE system \labelcref{eq:tildeSigmaGtContinuous} converges to the desired stationary distribution. However, these decay rates are not sharp. Each result shows that if the covariance of the regularized Wasserstein proximal is $\tilde\sigma \approx 2Tk_+^{-1}$, corresponding to $\sigma\approx 0$, then the convergence rate is slow.

\section{Experiments}\label{sec:experiments}
In the following numerical experiments, we compare the performance of the proposed ARWP method with several classical sampling methods and the non-accelerated BRWP method. This is first done on Gaussian target distributions, directly in covariance space, then using particle evolutions. Some low-dimensional non-log-concave examples follow to demonstrate the sensitivity and effectiveness in exploring away from local potential wells, as well as a high-dimensional Bayesian neural network example. The compared parameters for each of the experiments are given in \Cref{appsec:experimentParameters}. 

\subsection{One-Dimensional Gaussian}\label{ssec:OneDimGaussian}
We first verify the analysis for convergence in distributions presented in \Cref{sec:QuadraticConvergence}. In discrete time, the ARWP update \labelcref{eqs:symplecticARWP} has a closed-form update in covariances, given explicitly in \Cref{app:discreteTimeGaussianUpd}. We verify the results of the linearized discrete-time update in \Cref{ssec:DiscreteTimeLinearization}.

\begin{figure}
\centering

\renewcommand{\arraystretch}{0}
\noindent\makebox[\textwidth]{
\begin{tblr}{
    colspec={ccc}, colsep=-2pt
    }
 50 & 100 & 500 \\\hline
 \adjincludegraphics[width=0.35\textwidth,trim={0cm 0cm 1.5cm 1cm},clip,valign=c]{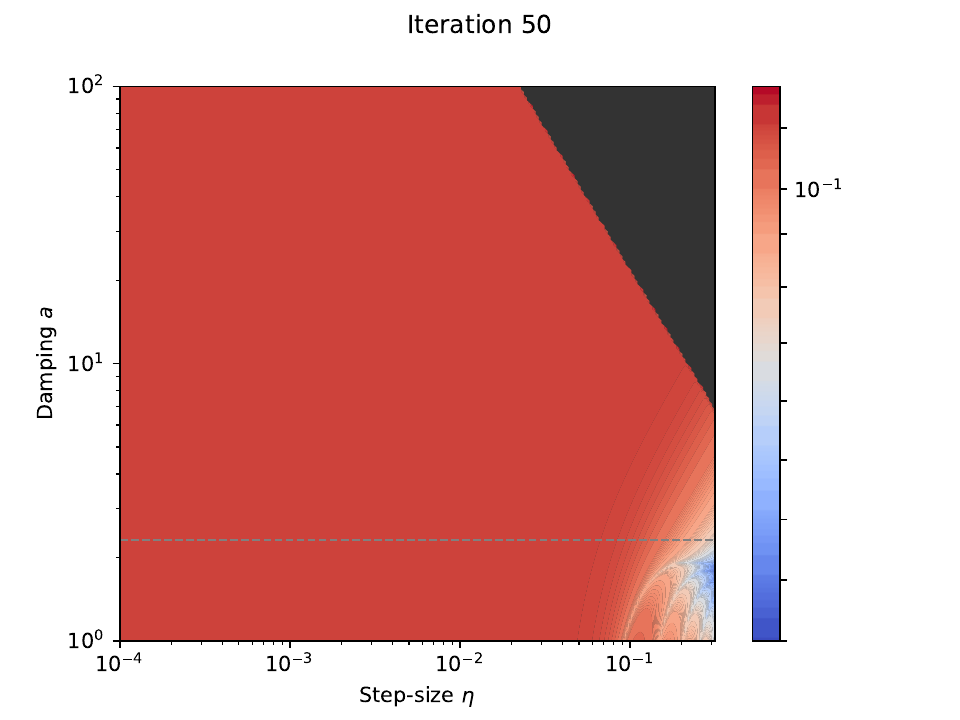}&
 \adjincludegraphics[width=0.35\textwidth,trim={0cm 0cm 1.5cm 1cm},clip,valign=c]{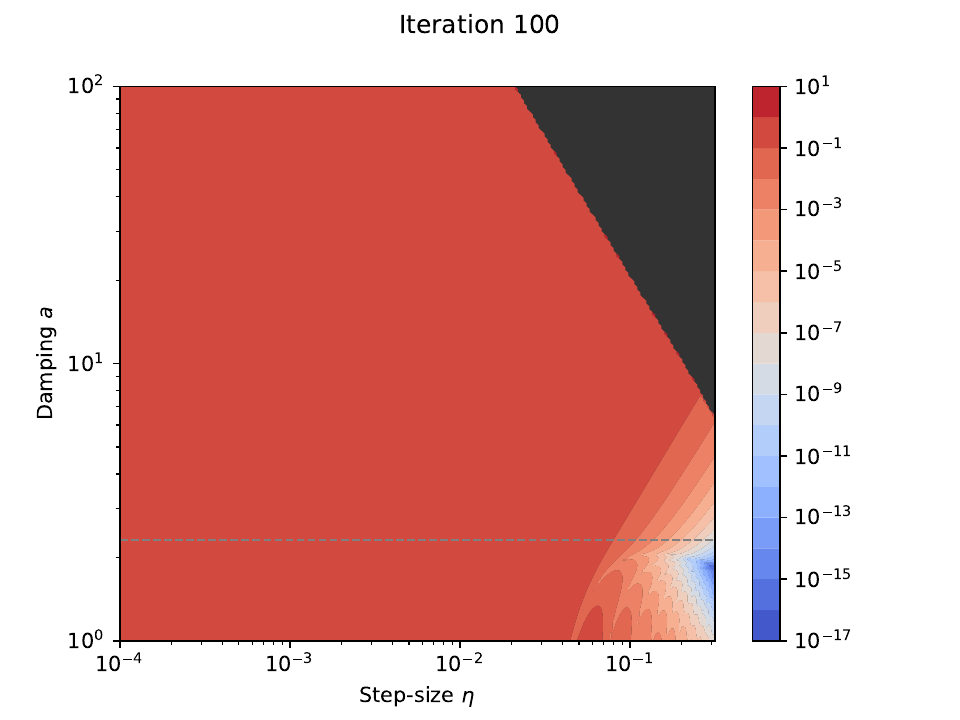}&
 \adjincludegraphics[width=0.35\textwidth,trim={0cm 0cm 1.5cm 1cm},clip,valign=c]{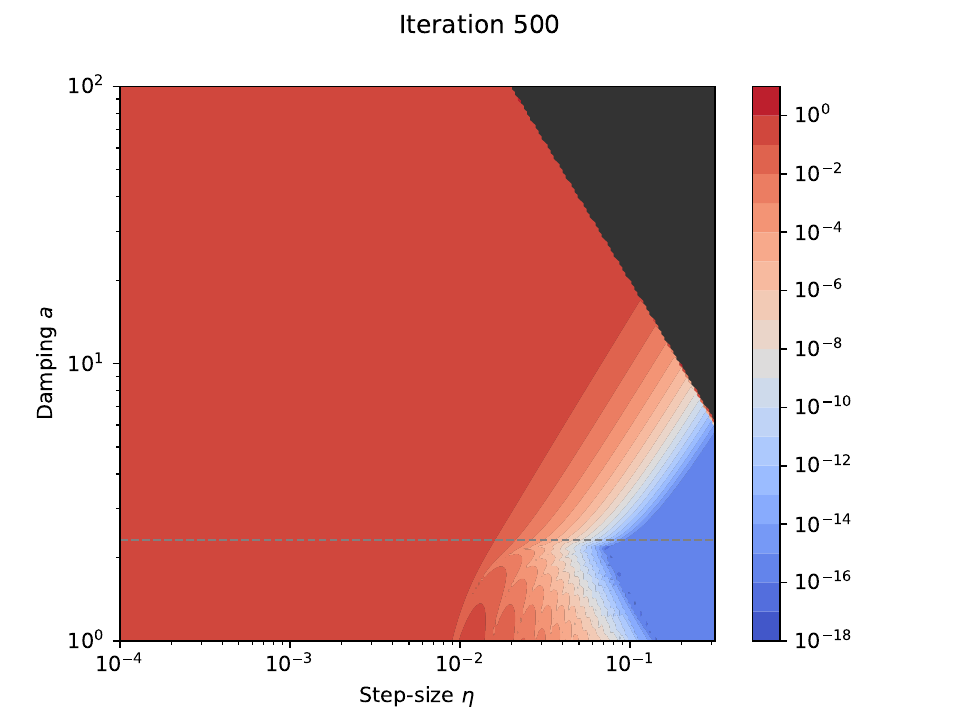}\\
 
 \adjincludegraphics[width=0.35\textwidth,trim={0cm 0cm 1.5cm 1cm},clip,valign=c]{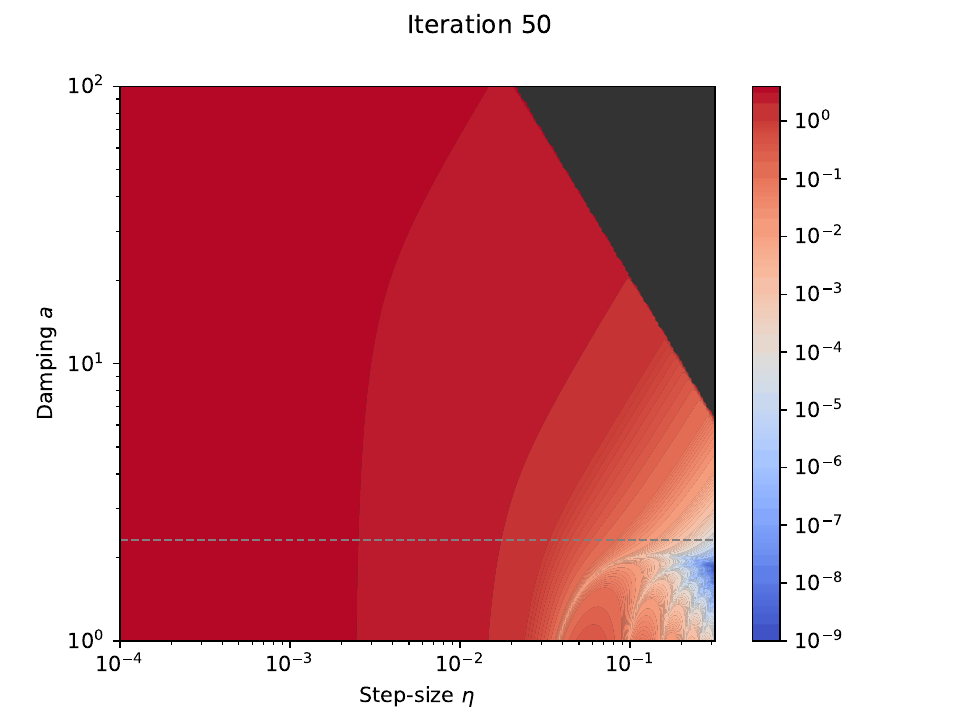}&
 \adjincludegraphics[width=0.35\textwidth,trim={0cm 0cm 1.5cm 1cm},clip,valign=c]{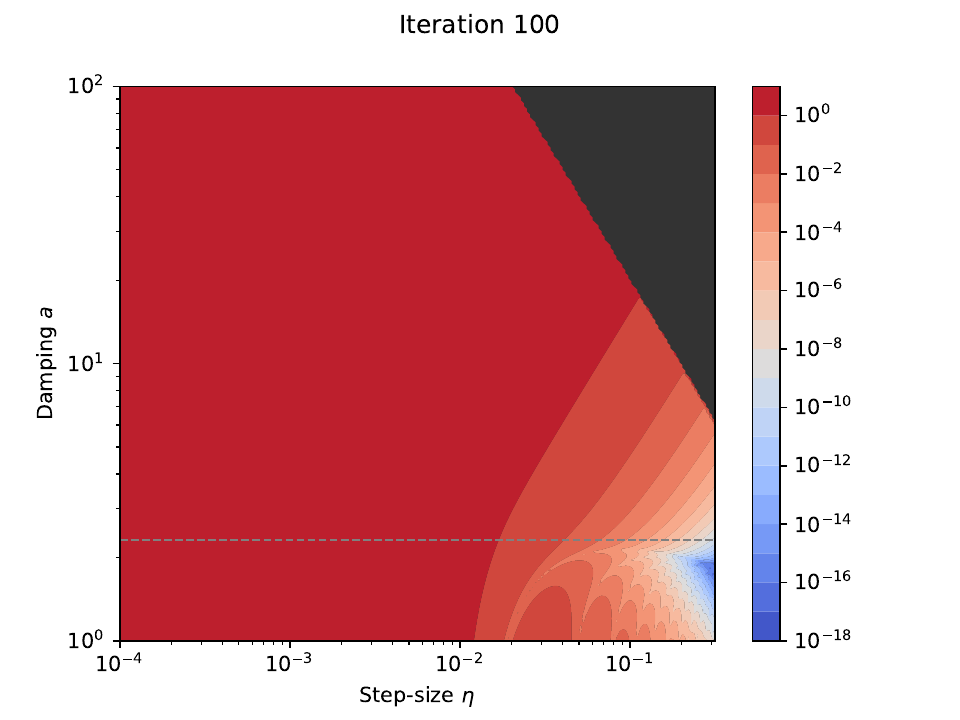}&
 \adjincludegraphics[width=0.35\textwidth,trim={0cm 0cm 1.5cm 1cm},clip,valign=c]{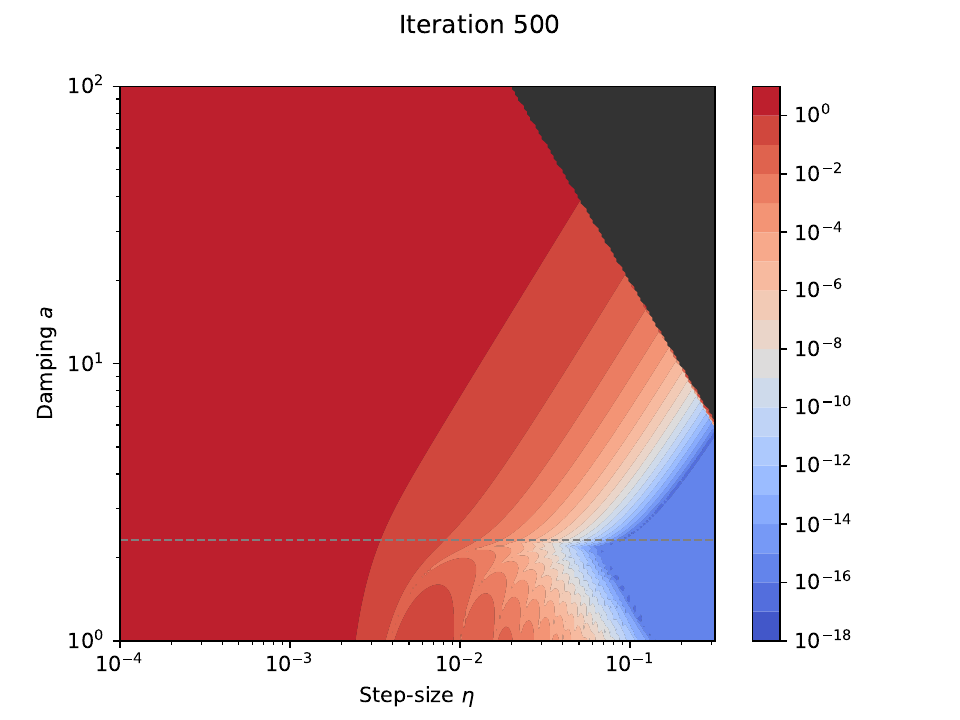}
\end{tblr}}
\caption{Contour plots of the covariance error of discrete-time ARWP \labelcref{eqs:symplecticARWP}, for target distribution $\Lambda = \gN(0, I_1)$, with initialization $\gN(0,10^{-3})$ (top) and $\gN(0,4)$ (bottom). Error is plotted against the damping parameter $a\in [10^0, 10^2]$ and step-size $\eta \in [10^{-4}, 10^{-0.5}]$, and fixed $T=0.2$. The gray line indicates the optimal damping parameter in continuous time, given by \labelcref{eq:optimalDamping}. The black region in the top right corner indicates (empirical) divergence, occurring when $a\eta > 2$.}\label{fig:lowdimGaussian}
\end{figure}
We consider the simplest case where the target variance is $\gN(0,1)$, or equivalently $V(x)= \|x\|^2/2$, and continue to fix $\beta = 1$. To demonstrate the optimal choice of damping parameter $a$ and step-size $\eta$, we plot a contour plot of the (trace) norm $|\sigma_k - \sigma_\infty|$, where $\sigma_\infty = 1-T^2$ is such that $\wprox_{T,V}(\gN(0, \sigma_\infty)) = \gN(0, 1)$. This error is plotted against the damping parameter $a \in [10^0, 10^2]$ and step-size $\eta \in [10^{-4}, 10^{-0.5}]$, equally log-spaced with 200 points. 

\Cref{fig:lowdimGaussian} plots this error in the covariance when updating using the ARWP method, with fixed regularization parameter $T = 0.2$, and with the covariance starting either as $\sigma_0 = 10^{-3}$ or as $\sigma_0 = 4$. From \Cref{prop:ctsTimeLinearized} applied with $\lambda=1$, we know that the optimal asymptotic rate is given when $a = 2\sqrt{2} \lambda^{-1/2}\sqrt\frac{1-T\lambda^{-1}}{1+T\lambda^{-1}}$, which qualitatively manifests as a cusp in the contour and is marked by a gray dashed line. Furthermore, the contour plot exhibits a ``bouncing'' phenomenon as the step-size increases for a fixed damping parameter $a$. As the number of iterations is fixed, the $x$-axis can be approximately interpreted as time, and this ``bouncing'' is the characteristic of Euclidean accelerated methods.

\subsection{Ill-Conditioned Gaussian}\label{ssec:illCondGaussian}
We now consider a 2D Gaussian with $\Sigma = \diag(0.1, 5)$, i.e. with condition number 50, applied with a finite number of particles $N=100$. We compare with the standard Langevin algorithms ULA and MALA \cite{dwivedi2018log}, and the non-accelerated regularized Wasserstein proximal method BRWP \cite{tan2024noise}. We additionally compare with two accelerated algorithms arising from discretizing the kinetic Langevin dynamics, namely the inertial Langevin algorithm (ILA) \cite{falk2025inertial}, and the kinetic Langevin Monte Carlo (KLMC) method \cite{dalalyan2020sampling}. These two algorithms are recalled in \Cref{appsec:DiscretizedKineticLan}.

We also consider another standard choice of acceleration, where the damping parameter is chosen as 
\begin{equation}\label{eq:nesterovChoice}
    1- a_k\eta  = \frac{k-1}{k+2}.
\end{equation}
The case where the damping $a$ is constant is denoted in future figures as ``ARWP-Heavy-ball'', while the variable case \labelcref{eq:nesterovChoice} is denoted ``ARWP-Nesterov''. This is in accordance with the classical optimization algorithms. We note that ARWP-Heavy-ball requires an additional choice of damping parameter $a$.

\Cref{fig:2DgaussianKL} demonstrates the convergence in KL divergence of the proposed ARWP methods, as computed using a Gaussian KDE with bandwidth $0.05$ and numerically integrated over $[-5,5]^2$ with mesh size $\Delta x=0.01$. It is compared with the unaccelerated BRWP method, classical Langevin methods ULA and MALA, as well as the accelerated Langevin methods ILA and KLMC. We observe acceleration of ILA and KLMC compared to ULA. In this simple case, MALA is able to perform similarly to the accelerated methods. 

\begin{figure}
    \centering
    \includegraphics[width=\linewidth]{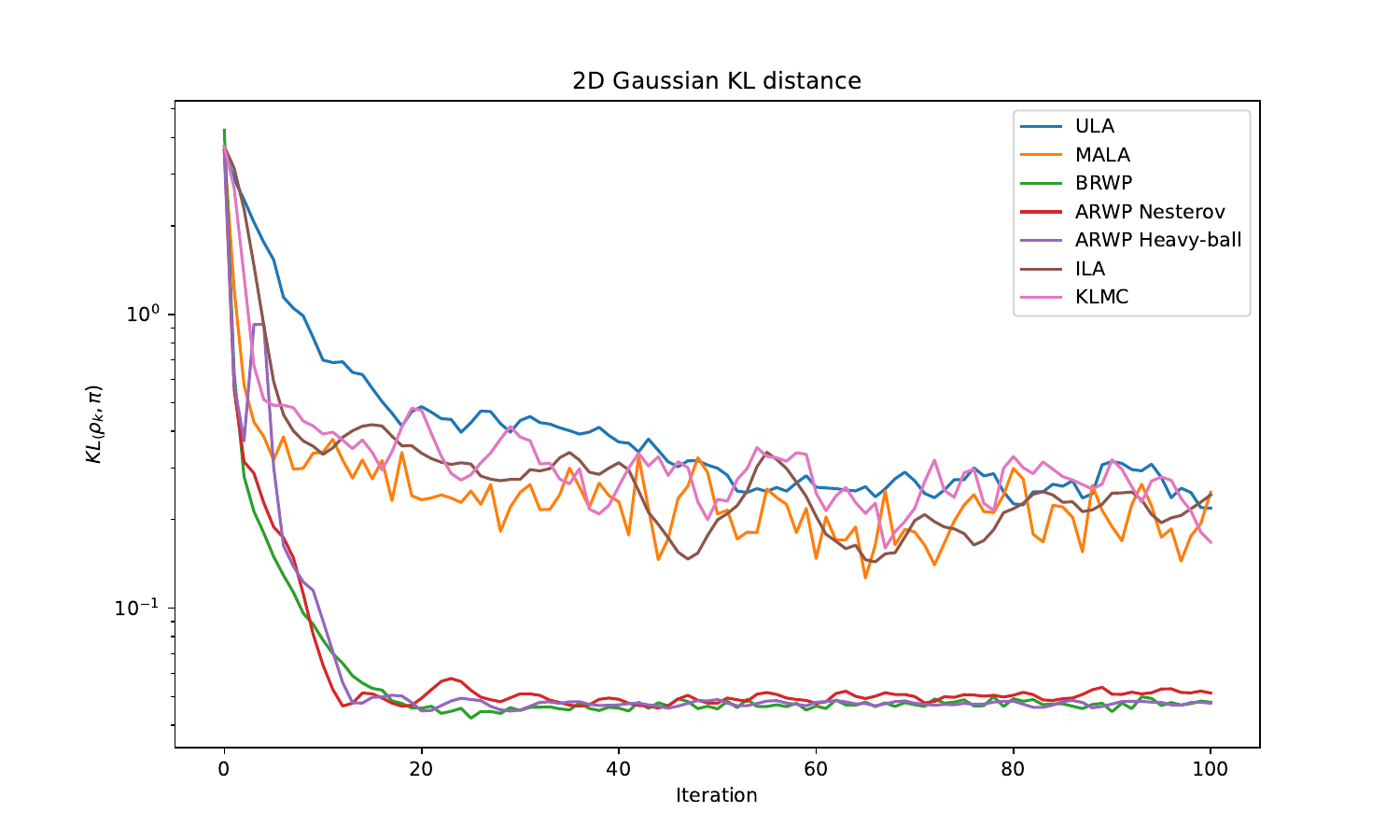}
    \caption{Convergence in KL divergence for the 2D Gaussian, run with 100 particles over 100 iterations. We observe that the deterministic methods ARWP and BRWP enjoy particle-wise convergence, indicated by the smaller oscillations between iterations. The accelerated Langevin methods ILA and KLMC continue to evolve due to the Brownian motion in the velocity. We observe that while BRWP has a faster initial convergence rate, both ARWP-Nesterov and ARWP-Heavy-ball reach their steady states faster. This is consistent with classical optimization results.}
    \label{fig:2DgaussianKL}
\end{figure}

\begin{figure}
    \centering
    \begin{subfigure}{0.32\textwidth}\centering
        \includegraphics[width=\linewidth, trim={2cm 1.7cm 1.7cm 1.5cm},clip,decodearray={0 1 0 1 0.5 1}]{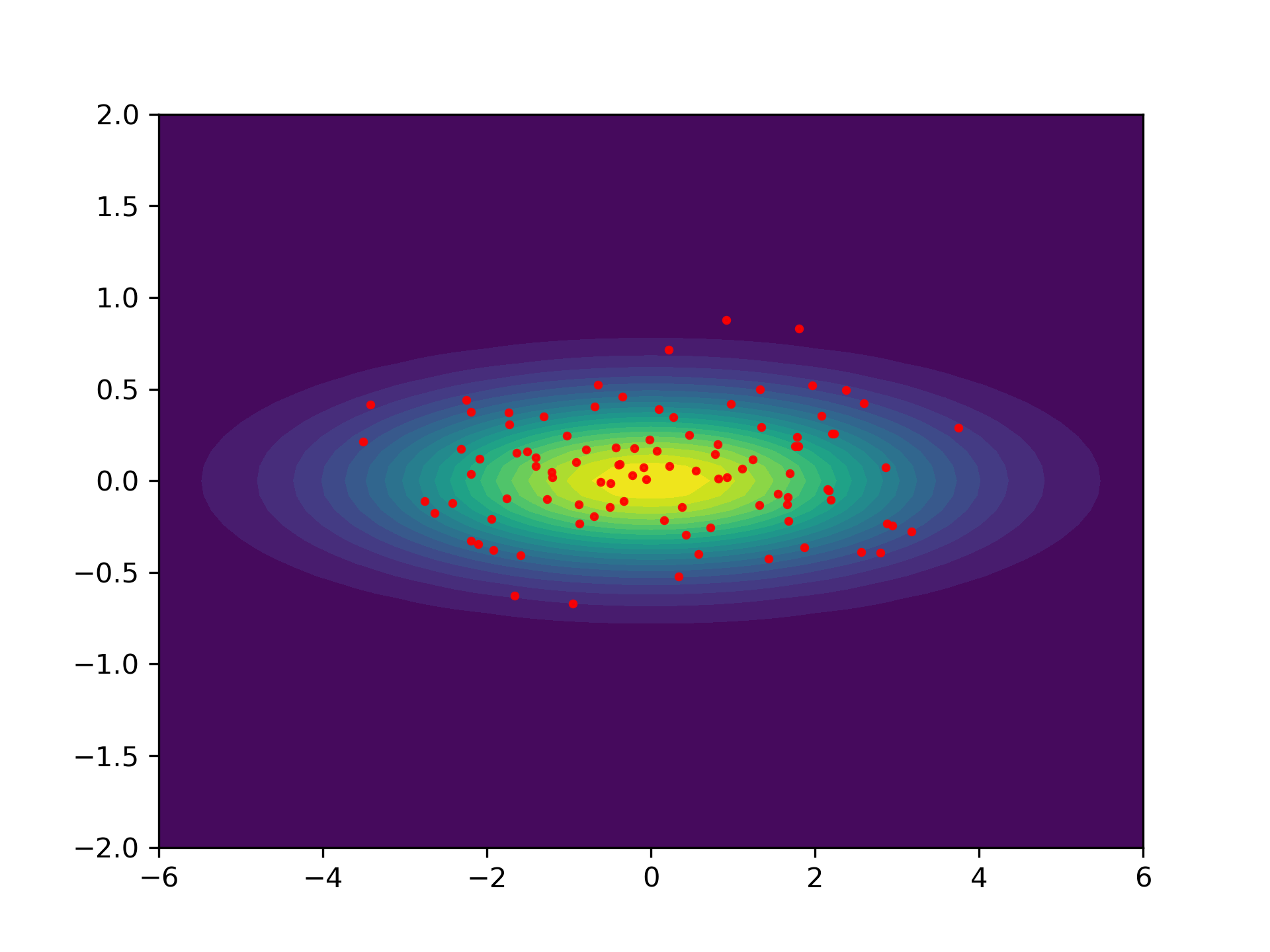}
        \caption{ULA}
    \end{subfigure}
    \begin{subfigure}{0.32\textwidth}\centering
        \includegraphics[width=\linewidth, trim={2cm 1.7cm 1.7cm 1.5cm},clip,decodearray={0 1 0 1 0.5 1}]{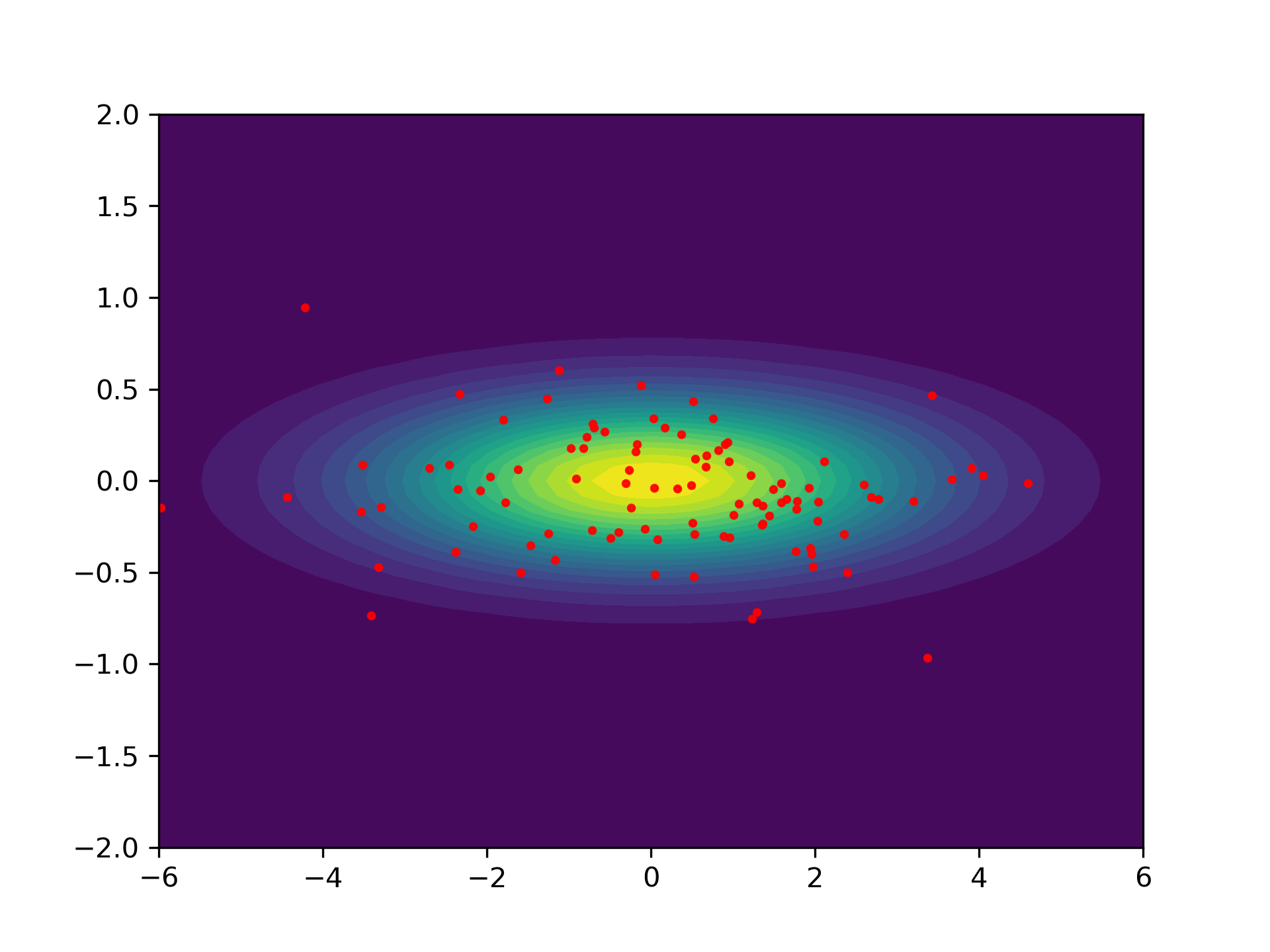}
        \caption{MALA}
    \end{subfigure}
    \begin{subfigure}{0.32\textwidth}\centering
        \includegraphics[width=\linewidth, trim={2cm 1.7cm 1.7cm 1.5cm},clip,decodearray={0 1 0 1 0.5 1}]{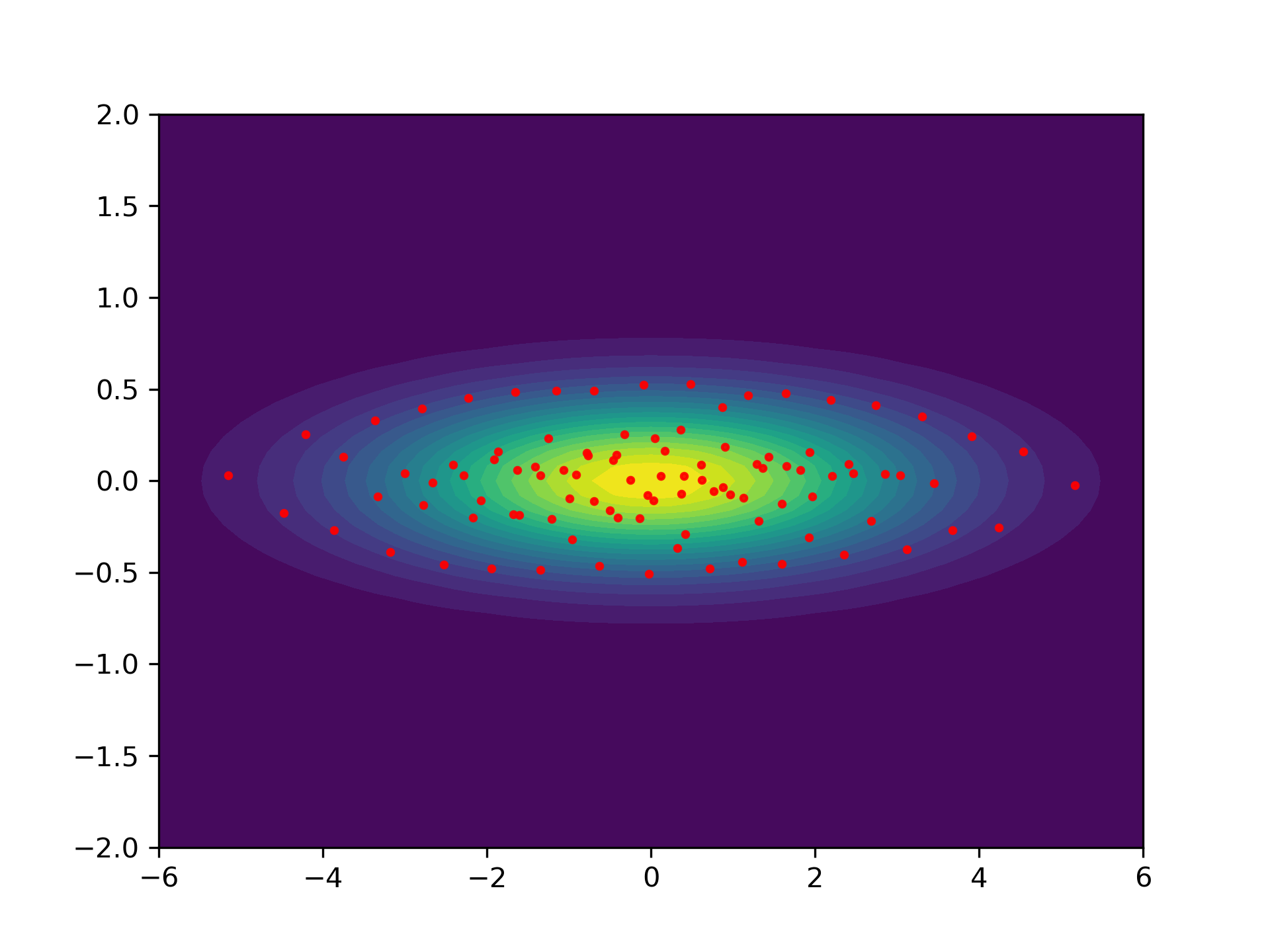}
        \caption{BRWP}
    \end{subfigure}\\
    \begin{subfigure}{0.32\textwidth}\centering
        \includegraphics[width=\linewidth, trim={2cm 1.7cm 1.7cm 1.5cm},clip,decodearray={0 1 0 1 0.5 1}]{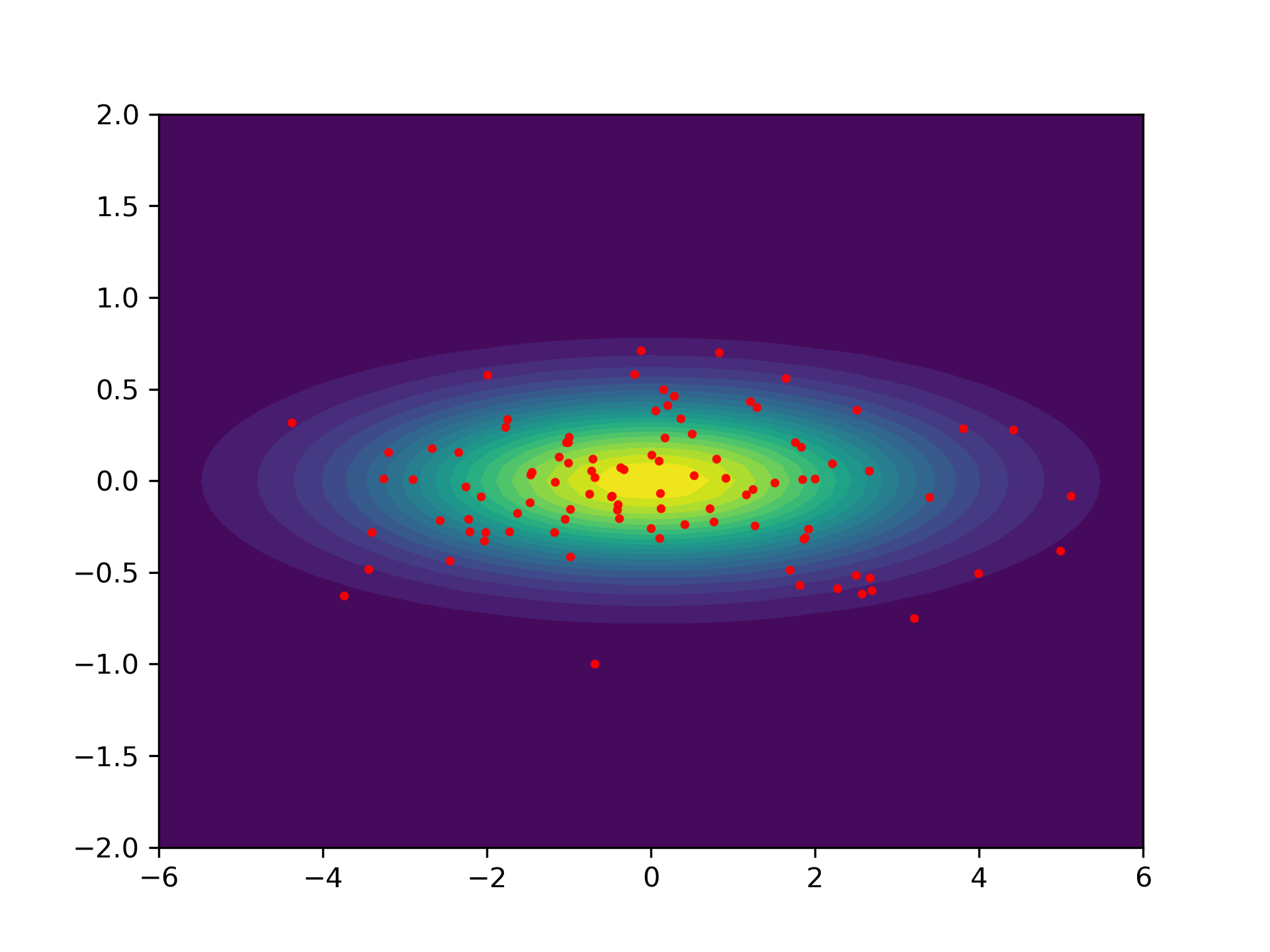}
        \caption{ILA}
    \end{subfigure}
    \begin{subfigure}{0.32\textwidth}\centering
        \includegraphics[width=\linewidth, trim={2cm 1.7cm 1.7cm 1.5cm},clip,decodearray={0 1 0 1 0.5 1}]{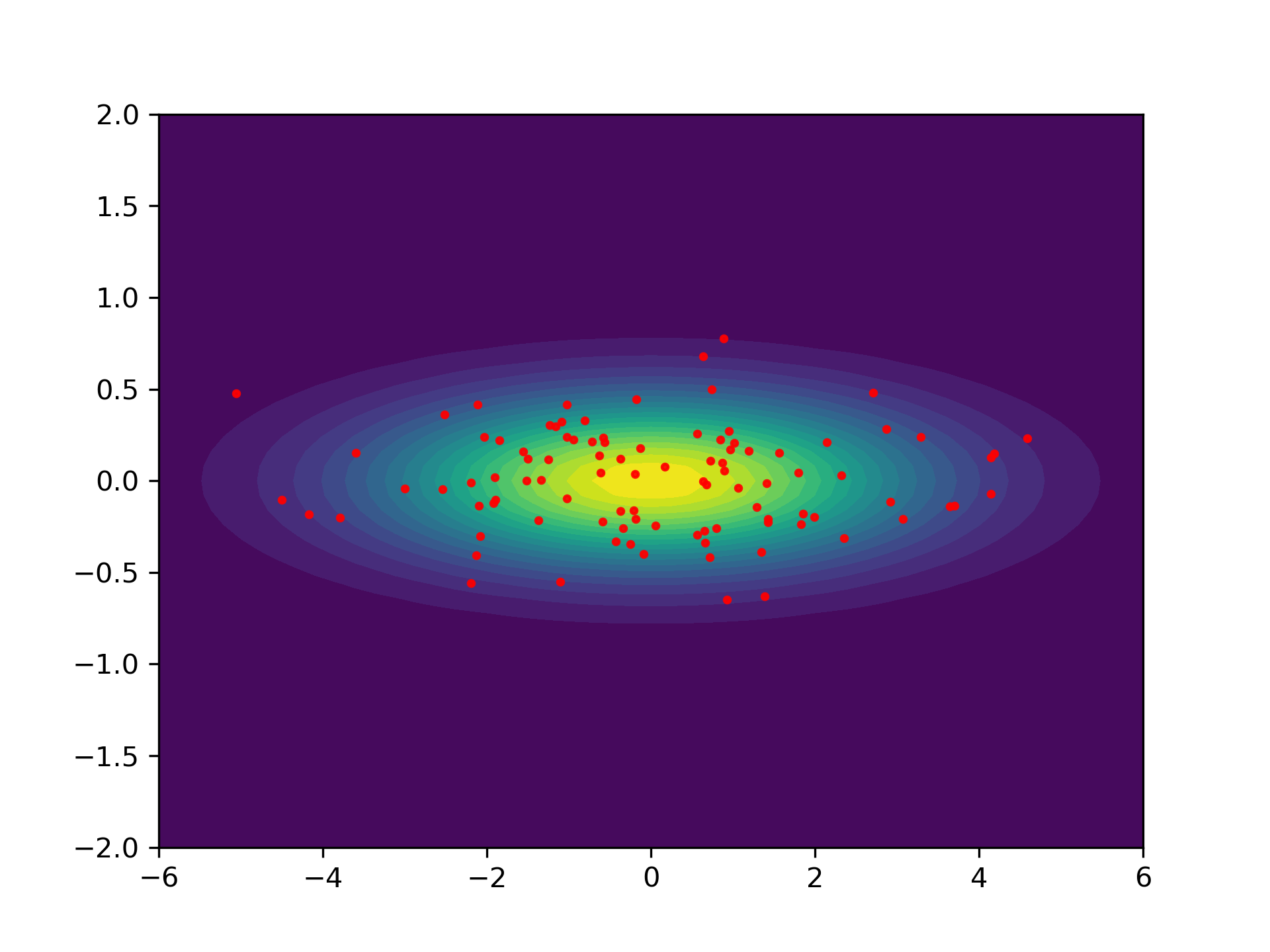}
        \caption{KLMC}
    \end{subfigure}
    \begin{subfigure}{0.32\textwidth}\centering
        \includegraphics[width=\linewidth, trim={2cm 1.7cm 1.7cm 1.5cm},clip,decodearray={0 1 0 1 0.5 1}]{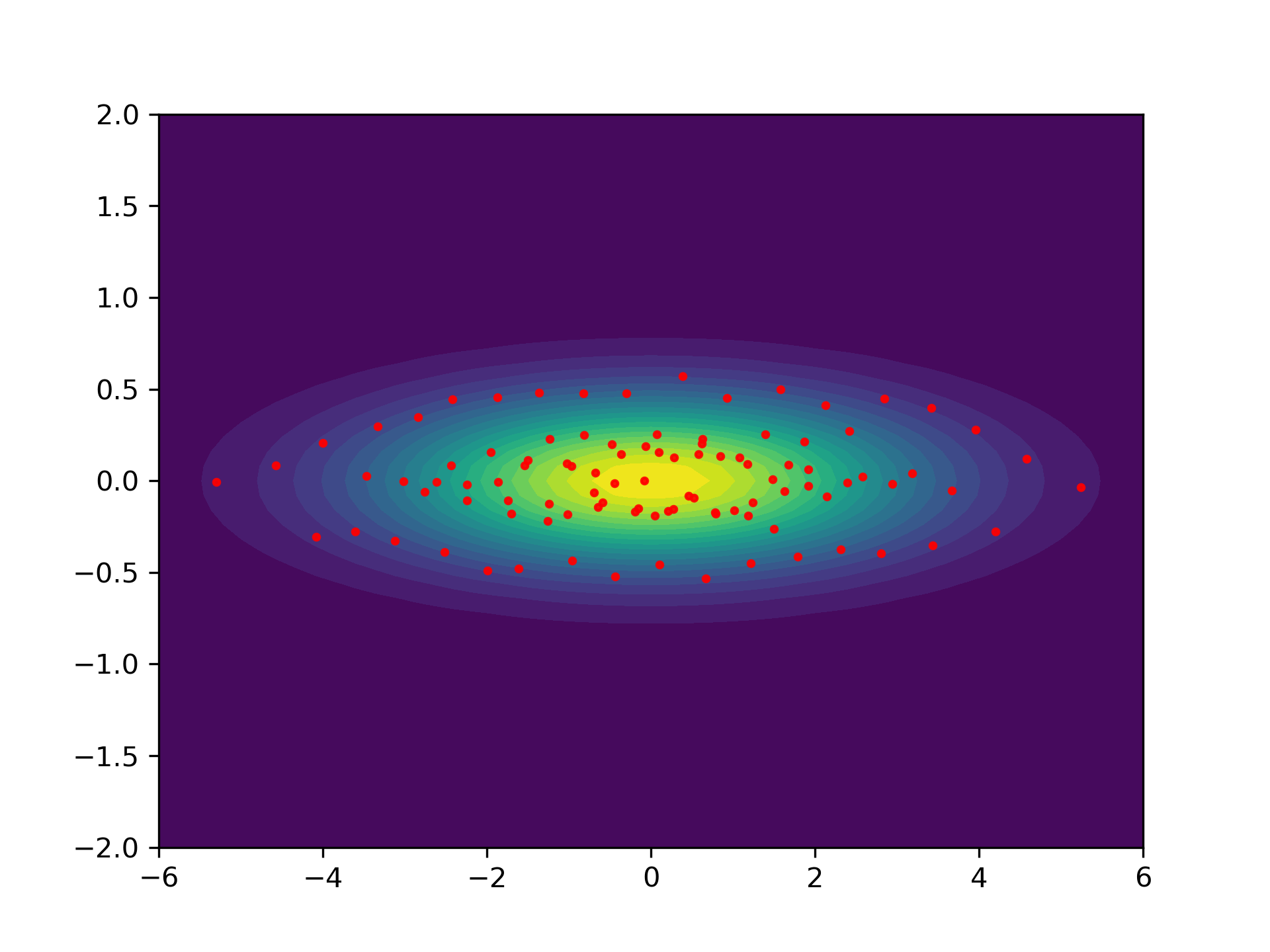}
        \caption{Proposed ARWP}
    \end{subfigure}
    \caption{Particle positions after 100 iterations for the 2D Gaussian with condition number $\kappa=50$, run with 100 particles. We observe that both the accelerated and non-accelerated Langevin algorithms look more randomly sampled, as the particles do not interact. Moreover, the proposed ARWP method has a similarly structured but slightly messier terminal position compared to BRWP. Both ARWP and BRWP particle positions converge and do not move.}
    \label{fig:GaussianParticle100Iter}
\end{figure}

The deterministic methods ARWP and BRWP are both able to reach significantly lower terminal KL divergence, due to the structure of the final iterates. In particular, with this low number of particles, unbiased methods such as MALA still have regions of mass that are not represented by particles, leading to a higher KL divergence. This is demonstrated in \Cref{fig:GaussianParticle100Iter}, where the particle positions under the Langevin algorithms are less structured. 

The acceleration in ARWP for this simple case is mild, manifesting as a slightly faster convergence with the same $T=0.05$. This is due to the slightly larger allowed step-size in ARWP. In particular, ARWP is able to use a step-size of $0.3$ instead of $0.2$ for BRWP, which diverges for step-size $0.3$. From \labelcref{eq:optimalDamping}, the optimal step-size for ARWP for optimally chosen damping is given by 
\begin{equation*}
    h^*_{\text{ARWP}} = \frac{1}{\sqrt{2}} (0.1)^{1/2} \sqrt{\frac{0.1+0.05}{0.1-0.05}} \approx 0.387.
\end{equation*}
However, the theoretical optimal step-size for BRWP is given by \cite[Cor. 1]{tan2024noise}
\begin{equation*}
    h^*_{\text{BRWP}} = \frac{1}{2} (0.1)\frac{0.1+0.05}{0.1-0.05} = 0.15,
\end{equation*}
which is smaller than that of ARWP.
\begin{remark}
    By taking a slightly larger step size, we sacrifice some speed in the fast directions for small covariance, while accelerating the slow directions for large covariance. Taking a slightly larger than optimal step-size usually results in acceleration.
\end{remark}

\subsection{2D Rosenbrock}\label{ssec:rosenbrock}
We consider the (scaled) Rosenbrock function \cite{pagani2022n}
\begin{equation*}
    V(x,y) = \frac{1}{20}\left[(1-x)^2 + 100(y-x^2)^2\right],
\end{equation*}
with the gradient operator
\begin{equation*}
    \begin{bmatrix}
        \partial_x V\\
        \partial_y V
    \end{bmatrix}(x,y) = \frac{1}{20}\begin{bmatrix}
        2(1-x) - 400x(y-x^2)\\
        200(y-x^2)
    \end{bmatrix}.
\end{equation*}
This is a difficult non-log-concave distribution to sample. The high Lipschitz constants away from the valley require a small step size, leading to slow exploration along the valley. In particular, the mass of the distribution is distributed away from the valley near the minimizer\footnote{This also means that approximating the KL divergence using numerical integration is expensive.}; only about 30\% of the mass is distributed within the square region indicated within the figure \cite[Fig. 1]{pagani2022n}. Therefore, it is desirable for a method to be able to sample from the tails of this distribution. We run the methods with 100 particles initialized with distribution $\gN(0, I_2)$ up to 500 iterations.

\begin{figure}
\centering

\renewcommand{\arraystretch}{0}
\noindent\makebox[0.9\textwidth]{
\begin{tblr}{
    colspec={ccc}, colsep=1pt
    }
 Iter. 50 & 200 & 500\\\hline
 \adjincludegraphics[width=0.35\textwidth,trim={1.3cm 0.7cm 1.7cm 1.2cm},clip,valign=c,decodearray={0 1 0 1 0.5 1}]{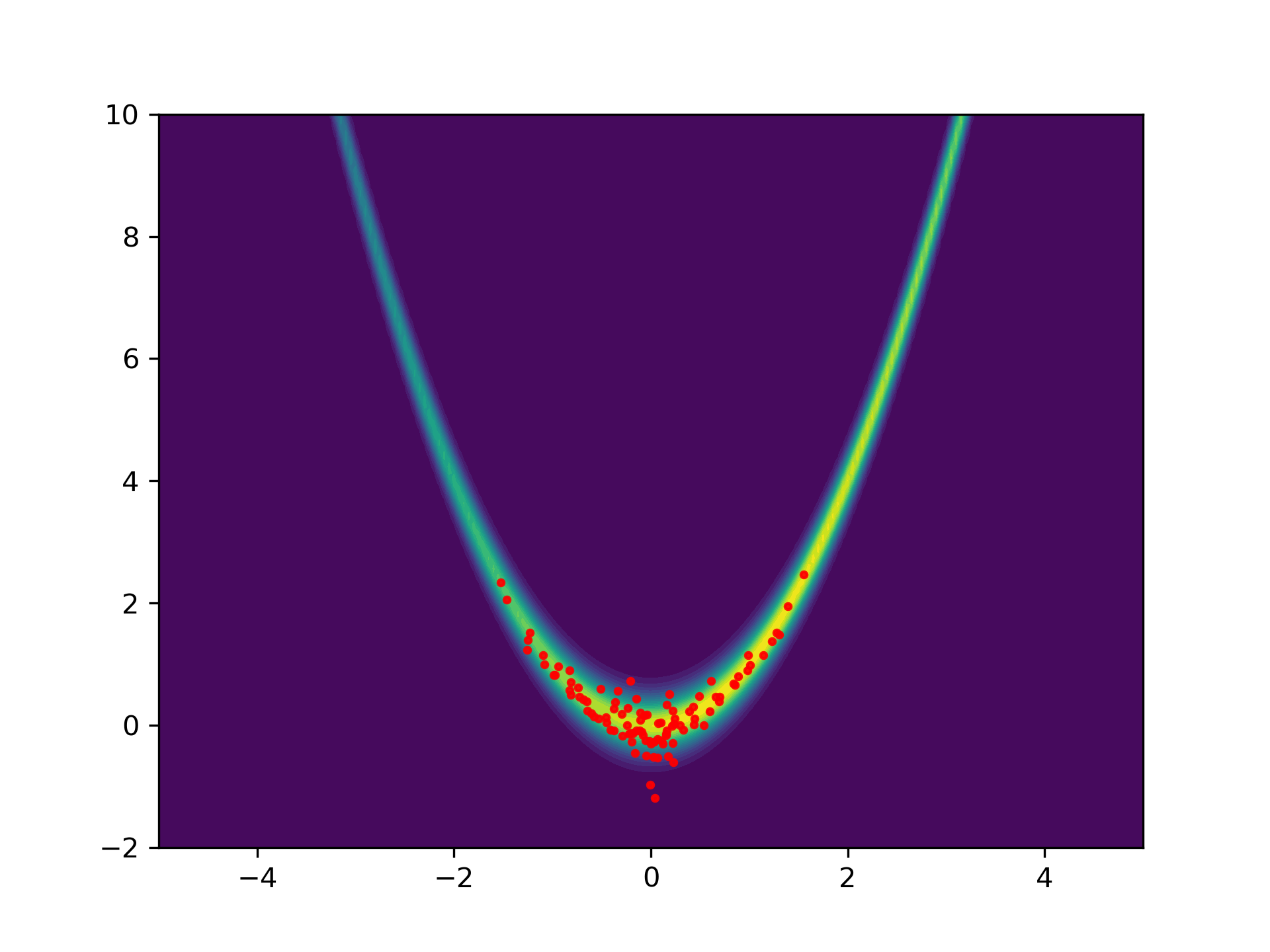}&
 \adjincludegraphics[width=0.35\textwidth,trim={1.3cm 0.7cm 1.7cm 1.2cm},clip,valign=c,decodearray={0 1 0 1 0.5 1}]{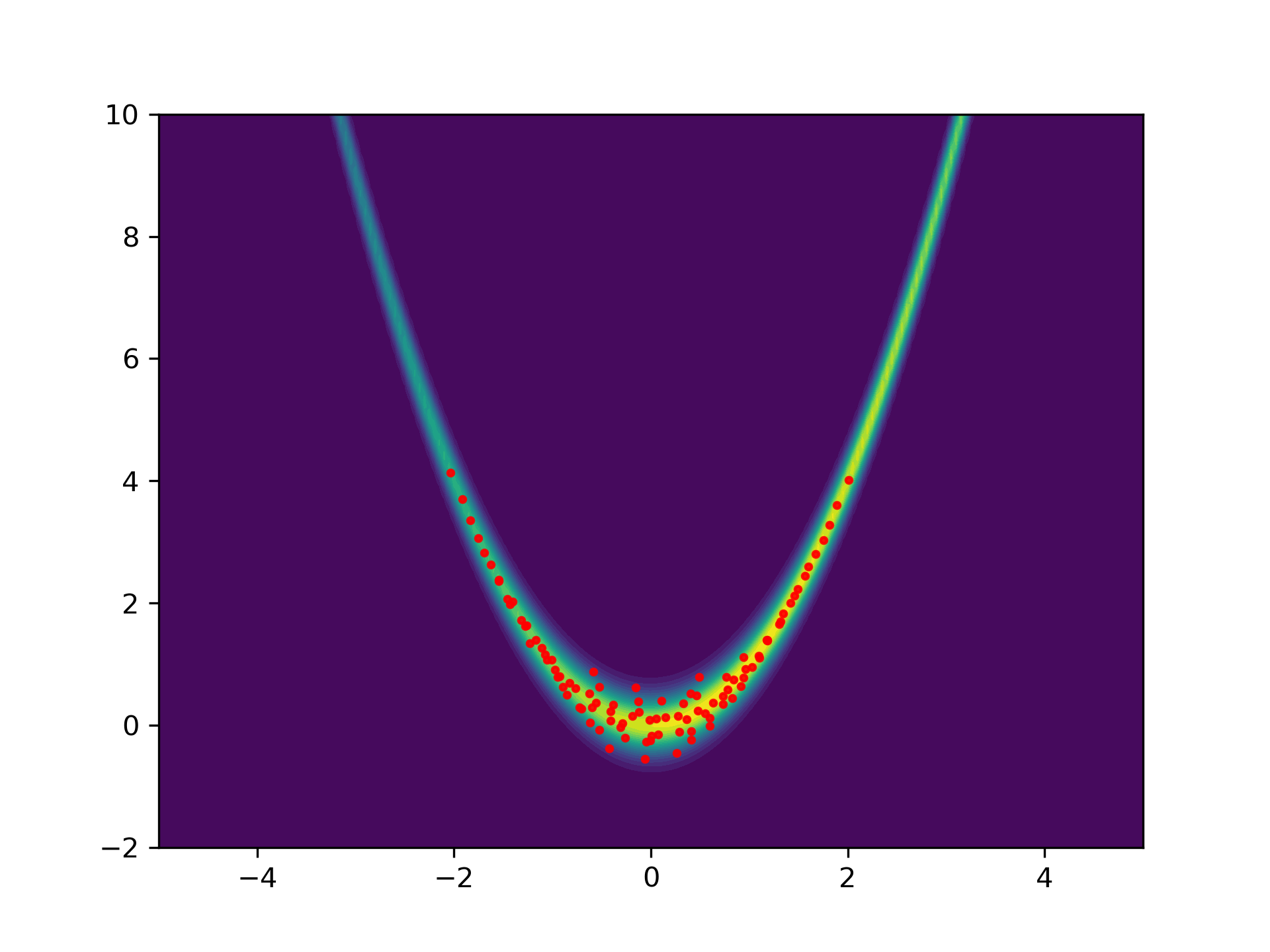}&
 \adjincludegraphics[width=0.35\textwidth,trim={1.3cm 0.7cm 1.7cm 1.2cm},clip,valign=c,decodearray={0 1 0 1 0.5 1}]{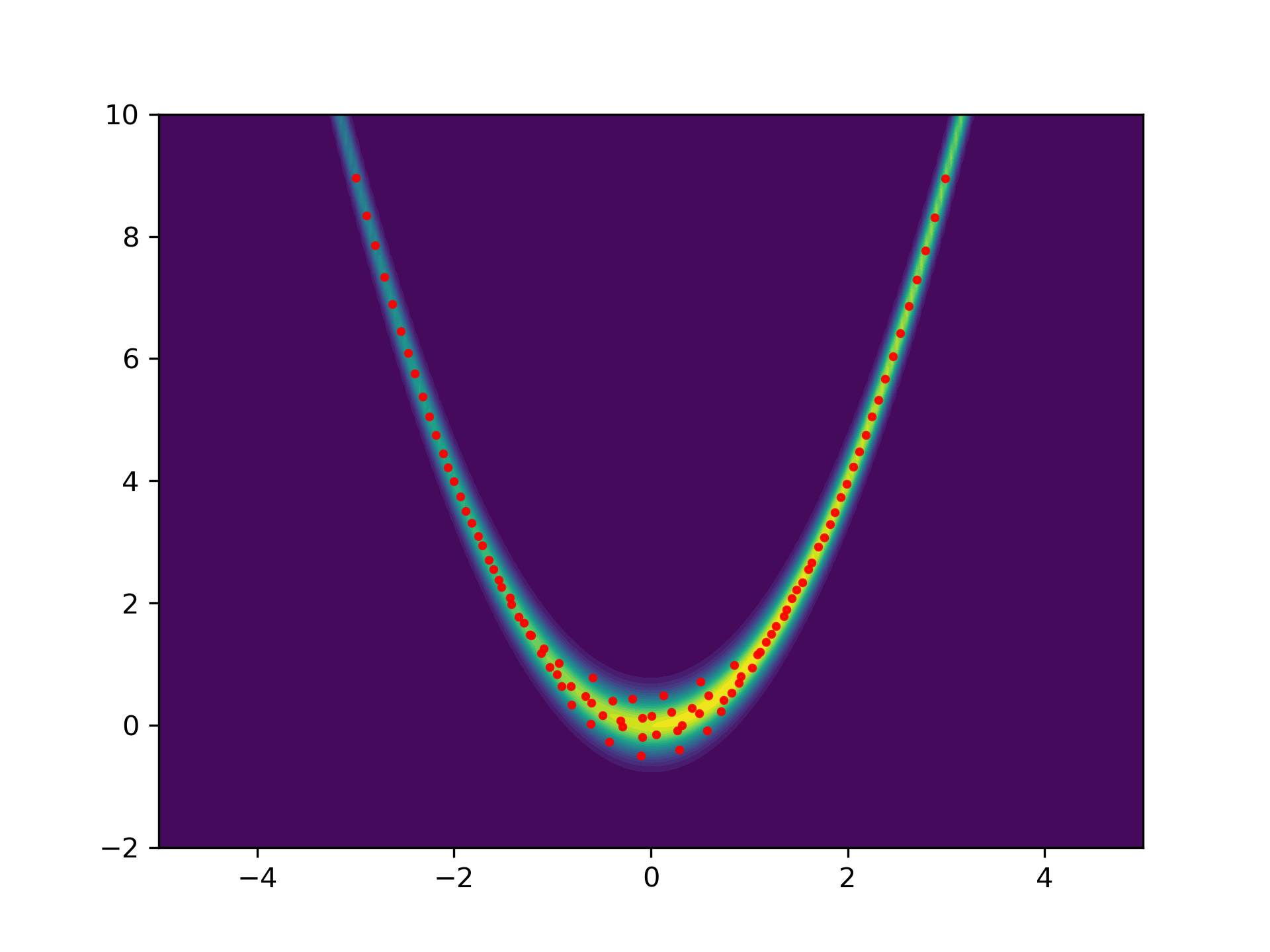}\\
 
 \adjincludegraphics[width=0.35\textwidth,trim={1.3cm 0.7cm 1.7cm 1.2cm},clip,valign=c,decodearray={0 1 0 1 0.5 1}]{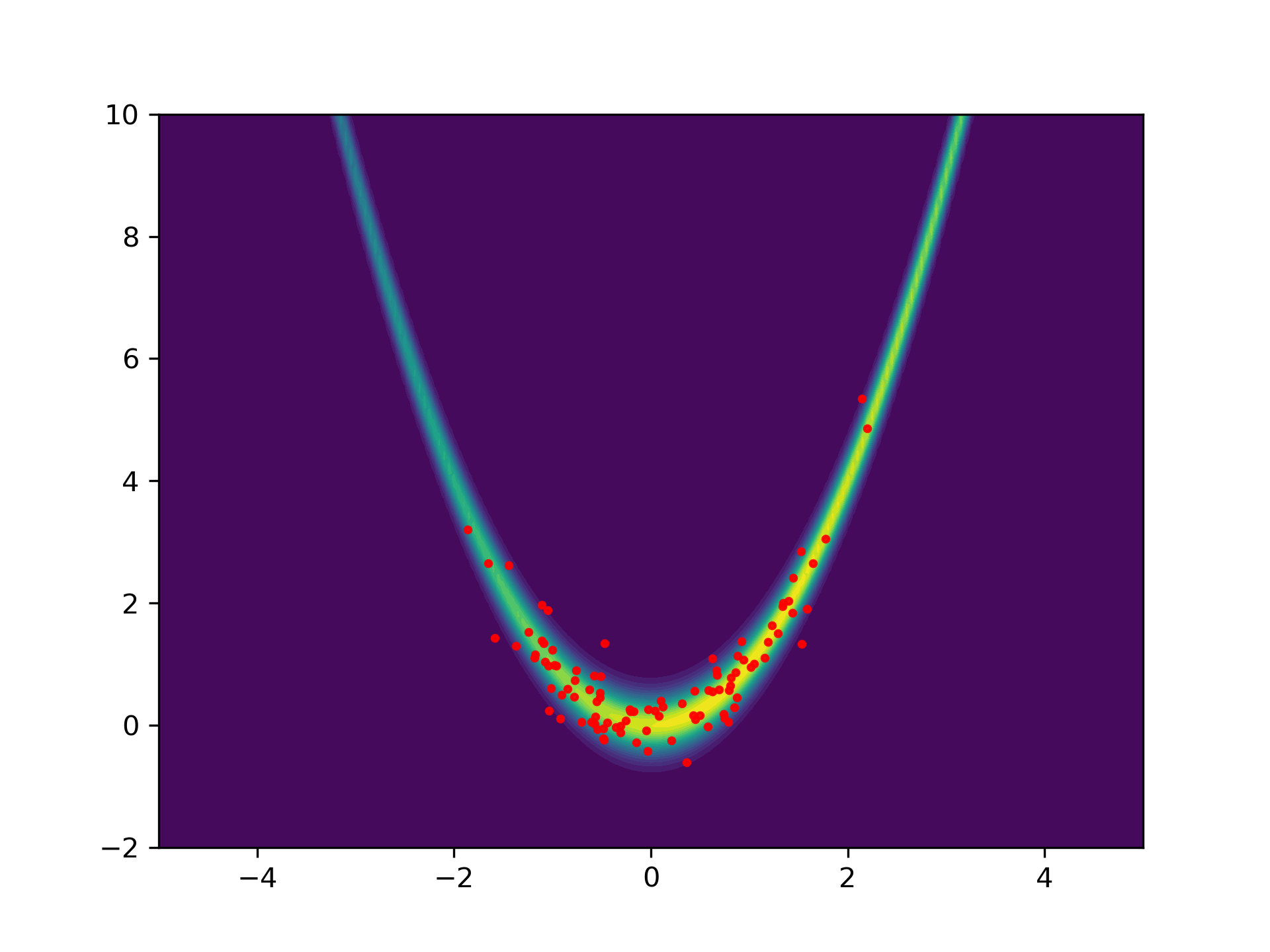}&
 \adjincludegraphics[width=0.35\textwidth,trim={1.3cm 0.7cm 1.7cm 1.2cm},clip,valign=c,decodearray={0 1 0 1 0.5 1}]{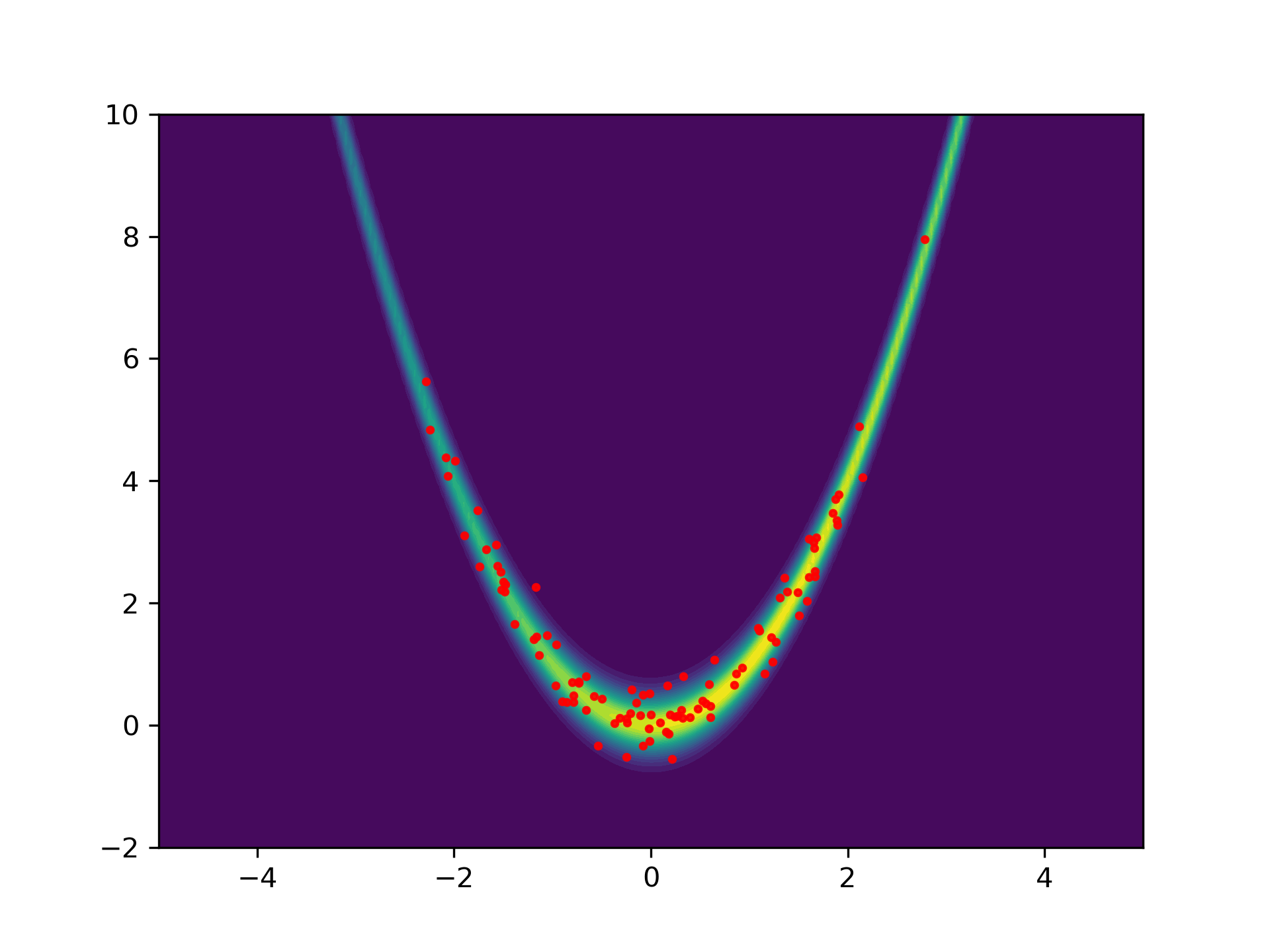}&
 \adjincludegraphics[width=0.35\textwidth,trim={1.3cm 0.7cm 1.7cm 1.2cm},clip,valign=c,decodearray={0 1 0 1 0.5 1}]{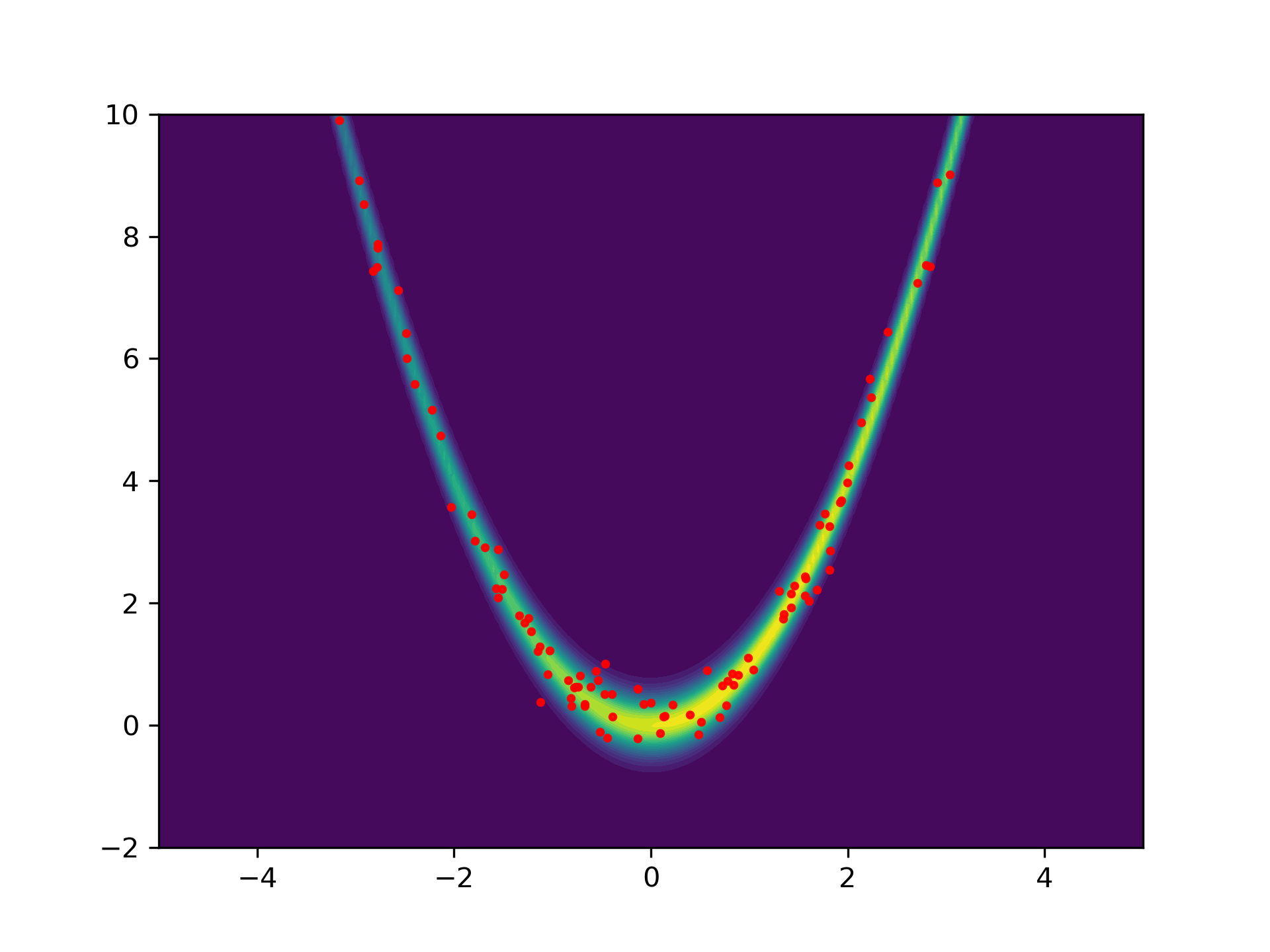}\\

 \adjincludegraphics[width=0.35\textwidth,trim={1.3cm 0.7cm 1.7cm 1.2cm},clip,valign=c,decodearray={0 1 0 1 0.5 1}]{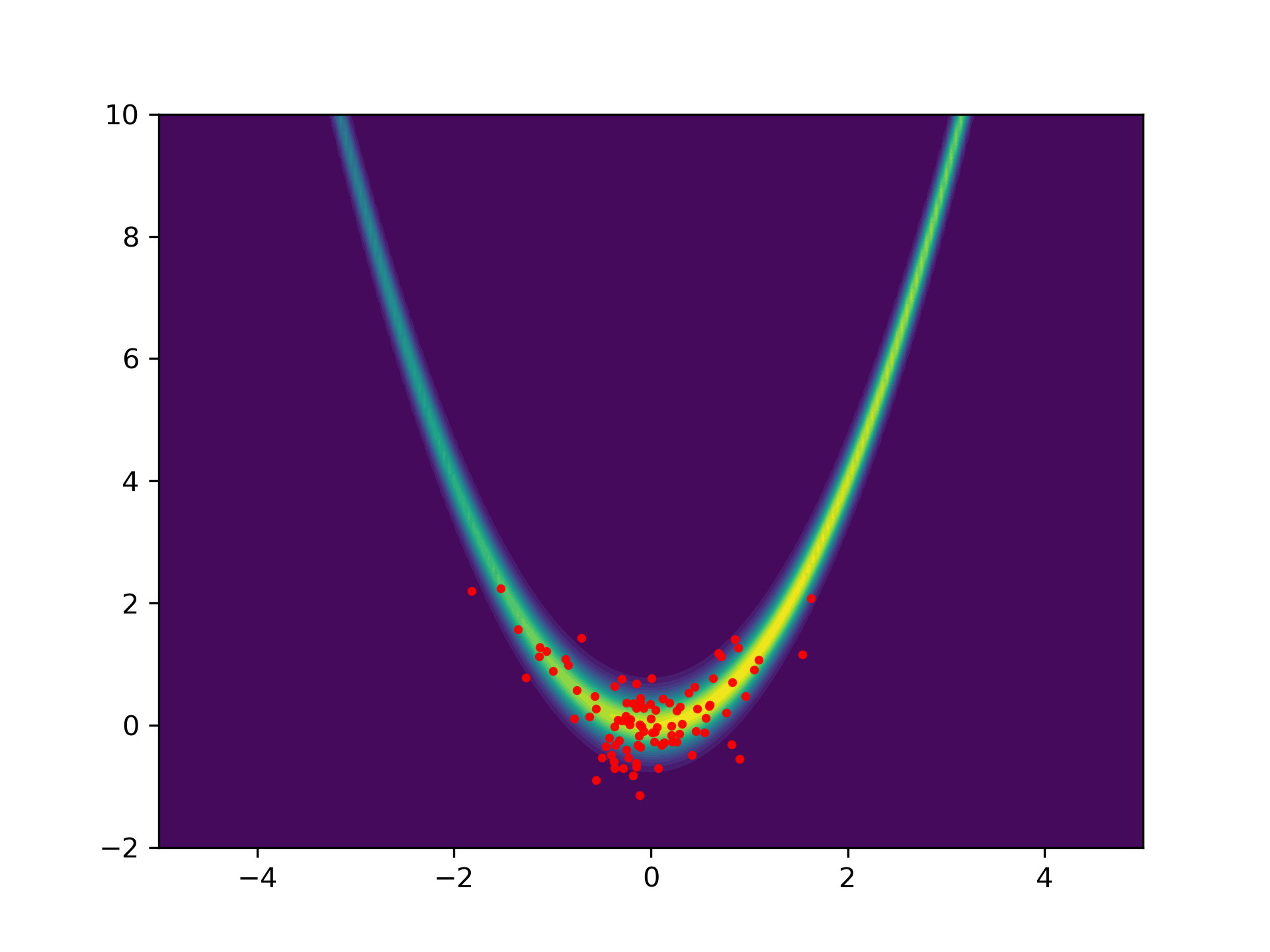}&
 \adjincludegraphics[width=0.35\textwidth,trim={1.3cm 0.7cm 1.7cm 1.2cm},clip,valign=c,decodearray={0 1 0 1 0.5 1}]{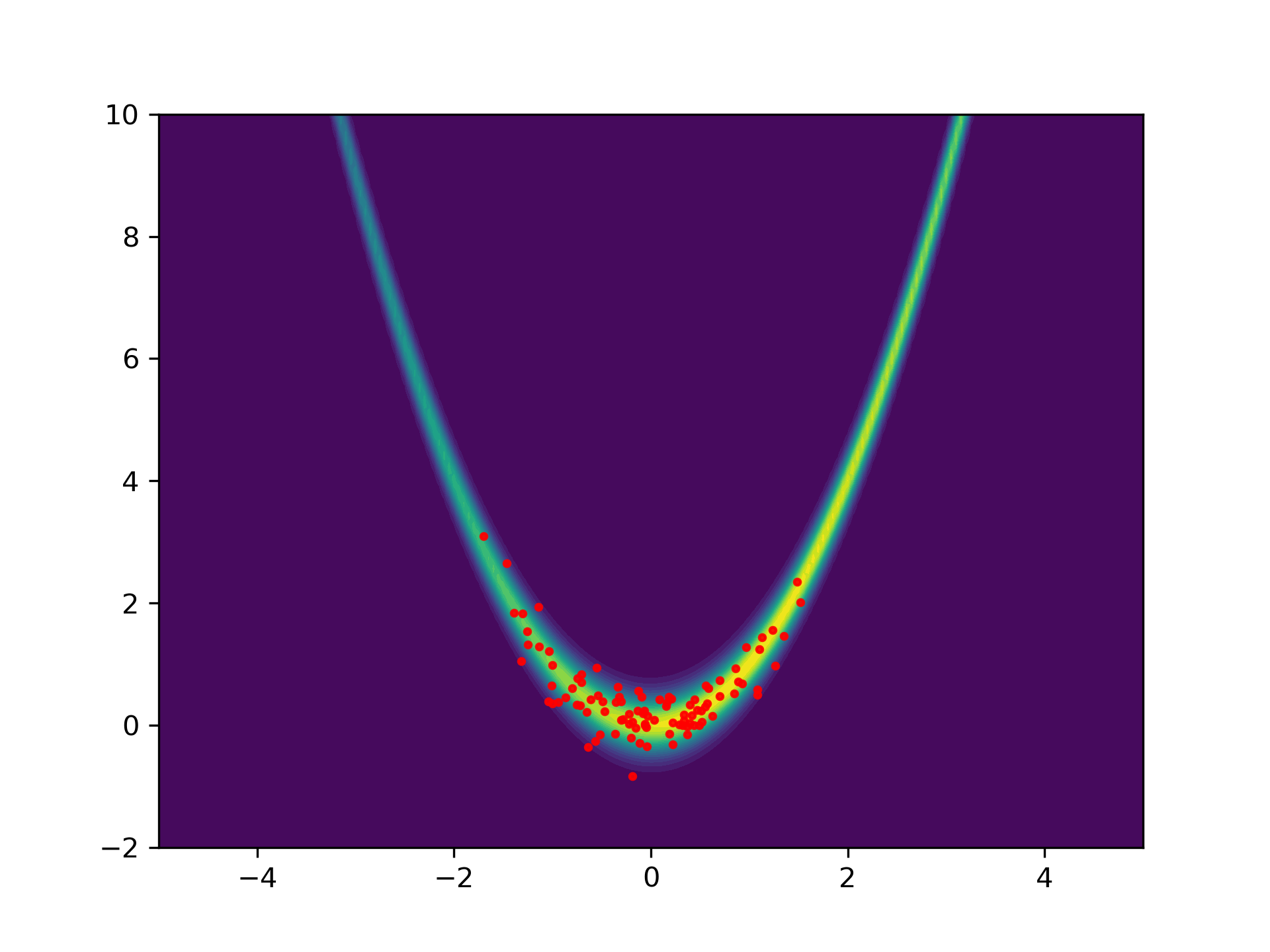}&
 \adjincludegraphics[width=0.35\textwidth,trim={1.3cm 0.7cm 1.7cm 1.2cm},clip,valign=c,decodearray={0 1 0 1 0.5 1}]{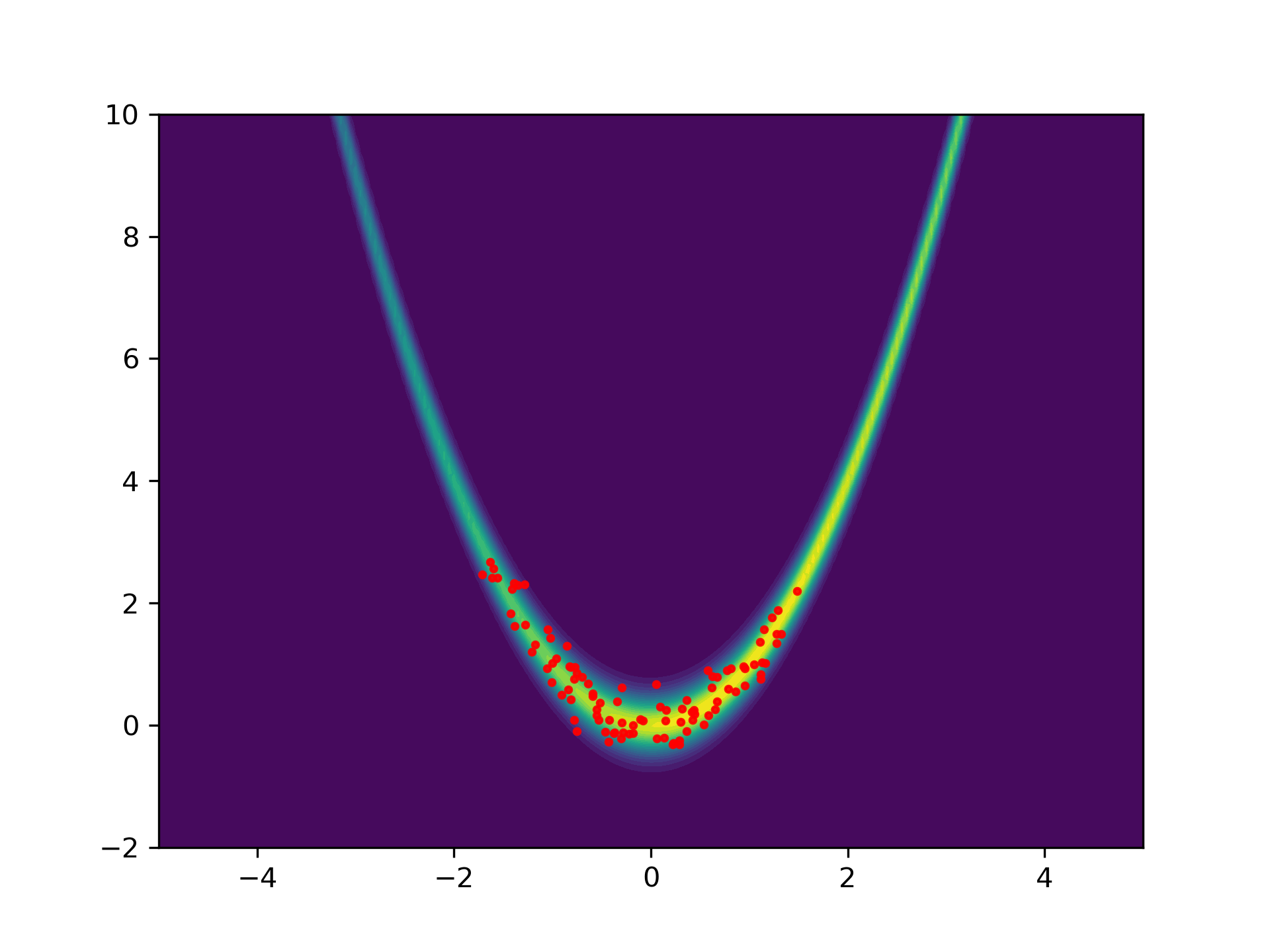}\\

 \adjincludegraphics[width=0.35\textwidth,trim={1.3cm 0.7cm 1.7cm 1.2cm},clip,valign=c,decodearray={0 1 0 1 0.5 1}]{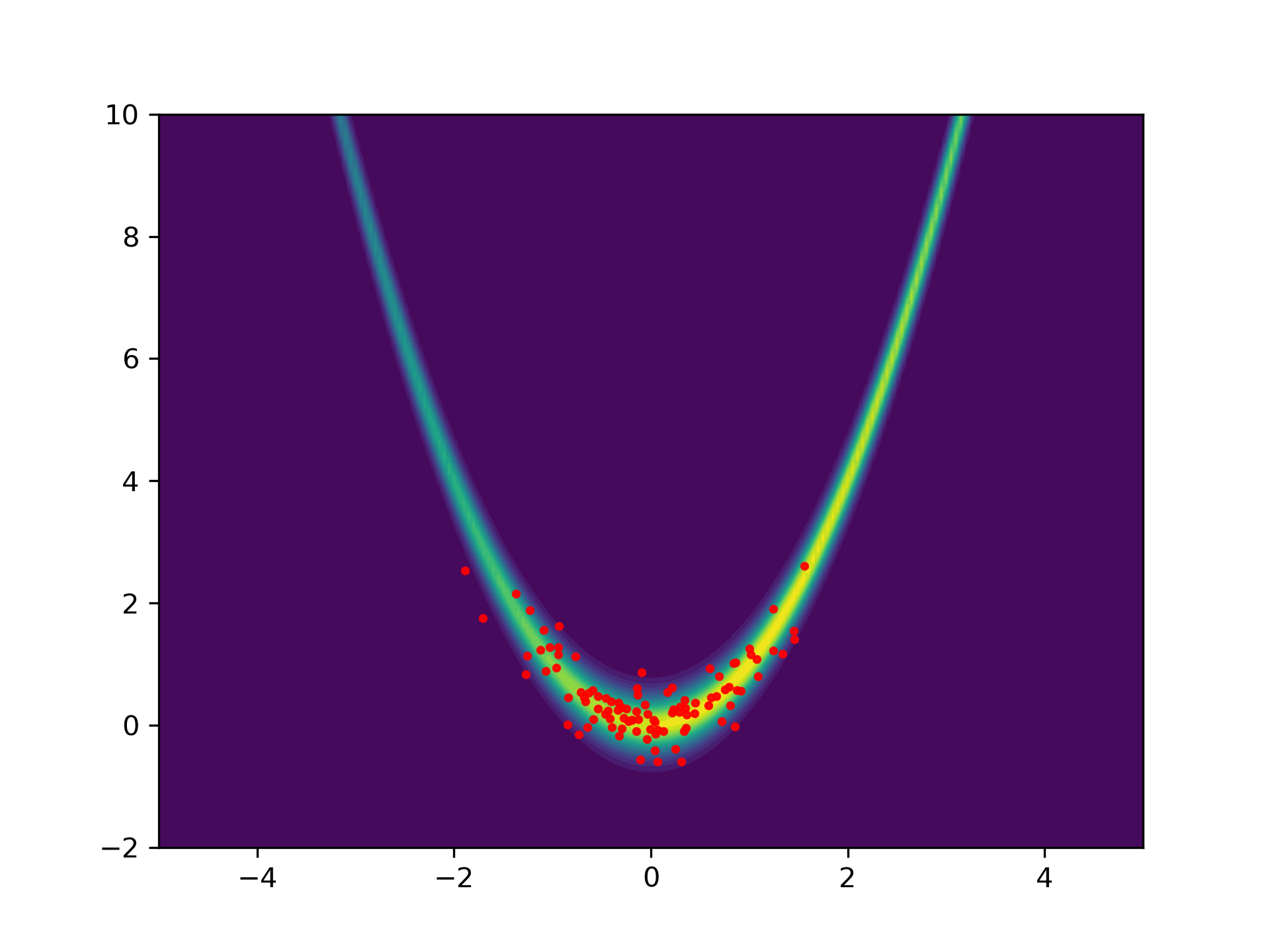}&
 \adjincludegraphics[width=0.35\textwidth,trim={1.3cm 0.7cm 1.7cm 1.2cm},clip,valign=c,decodearray={0 1 0 1 0.5 1}]{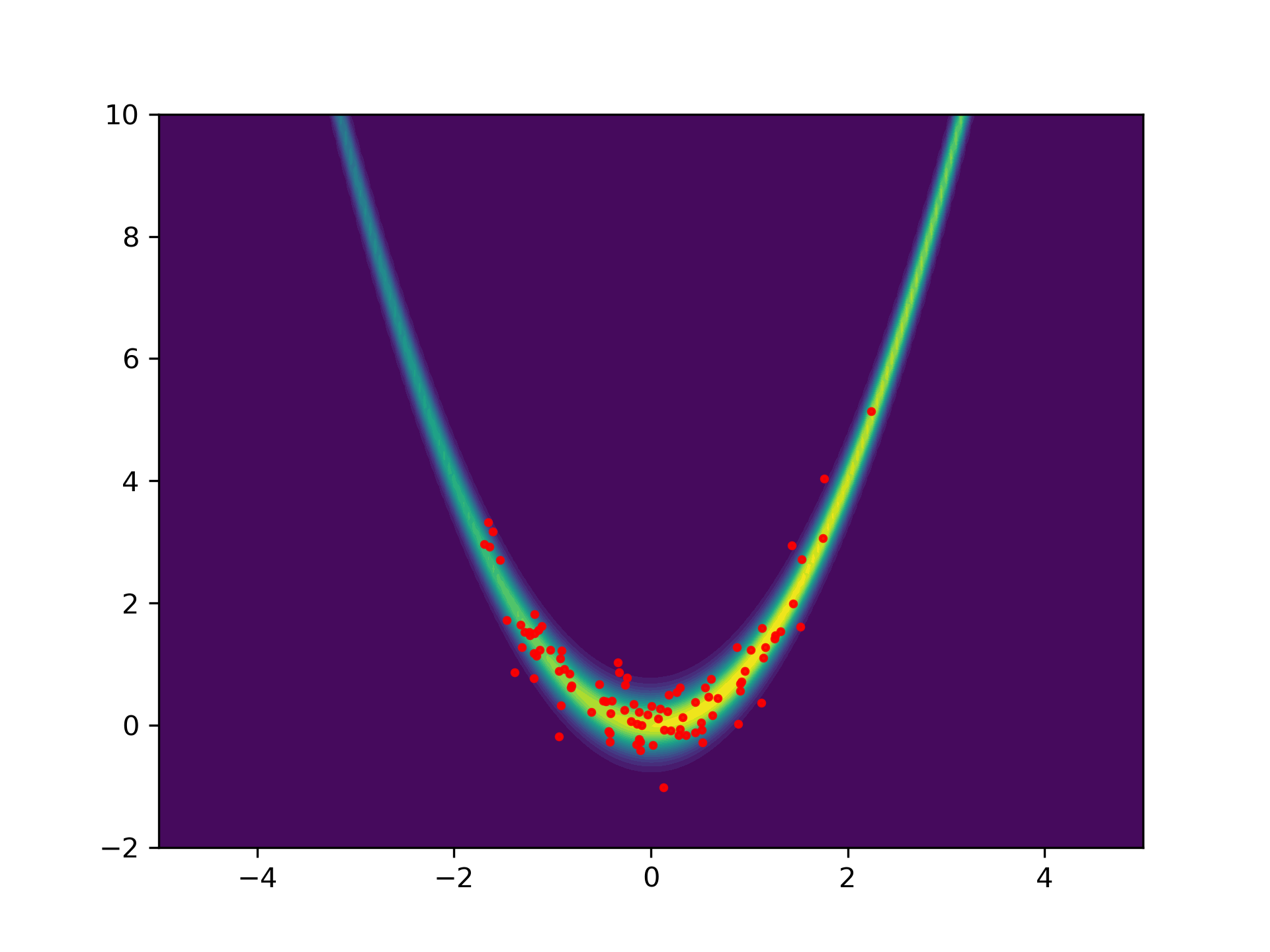}&
 \adjincludegraphics[width=0.35\textwidth,trim={1.3cm 0.7cm 1.7cm 1.2cm},clip,valign=c,decodearray={0 1 0 1 0.5 1}]{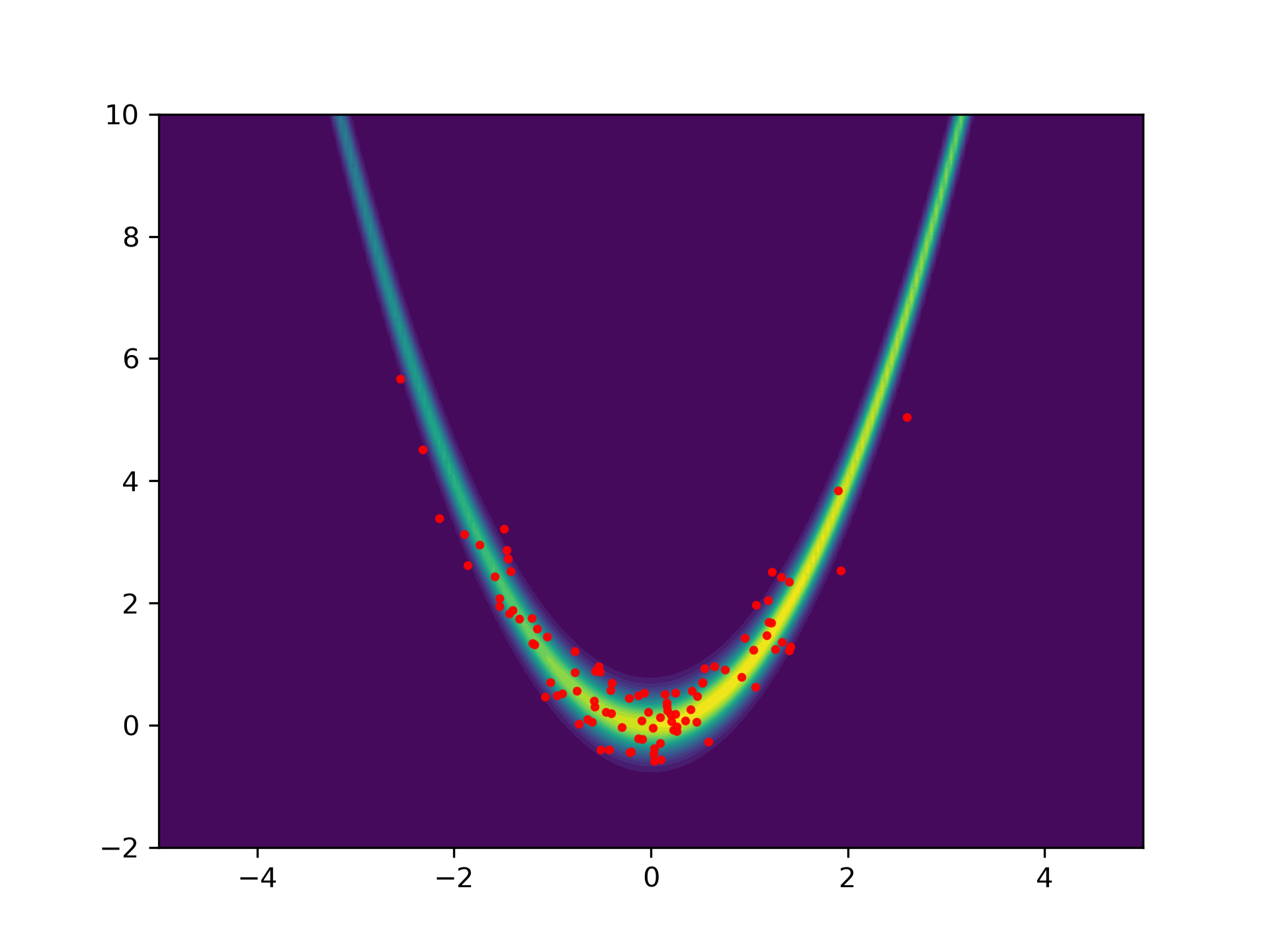}\\
\end{tblr}}
\caption{Particle evolution for Rosenbrock distribution at iterations 50, 200, 500. Top to bottom: (1) ARWP-Nesterov ($T = \eta = 0.02$), (2) ``underdamped ILA'' ($\eta = 0.05$, damping parameter $=2$), (3) KLMC ($\eta = 0.01$, damping parameter $=5$), and (4) ULA ($\eta = 0.01$). We observe that ARWP and ILA are better at exploring the tails than KLMC. However, ILA and ULA both have particles straying away from the main parabola due to time-discretization bias.}\label{fig:rosenbrock}
\end{figure}

\Cref{fig:rosenbrock} plots the evolution of ARWP-Nesterov, ILA with a low damping parameter, KLMC, and ULA. We observe that ARWP-Nesterov is able to properly diffuse along the parabolic potential well, while keeping an appropriate number of particles near the origin. \Cref{appsec:rosenbrock} contains some additional comparisons with other baselines ILA and BRWP, as well as a hyperparameter ablation for ARWP-Heavy-ball.

\subsection{Multi-Modal Gaussian Mixture}\label{ssec:gmm}
We now consider a four-mode weighted Gaussian mixture in two dimensions. In this case, the potential is given by 
\begin{equation*}
    V(x) = -\log \left[-\sum_{i=1}^4  w_i \exp(-{2 \pi \sigma_i^2}\left\|\frac{x - c_i}{\sigma_i^2}\right\|^2)\right],
\end{equation*}
where the centers are given by $\{c_i\} = \{(0,0),\,(3,0),\, (-3,-1),\, (-3,1)\}$, weights $\{w_i\} = \{1, 0.5, 0.5, 0.5\}$, and bandwidths $\{\sigma_i^2\} = \{0.5, 0.25, 0.25, 0.25\}$. This potential is a large well at the origin, with a smaller well on one side, and two smaller wells on the other side. This is run with 100 particles, with initial distribution $\gN((3,0), I_2)$ in a suboptimal well. The methods are run for 400 iterations to allow for sufficient mixing.

\Cref{fig:KL_GMM} gives the evolution of the KL divergence of the various methods, run with optimal parameters as found using a grid search to minimize divergence at iteration 400. The KL divergence is approximated using a Gaussian KDE with bandwidth 0.1, and integrated over the grid $[-5,5]^2$ with mesh size $\Delta x=0.01$. We observe that the KL divergence of the ARWP methods are able to decrease faster than the compared Langevin methods as well as BRWP. Moreover, the combination of the acceleration and modified kernel allows for the particles to diffuse into all the potential wells in a structured manner, leading to a lower terminal divergence. 

\Cref{fig:gmmParticles} plots the particles at iterations 10, 50, and 200 for ARWP-Heavy-ball, ILA, KLMC and MALA. Each of the methods are able to diffuse to the opposite wells by iteration 200. We observe again a structured phenomenon for ARWP across all of the different wells. 

\begin{figure}[ht]
    \centering
    \includegraphics[width=0.9\linewidth]{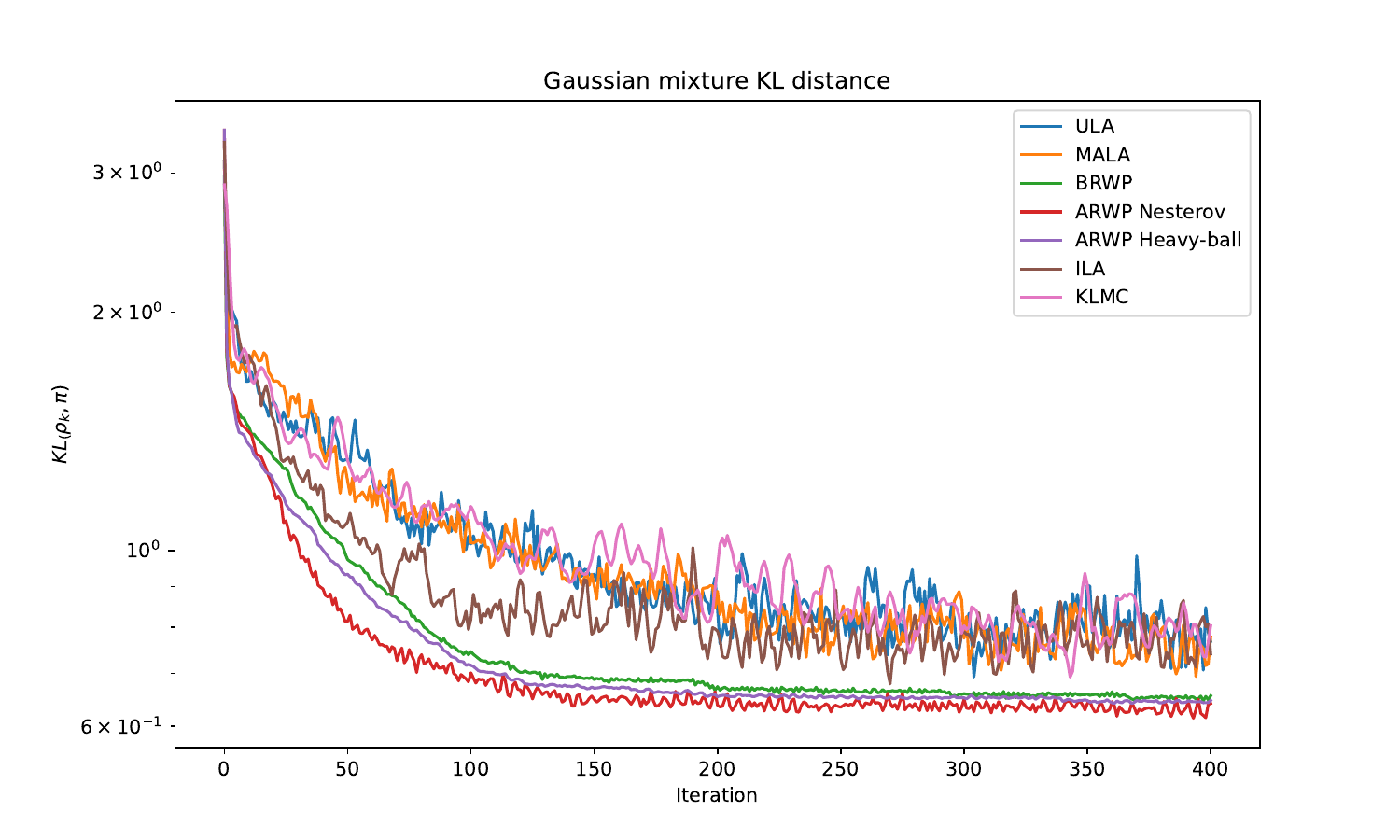}
    \caption{KL divergence between the particles and the underlying distribution for the Gaussian mixture. We observe that ARWP method converges faster than BRWP, and the particle cloud stabilizes with lower KL divergence than the corresponding Langevin methods. KLMC experiences mode collapse, which persists through hyperparameter changes.}
    \label{fig:KL_GMM}
\end{figure}

\begin{figure}[ht]
\centering

\renewcommand{\arraystretch}{0}
\noindent\makebox[0.9\textwidth]{
\begin{tblr}{
    colspec={ccc}, colsep=1pt
    }
 Iter. 10 & 50 & 200\\\hline
 \adjincludegraphics[width=0.35\textwidth,trim={1.3cm 0.7cm 1.7cm 1.2cm},clip,valign=c,decodearray={0 1 0 1 0.5 1}]{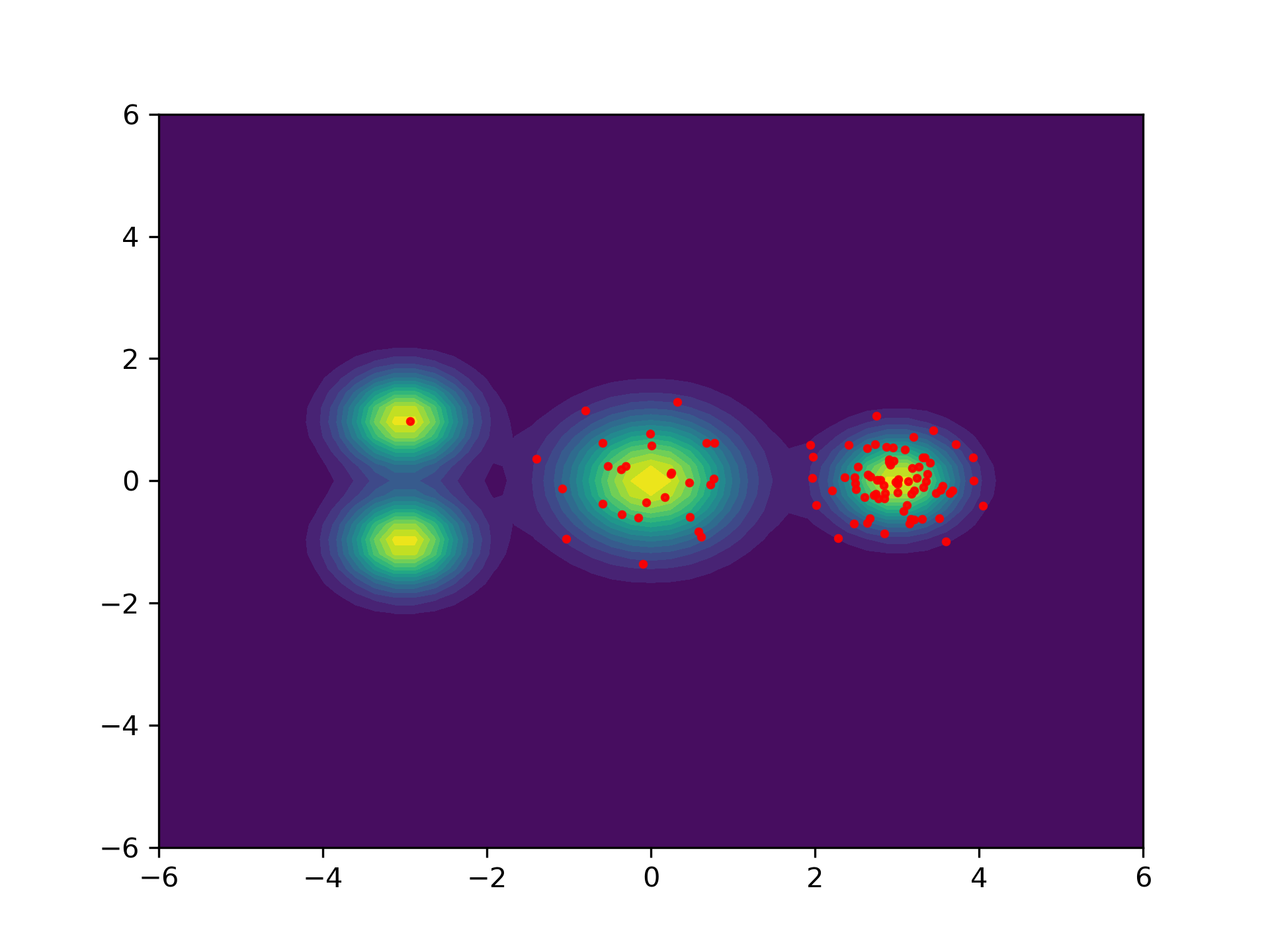}&
 \adjincludegraphics[width=0.35\textwidth,trim={1.3cm 0.7cm 1.7cm 1.2cm},clip,valign=c,decodearray={0 1 0 1 0.5 1}]{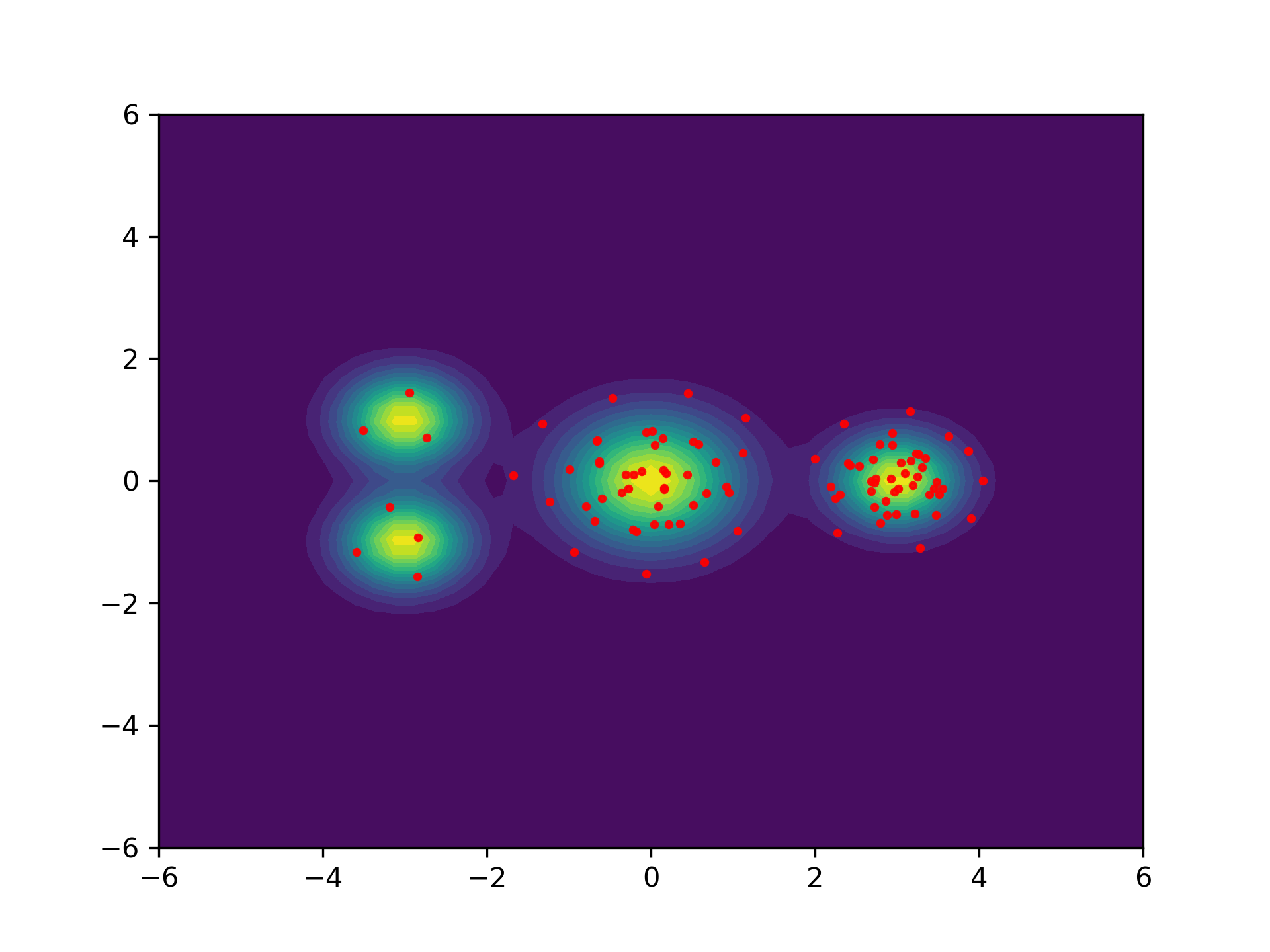}&
 \adjincludegraphics[width=0.35\textwidth,trim={1.3cm 0.7cm 1.7cm 1.2cm},clip,valign=c,decodearray={0 1 0 1 0.5 1}]{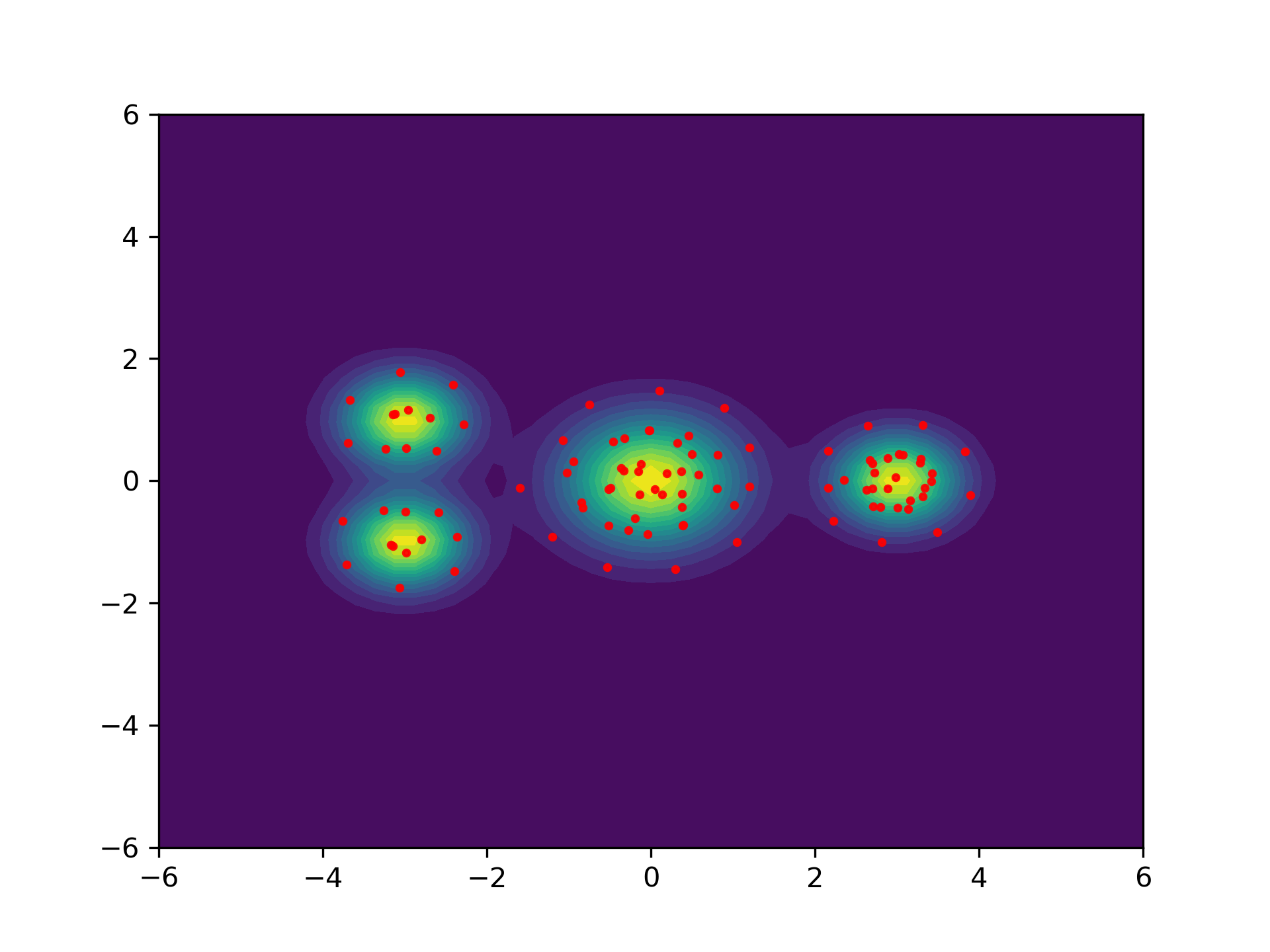}\\
 
 \adjincludegraphics[width=0.35\textwidth,trim={1.3cm 0.7cm 1.7cm 1.2cm},clip,valign=c,decodearray={0 1 0 1 0.5 1}]{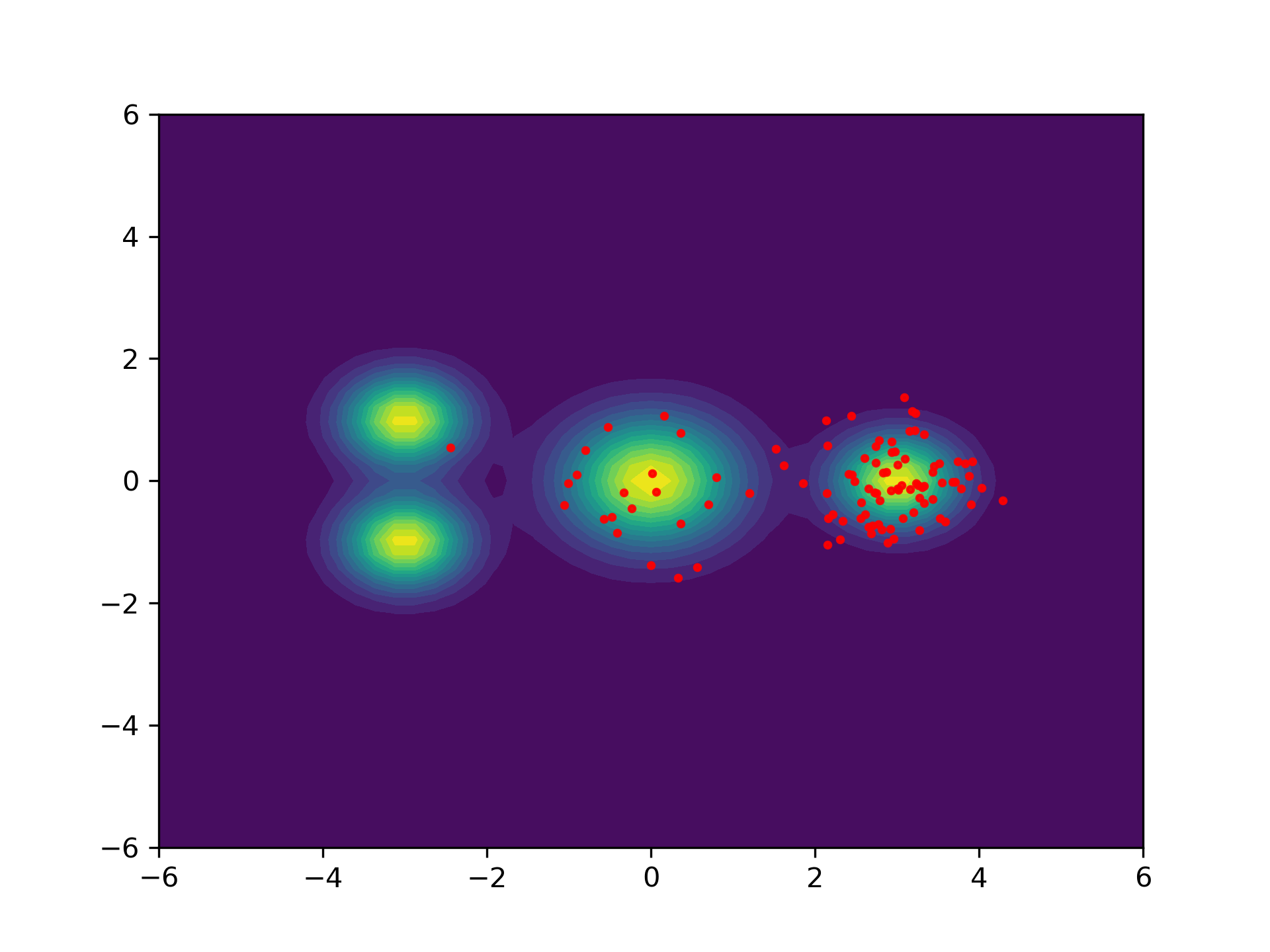}&
 \adjincludegraphics[width=0.35\textwidth,trim={1.3cm 0.7cm 1.7cm 1.2cm},clip,valign=c,decodearray={0 1 0 1 0.5 1}]{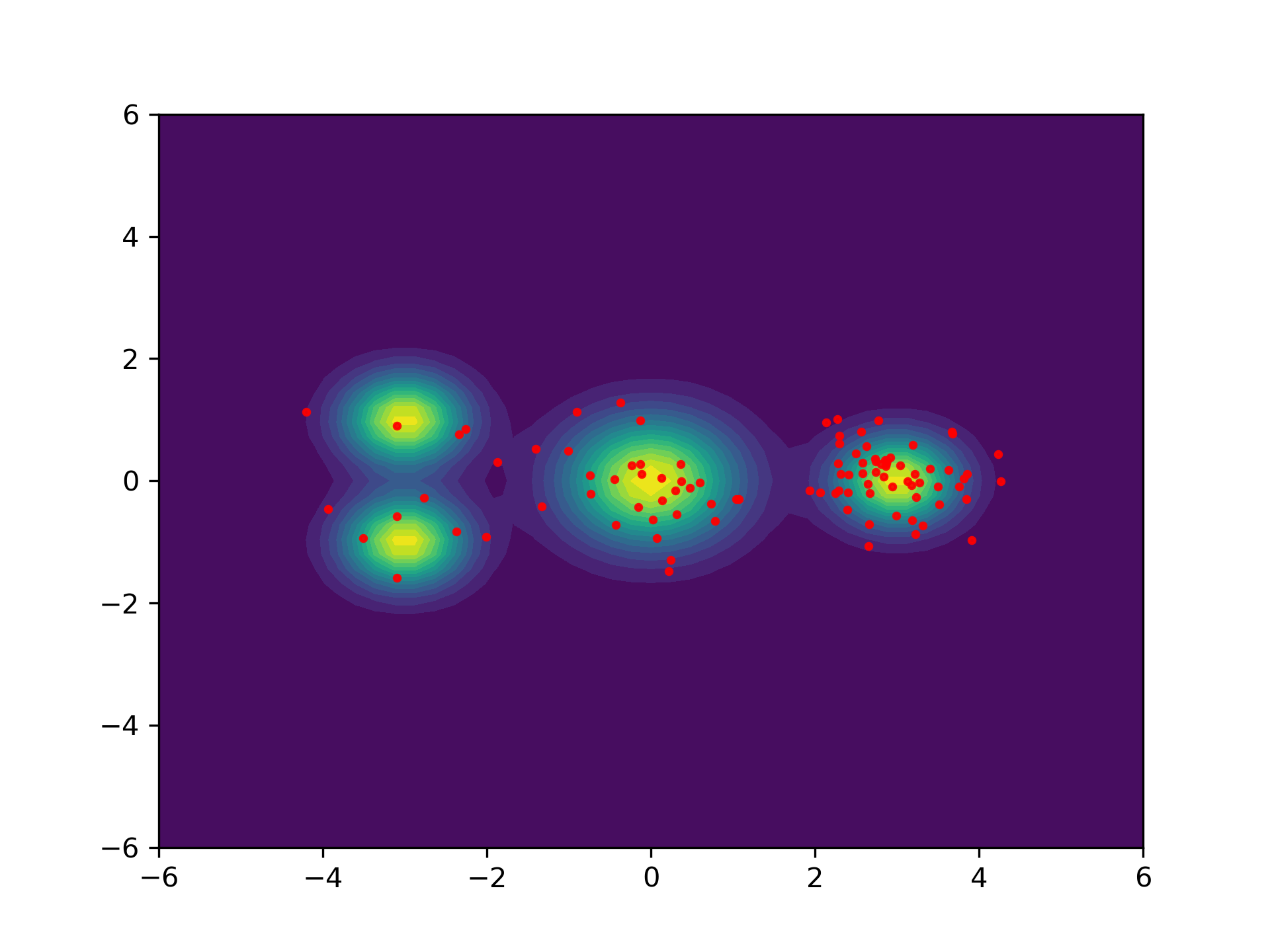}&
 \adjincludegraphics[width=0.35\textwidth,trim={1.3cm 0.7cm 1.7cm 1.2cm},clip,valign=c,decodearray={0 1 0 1 0.5 1}]{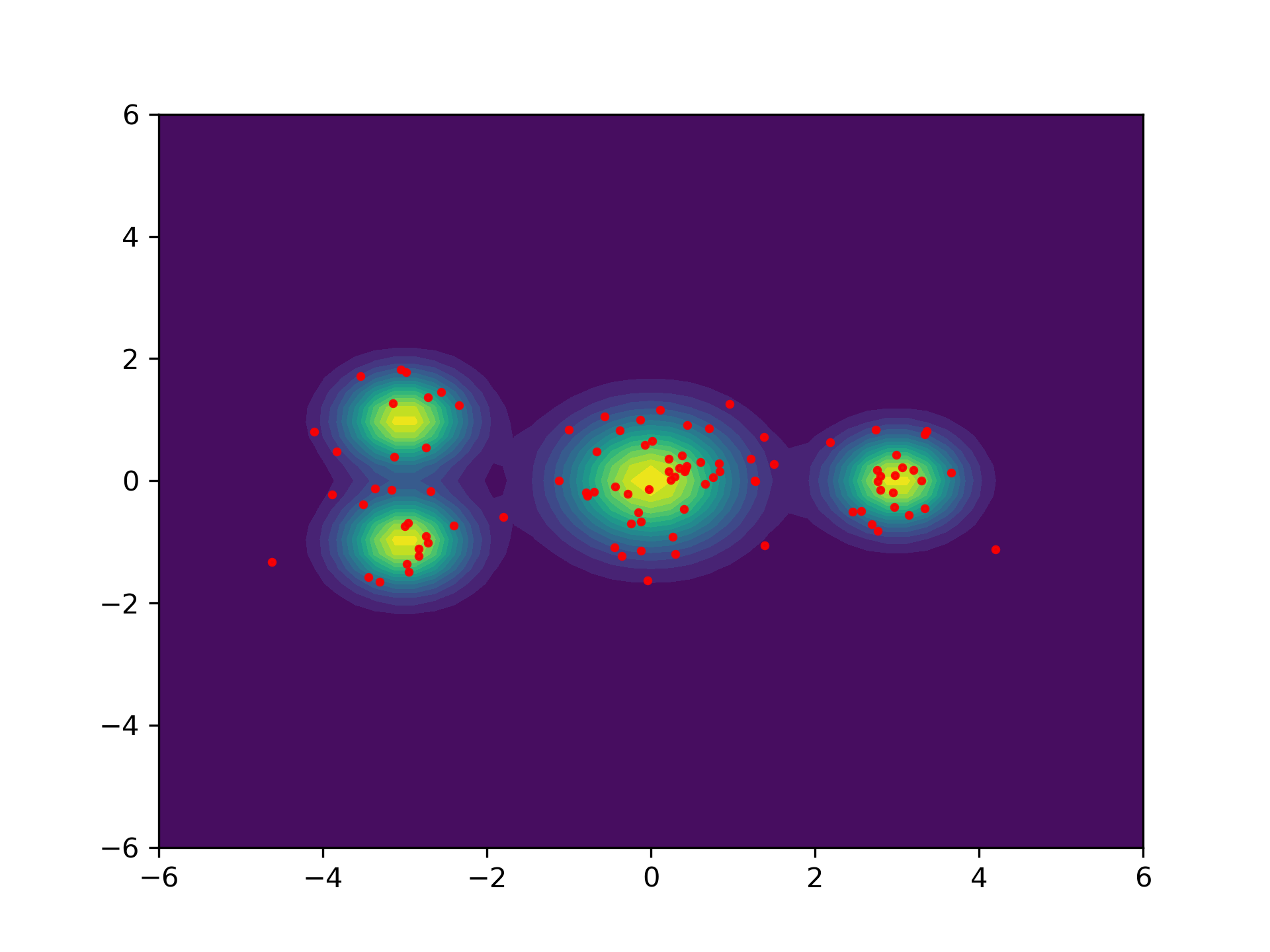}\\

 \adjincludegraphics[width=0.35\textwidth,trim={1.3cm 0.7cm 1.7cm 1.2cm},clip,valign=c,decodearray={0 1 0 1 0.5 1}]{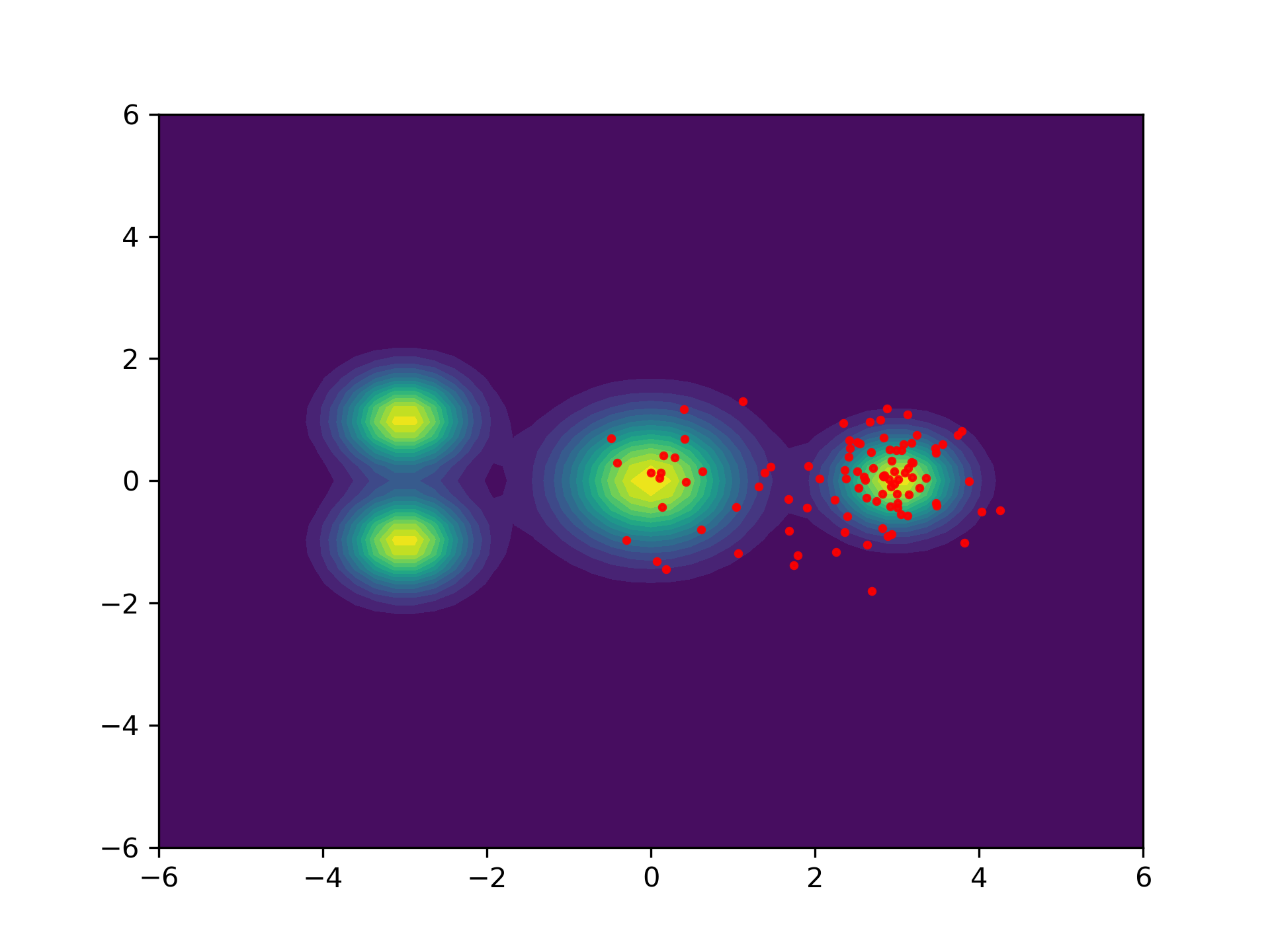}&
 \adjincludegraphics[width=0.35\textwidth,trim={1.3cm 0.7cm 1.7cm 1.2cm},clip,valign=c,decodearray={0 1 0 1 0.5 1}]{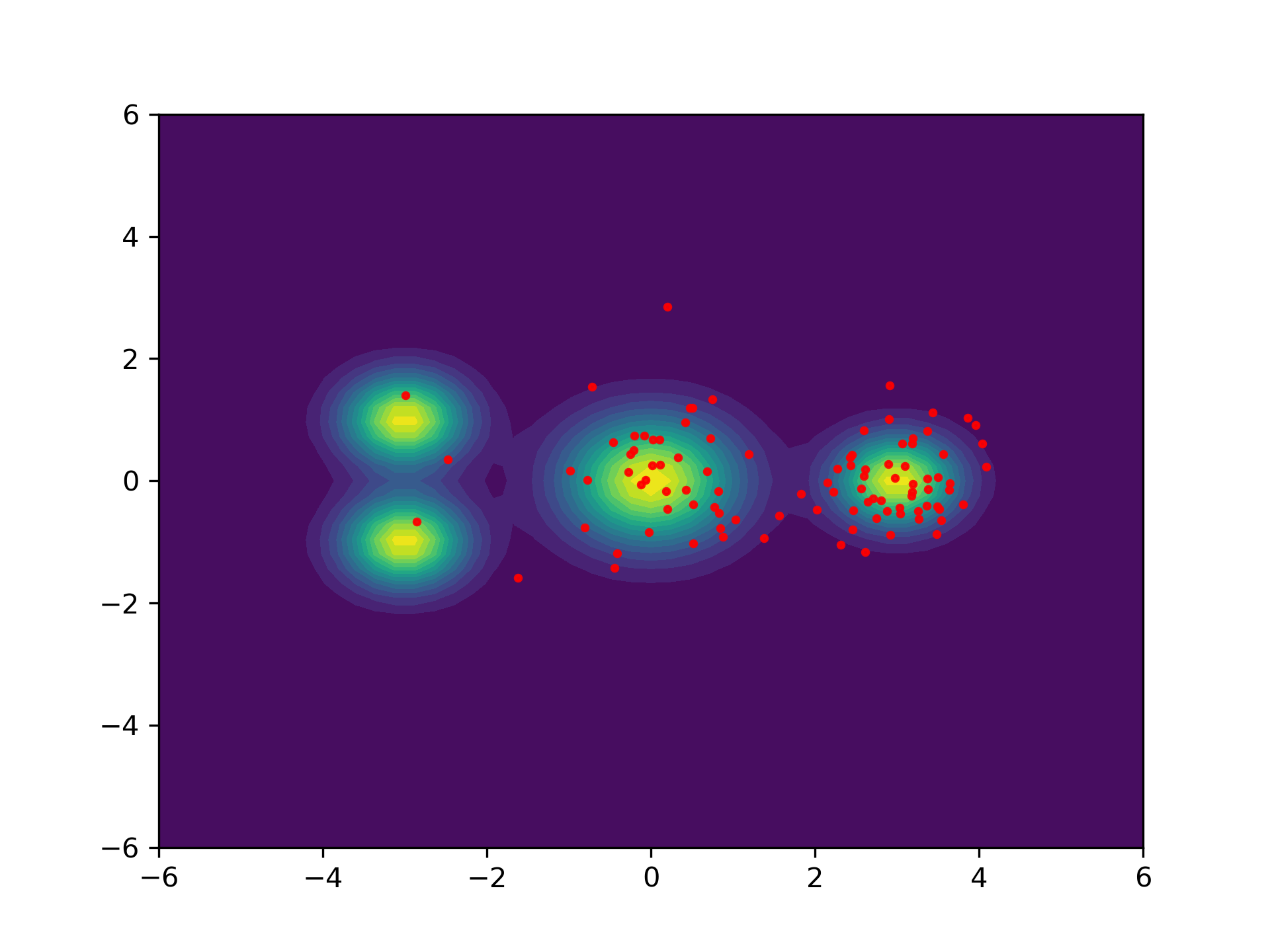}&
 \adjincludegraphics[width=0.35\textwidth,trim={1.3cm 0.7cm 1.7cm 1.2cm},clip,valign=c,decodearray={0 1 0 1 0.5 1}]{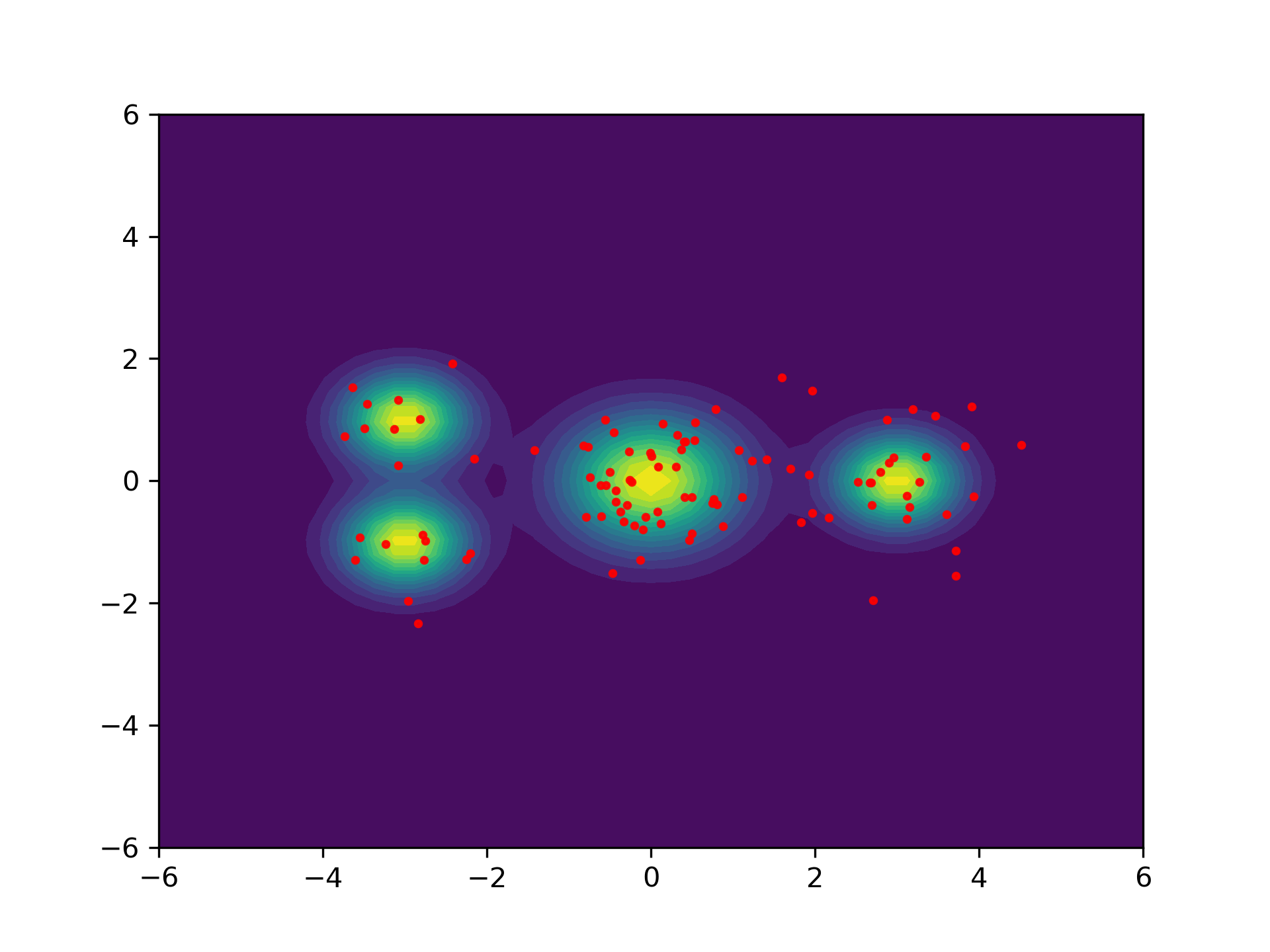}\\

 \adjincludegraphics[width=0.35\textwidth,trim={1.3cm 0.7cm 1.7cm 1.2cm},clip,valign=c,decodearray={0 1 0 1 0.5 1}]{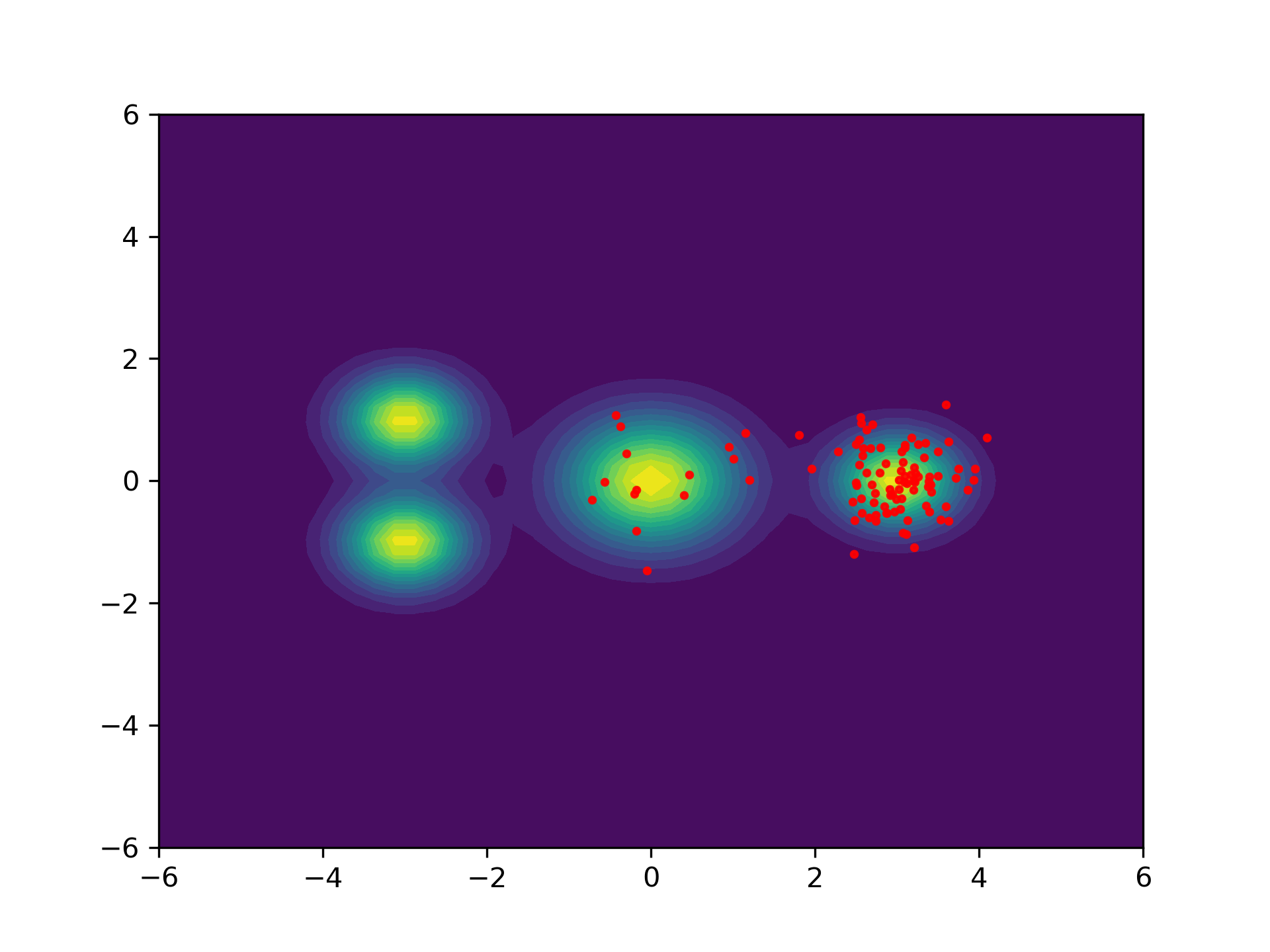}&
 \adjincludegraphics[width=0.35\textwidth,trim={1.3cm 0.7cm 1.7cm 1.2cm},clip,valign=c,decodearray={0 1 0 1 0.5 1}]{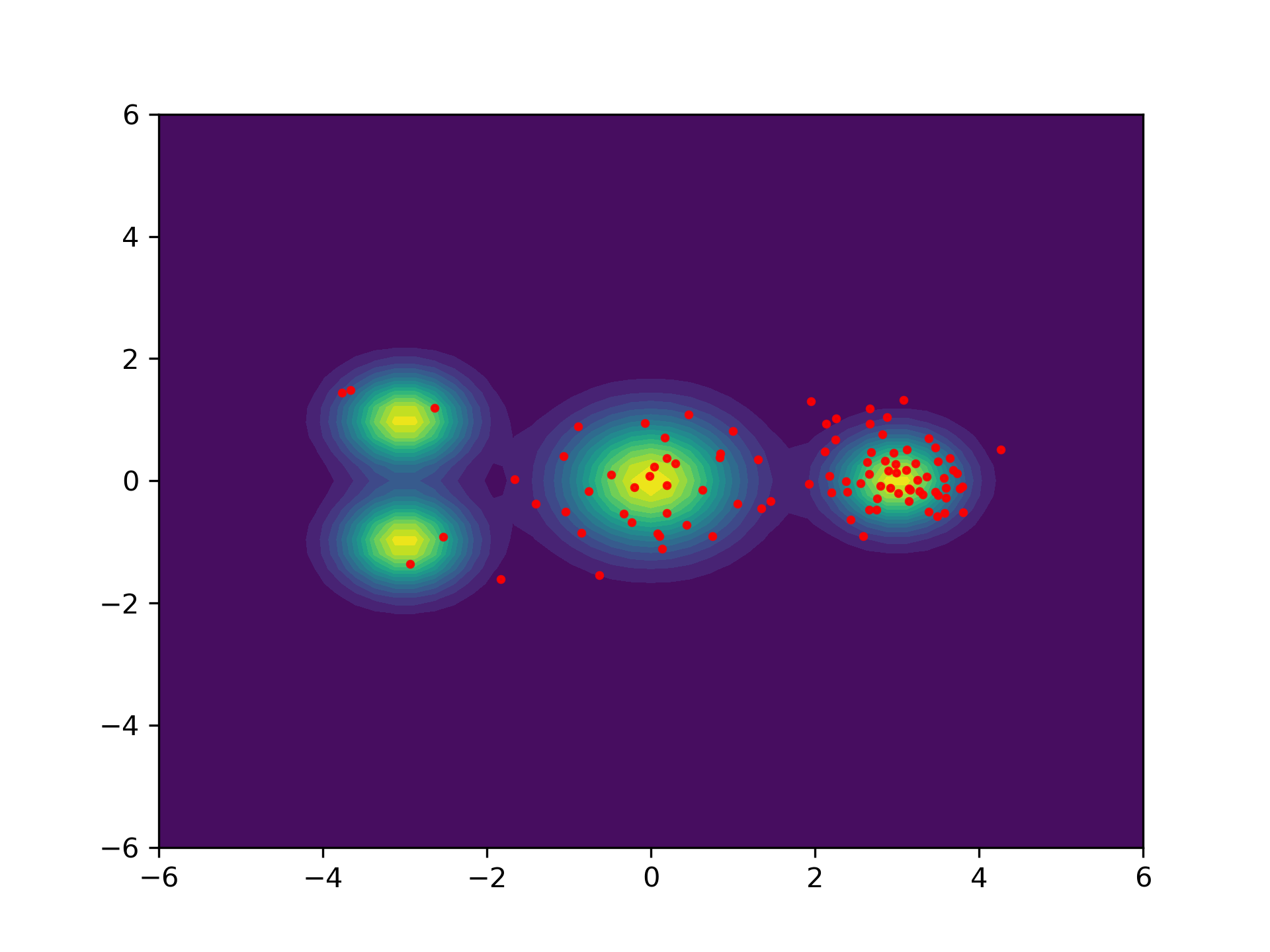}&
 \adjincludegraphics[width=0.35\textwidth,trim={1.3cm 0.7cm 1.7cm 1.2cm},clip,valign=c,decodearray={0 1 0 1 0.5 1}]{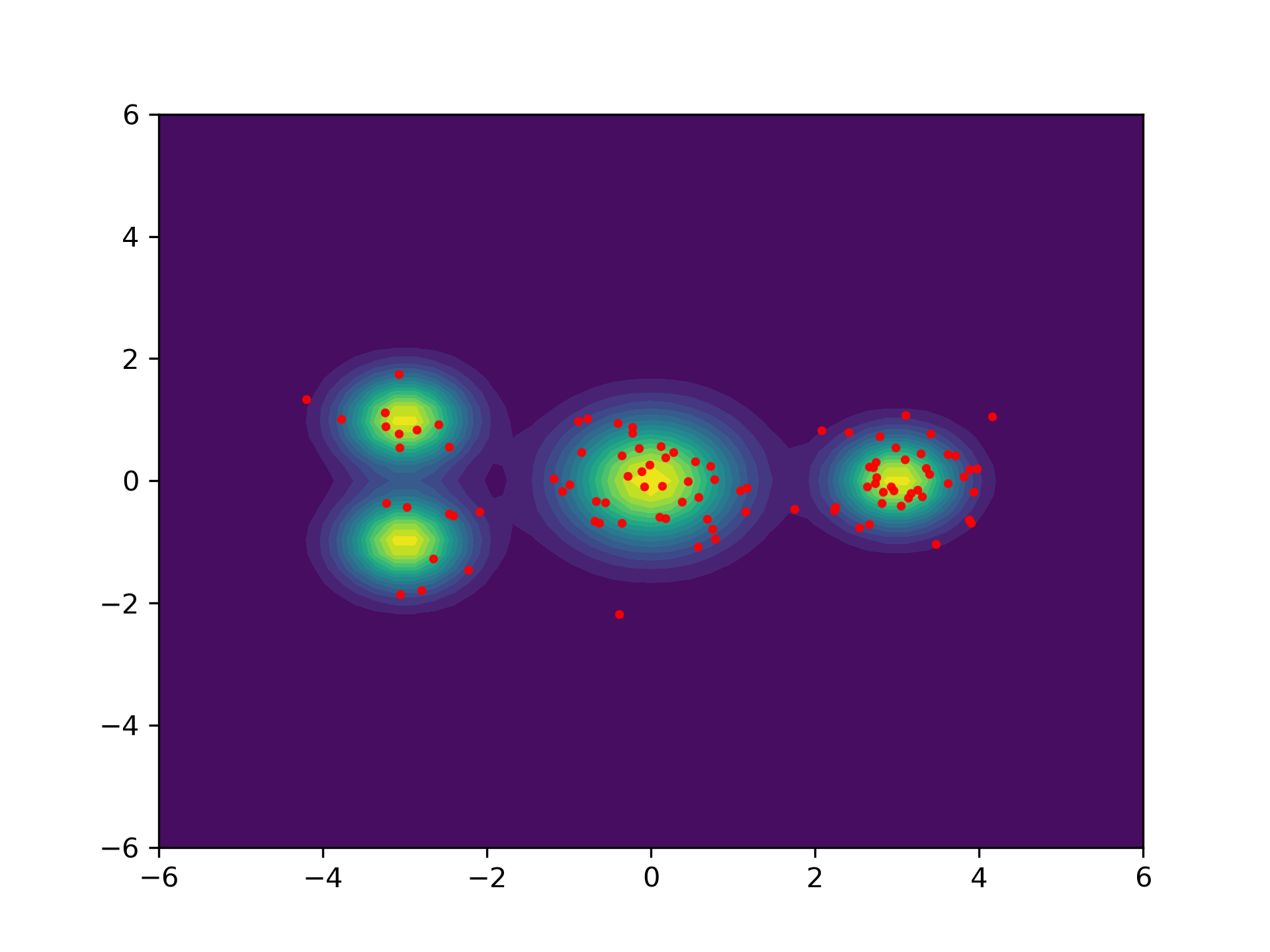}
\end{tblr}}
\caption{Top to bottom: particles for (1) ARWP-Heavy-ball ($\eta = 0.6,\ T=0.1,\ a=1$), (2) ILA ($\eta = 0.3$, damping $=1$), (3) KLMC ($\eta = 1,\ a=1$), and (4) MALA ($\eta = 0.3$), observed at iterations 10, 50, and 200. We observe reasonable mixing aside from KLMC, and a typical structured phenomenon at iteration 200 for ARWP. KLMC appears to exhibit larger bias around the rightmost well. Note that the particles have not stabilized at this point, and continue to flow from the right potential wells to the left wells.}\label{fig:gmmParticles}
\end{figure}

\subsection{Bayesian Neural Networks}\label{ssec:bnn}
For a high-dimensional non-log-concave target distribution, we consider the Bayesian neural network experiment as done in \cite{wang2022accelerated,tan2024noise,tan2025preconditioned}. This consists of training some neural networks over five UCI\footnote{\url{https://archive.ics.uci.edu/datasets}} datasets. We compare against existing baselines given by various gradient flows.

Each particle is represented by a two-hidden-layer ReLU neural network, each with 50 neurons and default Gaussian initialization. The epoch and batch-size hyperparameters are taken as in \cite{wang2022accelerated}, and the methods are run with $N=10$ ``particles''. The step-size and $T$ hyperparameters are chosen by a grid search ranging from $\eta \in [2\times 10^{-2}, 3 \times 10^{-1}]$ and $T \in [10^{-3}, 10^{-2}]$. Values reported are averaged over 20 independent runs.

\Cref{tab:RMSE} reports the test root-mean-square error for training with BRWP, accelerated information gradient (AIG), Wasserstein gradient flow \cite{wang2022accelerated}, and Stein variational gradient descent \cite{wang2019stein}. We observe that ARWP(-Nesterov) is able to consistently outperform BRWP on this task. Compared with Adam, the particle-based methods are also able to find networks with better generalization.

\begin{table}[ht]
\centering
\caption{Test root-mean-square-error (RMSE) on test datasets on various Bayesian neural network tasks, averaged over 20 runs. Bold indicates smallest in row, underlined denotes second smallest. ARWP(-Nesterov) consistently performs better than BRWP on this task, with a higher variance. This indicates that the particles are able to find better test-generalization, at the cost of also finding some poorer particles.}
\label{tab:RMSE}
\begin{adjustbox}{width=1\textwidth}
\begin{tabular}{@{}cr|rrrrr@{}}
\toprule
Dataset & \multicolumn{1}{c}{Adam}&\multicolumn{1}{c}{ARWP} & \multicolumn{1}{c}{BRWP} & \multicolumn{1}{c}{AIG} & \multicolumn{1}{c}{WGF} & \multicolumn{1}{c}{SVGD} \\ \midrule
Boston   & $3.350_{\pm8.33\mathrm{e}-1}$& $2.902_{\pm 7.25\mathrm{e}-1}$ & $3.309_{\pm 5.31\mathrm{e}-1}$ & $2.871_{\pm 3.41\mathrm{e}-3}$ & $3.077_{\pm 5.52\mathrm{e}-3}$ & $\pmb{2.775_{\pm 3.78\mathrm{e}-3}}$ \\
Combined & $3.971_{\pm1.79\mathrm{e}-1}$&$\pmb{3.939_{\pm 1.89\mathrm{e}-1}}$& ${3.975_{\pm 3.94\mathrm{e}-2}}$ & $4.067_{\pm 9.27\mathrm{e}-1}$ & $4.077_{\pm 3.85\mathrm{e}-4}$ & ${4.070_{\pm 2.02\mathrm{e}-4}}$ \\
Concrete & $4.698_{\pm4.85\mathrm{e}-1}$&$\pmb{4.257_{\pm 8.46\mathrm{e}-1}}$ & $4.478_{\pm 2.05\mathrm{e}-1}$ & ${4.440_{\pm 1.34\mathrm{e}-1}}$ & $4.883_{\pm 1.93\mathrm{e}-1}$ & $4.888_{\pm 1.39\mathrm{e}-1}$ \\
Kin8nm   & $0.089_{\pm 2.72\mathrm{e}-3}$ & $\underline{0.089}_{\pm2.47\mathrm{e}-3}$ & ${\underline{0.089}_{\pm 6.06\mathrm{e}-6}}$ & $0.094_{\pm 5.56\mathrm{e}-6}$ & $0.096_{\pm 3.36\mathrm{e}-5}$ & $0.095_{\pm 1.32\mathrm{e}-5}$ \\
Wine     & $0.629_{\pm4.01\mathrm{e}-2}$ & $0.608_{\pm 3.43\mathrm{e}-2}$ & $0.623_{\pm 1.35\mathrm{e}-3}$ & $0.606_{\pm 1.40\mathrm{e}-5}$ & $0.614_{\pm 3.48\mathrm{e}-4}$ & $\pmb{0.604_{\pm 9.89\mathrm{e}-5}}$ \\ \bottomrule
\end{tabular}
\end{adjustbox}
\end{table}

\section{Discussion}
This work introduces the accelerated regularized Wasserstein proximal (ARWP) method for sampling from a target distribution. There are several accelerated schemes in probability density space. One is from overdamped Langevin to kinetic Langevin dynamics. The other is to add a momentum variable to the score-based ODE. The ARWP method then arises by replacing the score in the latter accelerated information gradient flow with a computationally tractable kernel approximation, given by the regularized Wasserstein proximal operator. For quadratic target potentials, we provide a detailed Lyapunov analysis in terms of the damping parameter in continuous time and an asymptotic discrete-time mixing rate via linearization. Moreover, we achieve a faster asymptotic contraction rate than that of kinetic Langevin dynamics. Experiments demonstrate better tail exploration than accelerated Langevin methods and the characteristic structured-particle phenomenon. 



In a similar vein to the fast iterative shrinkage thresholding algorithm \cite{beck2009fast}, one may ponder whether or not a similar acceleration can hold using the (unregularized) Wasserstein proximal in order to accelerate the Wasserstein proximal gradient method \cite{salim2020wasserstein}. While a Wasserstein proximal gradient method can be written down using some appropriate exponential maps, and acceleration for geodesically convex functions on manifolds can exist \cite{liu2017accelerated}, acceleration on Wasserstein manifolds has not been explored in the literature. One possible direction would be to consider applying RWPO within the Wasserstein proximal gradient algorithm \cite{salim2020wasserstein}, relating the RWPO-based methods with classical proximal descent algorithms. The relationship between FISTA with an added score term, with the corresponding dynamics in density space, is also an open question. 

\subsection*{Acknowledgments} 
H.Y. Tan is partially supported by AFOSR YIP award No. FA9550-23-1-0087. H.Y. Tan and S. Osher are partially
funded by AFOSR MURI FA9550-18-502, ONR N00014-20-1-2787, and NSF 443948-SN-2199. W. Li is partially supported by the
AFOSR YIP award No. FA9550-23-1-0087, NSF DMS-2245097, and NSF RTG: 2038080.
\bibliographystyle{plain}
\bibliography{refs}
\appendix
\section{Underdamped Langevin equation}\label{app:underdampedLangevin}
One possible accelerated counterpart to the Langevin equation is the so-called \emph{kinetic equation} \cite{dalalyan2020sampling}. This corresponds to the \emph{underdamped Langevin diffusion} \cite{cheng2018underdamped}. This is the equivalent of Nesterov acceleration in the gradient space \cite{wang2022accelerated}. We can write down the the standard (overdamped) Langevin diffusion, given by 
\begin{equation*}
    \dd{X}_t = -\nabla V(X) \dd{t} + \sqrt{2} \dd{W}.
\end{equation*}

The underdamped Langevin dynamics is then given by the following, where $X$ and $P$ are spatial and momentum parameters respectively,

\begin{equation}
    \dd \begin{bmatrix}
        X \\
        P
    \end{bmatrix} = \begin{bmatrix}
        P\\
        -(a P + u \nabla V(X))
    \end{bmatrix} \dd{t} + \sqrt{2au} \begin{bmatrix}
        0\\
        I
    \end{bmatrix} \dd{W},
\end{equation}
where $a>0$ is a friction coefficient, and $u>0$ is an inverse mass. In this case, the Brownian motion can be seen to only act on the momentum variable $P$. In the case $u=1$, if we scale $a$ to infinity, the limit of the kinetic Langevin dynamics yields the standard overdamped Langevin diffusion \cite{nelson1967dynamical}. The distribution of this diffusion converges to its invariant distribution in $\R^{2d}$,
\begin{equation*}
    f_*(x,p) \propto \exp(-V(x) - \frac{1}{2u} \|p\|^2).
\end{equation*}
The corresponding accelerated Fokker--Planck equation is known as the \emph{Klein--Kramers equation}, which is an evolution of the joint density in phase space $f(x,p)$. The update is given by the second-order update \cite{risken1989fokker}
{
\begin{equation}
    \partial_t f + p\cdot \nabla_x f-u^{-1}\nabla_p f\cdot\nabla V(x)=  a \nabla_p \cdot(pf) + u^{-1} a \Delta^2_p f.
\end{equation}}

Various convergence results can be found in \cite{leimkuhler2024contraction}. This can also be seen as Hamilton's equations corresponding to the Hamiltonian
\begin{equation*}
    H(x,p) = V(x) + \frac{1}{2u} \|p\|^2.
\end{equation*}
An alternative is given by considering a different accelerated Fokker--Planck equation, with the same stationary distribution. In the continuous-time optimization setting, an accelerated gradient flow is given by 
\begin{equation*}
    \dot x = p, \quad \dot p = -ap - \nabla V(x).
\end{equation*}
The analogous dynamics in the probability space are given by \cite{chen2025accelerating}, referred to by the authors as \emph{heavy-ball flow},
\begin{equation}\label{eq:acceleratedFP}
    \partial_t f + p \cdot \nabla_x f - \nabla_p \cdot\left(\left(ap + \nabla_x \frac{\delta E}{\delta \rho}\left[\int_{\R^d} f(x,p) \dd{p}\right]\right)f\right) = 0.
\end{equation}
Here, $E$ is some divergence or metric to the stationary distribution, such as the relative entropy/KL divergence, and $\delta E/\delta \rho$ represents the first variation. The main difference with the Klein--Kramers equation is the second order term: instead of having a Laplacian in momentum space $\Delta_p f_t(x,p)$ over the joint density $f(x,p)$, one has a mixed gradient $\nabla_p \cdot (f \nabla_x)$ over the marginal $\int f(x,p) \dd{p}$. In the case $E(\rho) = \KL(\rho \|\pi)$, writing $\rho_t$ for the marginal over $p$, the heavy-ball flow is specialized as 
\begin{equation}\label{eq:heavyballKLFP}
    \partial_t f + p \cdot \nabla_x f - \nabla_p \cdot\left[\left(ap + \nabla V(x) + \nabla \log \rho_t(x)\right)f\right] = 0.
\end{equation}
This equation describes the phase-space measure corresponding the following particle evolution \cite[Eq. 94]{chen2025accelerating}
\begin{equation} \label{eq:initHeavyBall}
    \frac{\dd{}}{\dd{t}} \begin{bmatrix}
        X \\
        P
    \end{bmatrix} =  \begin{bmatrix}
        P \\
        -aP-\nabla V(X) - \nabla \log \rho_t(X)
    \end{bmatrix}.
\end{equation}

Other than considering this as an analogue of Nesterov acceleration in measure space, another method of deriving this equation \cref{eq:acceleratedFP} is by damping an appropriate Hamiltonian flow in the Wasserstein-2 space. This is interpreted by \cite{wang2022accelerated} as arising from a measure-valued analog of Nesterov acceleration using the Wasserstein metric. See more details in \cite{wang2022accelerated}. 

\section{Proofs}

\subsection{Convergence rate of linearized discrete time update}\label{app:linearizedDiscreteTime}
This section shows the convergence rate given in \Cref{ssec:DiscreteTimeLinearization}. Recall the update matrix (in each dimension) is given by 
\begin{equation}
    \begin{pmatrix}
        \delta_{n+1}\\
        g_{n+1}
    \end{pmatrix} = [I+\eta A] \begin{pmatrix}
        \delta_{n}\\
        g_{n}
    \end{pmatrix},\quad A = \begin{bmatrix}
    0 & 2\lambda - 4Tk_+^{-1} \\
    -\lambda^{-2} & -a
\end{bmatrix}.
\end{equation}
The eigenvalues of $A$ are 
\begin{align*}
    \chi_\pm &= \frac{1}{2}[\Tr \pm \sqrt{\Tr - 4\det}]\\
    &= \frac{1}{2}\left[-a \pm \sqrt{a^2 - 4 \lambda^{-2}(2\lambda - 4TK_+^{-1})}\right].
\end{align*}
and the step-size has to be chosen such that for every eigenvalue, $|1+\eta \chi_{\pm}|<1$. The contraction/convergence rate is the largest of the values $|1+\eta \chi_{\pm}|$ over all eigenvalues. We recall a technical assumption that $T \le (1+\sqrt{2})^{-1} \lambda_{\min}$. This ensures that the function $\lambda \mapsto \lambda^{-1} \frac{\lambda-T}{\lambda+T}$ is (strictly) decreasing over $[\lambda_{\min}, \lambda_{\max}]$.

\subsubsection{All complex eigenvalues, low critical damping}
This happens if for each eigenvalue $\lambda$ of $\Lambda$,
\begin{equation}\label{eq:a2conditionComplexEig}
    a^2 \le 8\lambda^{-1} (1-2Tk_+^{-1} \lambda^{-1}).
\end{equation}
The step-size condition on $\eta $ is
\begin{equation*}
     \left|1+\frac{\eta }{2}\left[-a\pm\sqrt{a^2-4\lambda^{-2}(2\lambda-4TK_+^{-1})}\right] \right|< 1.
\end{equation*}
The absolute value is less than 1 if and only if
\begin{align*}
    (1-\frac{a\eta }{2})^2 + \frac{\eta^2}{4}(-a^2+4\lambda^{-2}(2\lambda - 4Tk_+^{-1})) &< 1\\
    \Leftrightarrow 1-a\eta +\eta^2\lambda^{-2}(2\lambda-4Tk_+^{-1})&< 1.
\end{align*}
Rearranging, the step-size condition that yields convergence is
\begin{align*}
    \eta  < \frac{a}{2\lambda^{-1} (1-2Tk_+^{-1}\lambda^{-1})}.
\end{align*}
The rate is fastest when the norm of $1+\eta \chi_{\pm}$ is minimized, which occurs at
\begin{equation}\label{eq:hconditionComplexEig}
    \eta =\frac{a}{4\lambda^{-1}(1-2Tk_+^{-1}\lambda^{-1})}.
\end{equation}
Since this has to be true for every eigenvalue, it is sufficient (and necessary) for \labelcref{eq:a2conditionComplexEig} to be true for $\lambda_{\max}$. The maximal step-size is also given by the smallest value of \labelcref{eq:hconditionComplexEig}, i.e. for $\lambda_{\min}$. This yields the parameters
\begin{equation*}
    a = 2\sqrt{2} \lambda_{\max}^{-1/2} \sqrt{1-2Tk_{+,\max}^{-1} \lambda_{\max}^{-1}},\quad \eta  = \frac{1}{\sqrt{2}} \lambda^{-1/2}_{\max}\lambda_{\min} \frac{\sqrt{1-2Tk_{+,\max}^{-1} \lambda_{\max}^{-1}}}{1-2Tk_{+,\min}^{-1} \lambda_{\min}^{-1}}.
\end{equation*}
To obtain the form in \Cref{prop:LinearizedNonasymptotic}, one may observe that 
\begin{equation*}
    1-2Tk_+^{-1} \lambda^{-1} = \frac{\lambda-T}{\lambda+T}.
\end{equation*}
Then, the damping and optimal step-size are given by 
\begin{equation*}
    a = 2\sqrt{2} \lambda_{\max}^{-1/2} \sqrt{\frac{\lambda_{\max}-T}{\lambda_{\max}+T}},\quad \eta  = \frac{1}{\sqrt{2}} \lambda^{-1/2}_{\max}\lambda_{\min} {\sqrt{\frac{\lambda_{\max}-T}{\lambda_{\max}+T}}}{\frac{\lambda_{\min}+T}{\lambda_{\min}-T}}.
\end{equation*}
The rate is given when the complex part is largest, i.e. when $\lambda$ is minimized. It is given by 
\begin{align*}
    &\quad \max_{\lambda \in [\lambda_{\min},\lambda_{\max}]} \left|1+\frac{\eta }{2}\left[-a\pm\sqrt{a^2-4\lambda^{-2}(2\lambda-4TK_+^{-1})}\right] \right|\\
    &= \sqrt{1-a\eta  + 2\eta^2\lambda_{\min}^{-1}(1 - 2Tk_{+,\min}^{-1} \lambda_{\min}^{-1})}\\
    &= \sqrt{1 - \kappa^{-1} \frac{\lambda_{\max}-T}{\lambda_{\max}+T} \frac{\lambda_{\min}+T}{\lambda_{\min}-T}}.
\end{align*}

\subsubsection{All real eigenvalues, high critical damping}\label{appsec:highCritDamping}
This happens if for each eigenvalue $\lambda$,
\begin{equation}\label{eq:a2conditionComplexEig2}
    a^2 > 8\lambda^{-1} (1-2Tk_+^{-1} \lambda^{-1}).
\end{equation}
Since both eigenvalues are less than 0 and $\chi_- < \chi_+ < 0$, we need only check that $1+\eta  \chi_{-} > -1$. The step-size condition on $\eta $ such that the iteration is stable/convergent, thus becomes
\begin{gather*}
    1+\frac{\eta }{2}[-a-\sqrt{a^2-4\lambda^{-2}(2\lambda-4TK_+^{-1})}]>-1
\end{gather*}
The condition for stability satisfies 
\begin{align*}
    \eta  \le \frac{4}{a+\sqrt{a^2-4\lambda^{-2}(2\lambda-4TK_+^{-1})}}
\end{align*}
We can now consider the largest possible optimal $a$, which corresponds to $\lambda_{\min}$,
\begin{equation}
    a = 2\sqrt{2} \lambda_{\min}^{-1/2} \sqrt{\frac{\lambda_{\min}-T}{\lambda_{\min}+T}}.
\end{equation}
Using this damping parameter, we have that 
\begin{equation}
    a \ge 2\sqrt{2} \lambda^{-1/2} \sqrt{\frac{\lambda-T}{\lambda+T}},
\end{equation}
for all $\lambda \in [\lambda_{\min}, \lambda_{\max}]$, and all the eigenvalues of the update matrix $I+\eta A$ are real. Moreover, if the step-size satisfies
\begin{align*}
    \eta  \le \frac{2}{a},
\end{align*}
then we have that $|1+\eta \chi_+| \ge |1+\eta \chi_-|$. This can be seen by solving the equality $1+\eta \chi_+ = -(1+\eta \chi_-)$. Therefore, the rate for a given eigenvalue is given by 
\begin{equation}
    1+\eta  \chi_+ = 1 - \frac{\eta }{2}\left[a-\sqrt{a^2-8\lambda^{-1}\frac{\lambda-T}{\lambda+T}}\right].
\end{equation}
Up to the first order, this rate is controlled by the following lemma.

\begin{lemma}
    For a constant $c>0$, the function
    \begin{equation}
        q:[\sqrt{c}, \infty)\rightarrow\R,\quad  q(x) = x - \sqrt{x^2 - c},
    \end{equation}
    can be rewritten as
    \begin{equation}
        q(x) = \frac{c}{x + \sqrt{x^2-c}} > \frac{c}{2x}.
    \end{equation}
\end{lemma}

If we take $\eta  = 2/a$, the eigenvalues for convergence are all real, and the rate corresponding to an eigenvalue $\lambda \in [\lambda_{\min}, \lambda_{\max}]$ is given by,
\begin{align*}
    1-\frac{\eta }{2}\left[a - \sqrt{a^2 - 8\lambda^{-1} \frac{\lambda-T}{\lambda+T}}\right].
\end{align*}
This rate is slowest (i.e. largest) when $\lambda$ is maximized. To summarize, the choice of damping $a$ and step-size is
\begin{equation*}
    a = 2\sqrt{2} \lambda_{\min}^{-1/2} \sqrt{\frac{\lambda_{\min}-T}{\lambda_{\min}+T}},\quad \eta  = \frac{2}{a}.
\end{equation*}
This choice of parameters yields the rate 
\begin{align*}
    \quad 1-\frac{\eta }{2}\left[a - \sqrt{a^2 - 8\lambda^{-1}_{\max} \frac{\lambda_{\max}-T}{\lambda_{\max}+T}}\right]
    &= 1-a^{-1}\left[a-\sqrt{a^2 - 8\lambda_{\max}^{-1} \frac{\lambda_{\max}-T}{\lambda_{\max}+T}}\right]\\
    &= \sqrt{1 - \kappa^{-1} \frac{\lambda_{\max}-T}{\lambda_{\max}+T}\frac{\lambda_{\min}+T}{\lambda_{\min}-T}}.
\end{align*}



\section{Convergence of accelerated regularized continuous-time update}\label{app:ConvAccelRegContinuous}
Recall the continuous-time covariance update, given by the coupled ODEs \Cref{eq:tildeSigmaGtContinuous},
\begin{equation}
    \begin{cases}
        \dot{\tilde\sigma}_t = 2g_t \tilde\sigma_t - 4g_t T k_+^{-1},\\
        \dot g_t = -ag_t - g_t^2 - \lambda^{-1} + \tilde\sigma_t^{-1}.
    \end{cases}
\end{equation}
This differs from the linearized case due to the introduction of the $g_t^2$ term, which changes the dynamics away from $g_t \approx 0$. Moreover, we have a forcing term within $\dot{\tilde\sigma}_t$.

Let us additionally define the time-dependent variable
\begin{equation}
    b_\pm = \lambda^{-1/2} \pm \tilde\sigma_t^{-1/2}.
\end{equation}

For ease of notation, we denote the KL divergence between two zero-mean Gaussians directly using their covariances, $\KL(\Sigma_1, \Sigma_2) = \KL(\gN(0, \Sigma_1), \gN(0, \Sigma_2))$. We define a Lyapunov function by
\begin{equation}\label{eq:LyapunovDef}
    \gE_t =( \tilde\sigma_t - 2Tk_+^{-1}) [b_- + g_t]^2 + 2\KL(\tilde\sigma_t, \lambda).
\end{equation}

From \cite[Prop. 8]{wang2022accelerated}, we have a specialized bound, stronger than the traditional log-Sobolev inequality using the Bakry--Emery criterion. In the case where $E$ is the relative entropy, so that the first variation satisfies $\delta E(\rho_t , \rho_*)/\delta \rho_t  = \log(\rho_t / \rho_*) + 1$. One has the stronger bound
\begin{equation}\label{eq:tighterKLBound}
    \KL(\tilde\sigma, \lambda) \le \tilde\sigma \sqrt{\lambda} b_-^2 b_+ - \frac{1}{2}\tilde\sigma b_-^2.
\end{equation}
The proof by specializing this result to Gaussians is delayed until \Cref{app:strongerKLBoundYifei}, given in \Cref{cor:ImprovedKL}.

It remains to compute the time derivative of $\gE_t$, which we wish to show is negative. We have the following expressions for the time derivative of $\gE_t$:
\begin{align*}
    \frac{\dd{}}{\dd{t}} 2\KL(\tilde\sigma_t, \lambda) &= \frac{\dd{}}{\dd{t}} (\tilde\sigma_t \lambda^{-1} - \log (\tilde\sigma_t \lambda^{-1})-1)\\
    &= \dot{\tilde\sigma}_t (\lambda^{-1} - \tilde\sigma_t^{-1})\\
    &= (2g_t \tilde\sigma_t - 4g_t Tk_+^{-1}) (\lambda^{-1} - \tilde\sigma_t^{-1}),
\end{align*}
In addition, 
\begin{align*}
    \frac{\dd{}}{\dd{t}} (b_- + g_t) &= \frac{\dd{}}{\dd{t}} (\lambda^{-1/2} - \tilde\sigma_t^{-1/2} + g_t)\\
    &= -ag_t - g_t^2 - \lambda^{-1} + \tilde\sigma_t^{-1} + \frac{1}{2} \dot{\tilde\sigma}_t \tilde\sigma_t^{-3/2}\\
    &= -ag_t - g_t^2 - \lambda^{-1} + \tilde\sigma_t^{-1} + g_t \tilde\sigma_t^{-1/2} - 2Tk_+^{-1} g_t \tilde\sigma_t^{-3/2}.
\end{align*}
We can therefore compute the time derivative $\dot\gE_t$ as follows:
\begin{align}
    \dot\gE_t &= 2g(\tilde\sigma - 2Tk_+^{-1})[b_- + g]^2 \notag \\
    & \quad + 2(\tilde\sigma - 2Tk_+^{-1}) (b_- + g) (-ag - g^2 - \lambda^{-1} + \tilde\sigma^{-1} + g\tilde\sigma^{-1/2} - 2Tk_+^{-1}g\tilde\sigma^{-3/2}) \notag\\
    & \quad + 2g(\tilde\sigma - 2Tk_+^{-1})(\lambda^{-1} - \tilde\sigma^{-1}) \notag\\
    &=2(\tilde\sigma - 2Tk_+^{-1}) (b_- + g) (-ag - g^2 +g(b_-+g)+ g\tilde\sigma^{-1/2} - 2Tk_+^{-1}g\tilde\sigma^{-3/2}) \notag\\
    & \quad + 2(\tilde\sigma - 2Tk_+^{-1})(b_-)(-\lambda^{-1} + \tilde\sigma^{-1}) \notag\\
    &= -2(\tilde\sigma - 2Tk_+^{-1}) b_-^2 b_+\notag\\
    & \quad + 2g(\tilde\sigma - 2Tk_+^{-1})(b_- + g)(-a + \lambda^{-1/2} - 2Tk_+^{-1} \tilde\sigma^{-3/2}). \label{eq:appBInitEDot}
\end{align}
From the bound \labelcref{eq:tighterKLBound} on $\KL(\tilde\sigma_t, \lambda)$,
\begin{align*}
    \gE_t \le (\tilde\sigma - 2Tk_+^{-1}) [b_- + g]^2 + 2\tilde\sigma \sqrt{\lambda} b_-^2 b_+ - \tilde\sigma b_-^2.
\end{align*}
Rearranging,
\begin{equation}\label{eq:b2bbound}
    -2\tilde\sigma \sqrt{\lambda} b_-^2 b_+ \le -\gE_t + (\tilde\sigma - 2Tk_+^{-1}) [b_- + g]^2 - \tilde\sigma b_-^2.
\end{equation}
Substituting into \labelcref{eq:appBInitEDot}, and noting that $\tilde\sigma - 2Tk_+^{-1}$ and $b_+$ are positive, 
\begin{subequations}\label{eq:appBdE}
    \begin{align}
    \dot \gE_t &= -2(\tilde\sigma - 2Tk_+^{-1} ) b_-^2 b_+ \notag \\
    &\quad + 2g(\tilde\sigma - 2Tk_+^{-1})(b_-+g)[-a+\lambda^{-1/2} - 2Tk_+^{-1}\tilde\sigma^{-3/2}] \notag\\
    &\le \left(\frac{2\tilde\sigma - 4Tk_+^{-1}}{2\tilde\sigma\sqrt{\lambda}}\right) \left[-\gE_t + (\tilde\sigma - 2Tk_+^{-1})[B_-+g]^2 - \tilde\sigma b_-^2\right] \notag\\
    &\quad + 2g(\tilde\sigma - 2Tk_+^{-1}) (b_-+g)[-a+\lambda^{-1/2}-2Tk_+^{-1} \tilde\sigma^{-3/2}] \notag\\
    &= (\lambda^{-1/2} - 2\tilde\sigma^{-1}Tk_+^{-1} \lambda^{-1/2})\left[-\gE_t + (\tilde\sigma - 2Tk_+^{-1})[b_-+g]^2 - \tilde\sigma b_-^2\right] \notag\\
    &\quad + 2g(\tilde\sigma - 2Tk_+^{-1}) (b_-+g)[-a+\lambda^{-1/2}-2Tk_+^{-1} \tilde\sigma^{-3/2}] \notag\\
    &= -(\lambda^{-1/2}-2\tilde\sigma^{-1}Tk_+^{-1}\lambda^{-1/2}) \gE_t \label{eq:appBdEa} \\
    & \quad + \lambda^{-1/2} \tilde\sigma^{-1}(\tilde\sigma - 2Tk_+^{-1})[(\tilde\sigma - 2Tk_+^{-1})[b_-+g]^2 - \tilde\sigma b_-^2] \label{eq:appBdEb} \\
    & \quad + 2g(\tilde\sigma - 2Tk_+^{-1}) (b_-+g)[-a+\lambda^{-1/2} - 2Tk_+^{-1}\tilde\sigma^{-3/2}].\label{eq:appBdEc}
\end{align}
\end{subequations}

\subsubsection{One-dimensional case: critical momentum}\label{app:appBOneDim}
We now control the latter two terms \labelcref{eq:appBdEb,eq:appBdEc}, by showing their sum is negative. Then, we can use Gr\"onwall's inequality to conclude.

\begin{align*}
    & \quad \lambda^{-1/2} \tilde\sigma^{-1}(\tilde\sigma - 2Tk_+^{-1})[(\tilde\sigma - 2Tk_+^{-1})[b_-+g]^2 - \tilde\sigma b_-^2] \\
    & \quad + 2g(\tilde\sigma - 2Tk_+^{-1}) (b_-+g)[-a+\lambda^{-1/2} - 2Tk_+^{-1}\tilde\sigma^{-3/2}]\\
    &= (\tilde\sigma - 2Tk_+^{-1}) \begin{bmatrix}
        \lambda^{-1/2}\tilde\sigma^{-1}[2\tilde\sigma b_- g + \tilde\sigma g^2 - 2Tk_+^{-1} (b_-+g)^2]\\
        + 2g(b_-+g) [-a+\lambda^{-1/2} -2Tk_+^{-1} \tilde\sigma^{-3/2}]
    \end{bmatrix}\\
    &= (\tilde\sigma - 2Tk_+^{-1}) \begin{bmatrix*}[l]
        (2\lambda^{-1/2} - 2a+ 2\lambda^{-1/2}) b_- g + (\lambda^{-1/2} - 2a + 2\lambda^{-1/2}) g^2\\
        - 2Tk_+^{-1} \lambda^{-1/2} \tilde\sigma^{-1}(b_-+g)^2 - 4Tk_+^{-1} g(b_-+g) \tilde\sigma^{-3/2}
    \end{bmatrix*}\\
    &=(\tilde\sigma - 2Tk_+^{-1}) \begin{bmatrix*}[l]
        -\lambda^{-1/2} g^2\\
        - 2Tk_+^{-1} \lambda^{-1/2} \tilde\sigma^{-1}(b_-+g)^2 - 4Tk_+^{-1} g(b_-+g) \tilde\sigma^{-3/2}
    \end{bmatrix*},
\end{align*}
where the last equality holds if we take the momentum parameter 
\begin{equation}\label{eq:appBMomentumChoiceOne}
a = 2\lambda^{-1/2}.    
\end{equation}
It remains to use the control on $g^2$ and $(b_-+g)^2$ to bound the final $g(b_-+g)$ term. The component inside the bracket is a quadratic in $g$:
\begin{align*}
    &\quad -\lambda^{-1/2} g^2
        - 2Tk_+^{-1} \lambda^{-1/2} \tilde\sigma^{-1}(b_-+g)^2 - 4Tk_+^{-1} g(b_-+g) \tilde\sigma^{-3/2}\\
    &= g^2 (-\lambda^{-1/2} - 2Tk_+^{-1} \lambda^{-1/2} \tilde\sigma^{-1} - 4Tk_+^{-1} \tilde\sigma^{-3/2})\\
    &\quad + g b_-(-4Tk_+^{-1} \lambda^{-1/2} \tilde\sigma^{-1} - 4Tk_+^{-1} \tilde\sigma^{-3/2}) -2Tk_+^{-1} \lambda^{-1/2} \tilde\sigma^{-1} b_-^2.
\end{align*}
The coefficient of $g^2$ is negative. Maximizing over all possible $g$, the above expression is upper bounded by 
\begin{align}
    &\quad -2Tk_+^{-1} \lambda^{-1/2} \tilde\sigma^{-1} b_-^2 - \frac{b_-^2(-4Tk_+^{-1} \lambda^{-1/2} \tilde\sigma^{-1} - 4Tk_+^{-1} \tilde\sigma^{-3/2})^2}{4(-\lambda^{-1/2} - 2Tk_+^{-1} \lambda^{-1/2} \tilde\sigma^{-1} - 4Tk_+^{-1} \tilde\sigma^{-3/2})} \notag\\
    &= \frac{1}{4(-\lambda^{-1/2} - 2Tk_+^{-1} \lambda^{-1/2} \tilde\sigma^{-1} - 4Tk_+^{-1} \tilde\sigma^{-3/2})}  \notag\\
    &\qquad \cdot \Big[-8Tk_+^{-1} \lambda^{-1/2} \tilde\sigma^{-1} b_-^2(-\lambda^{-1/2} - 2Tk_+^{-1} \lambda^{-1/2} \tilde\sigma^{-1} - 4Tk_+^{-1} \tilde\sigma^{-3/2}) \notag \\ &\qquad - b_-^2(-4Tk_+^{-1} \lambda^{-1/2} \tilde\sigma^{-1} - 4Tk_+^{-1} \tilde\sigma^{-3/2})^2\Big] \notag\\
    & = \frac{b_-^2}{4c_t}\cdot \Big[8Tk_+^{-1} \lambda^{-1} \tilde\sigma^{-1} + 16T^2k_+^{-2} \lambda^{-1} \tilde\sigma^{-2} +32T^2k_+^{-2} \lambda^{-1/2}\tilde\sigma^{-5/2})  \notag\\
    &\qquad - 16 \left( T^2 k_+^{-2} \lambda^{-1} \tilde\sigma^{-2} + 2T^2 k_+^{-2} \lambda^{-1/2} \tilde\sigma^{-5/2} + T^2k_+^{-2} \tilde\sigma^{-3}\right)\Big] \notag\\
    &= \frac{b_-^2}{4c_t} \left[8Tk_+^{-1} \lambda^{-1} \tilde\sigma^{-1} - 16T^2k_+^{-2} \tilde\sigma^{-3}\right].\label{eq:appBBigBracketOne}
\end{align}
where $c_t$ indicates the negative denominator $c_t = (-\lambda^{-1/2} - 2Tk_+^{-1} \lambda^{-1/2} \tilde\sigma^{-1} - 4Tk_+^{-1} \tilde\sigma^{-3/2})$. 
This quantity is negative if the term in \labelcref{eq:appBBigBracketOne} is positive. This is the case if the following relationship on $\tilde\sigma$ and $T$ holds:
\begin{align}
    &\quad 8TK_+^{-1} \lambda^{-1} \tilde\sigma^{-1}  - 16T^2K_+^{-2} \tilde\sigma^{-3} \ge 0 \notag\\
    &\Leftrightarrow \lambda^{-1} - 2TK_+^{-1}\tilde\sigma^{-2} \ge 0 \notag\\
    &\Leftrightarrow\tilde\sigma^2 \ge 2TK_+^{-1} \lambda. \label{eq:appBConditionVariance}
\end{align}
This states that the variance of the regularized Wasserstein proximal can not be too small, or that the regularization has to be chosen to be sufficiently small. 

As a sanity check, we may verify that for the choice $T<\lambda$, this inequality holds near the terminal variance $\tilde\sigma \approx \lambda$. The necessary condition becomes
\begin{gather}
    \lambda^2 \ge 2TK_+^{-1}\lambda \notag\\
    \Leftrightarrow \lambda \ge 2TK_+^{-1} = \frac{2T}{1+T\lambda^{-1}} = \lambda + \frac{T-\lambda}{1+T\lambda^{-1}}\notag,
\end{gather}
which is equivalent to $T \le \lambda$, as initially assumed. 

Returning to \Cref{eq:appBdEa}, we have shown that under the assumption \labelcref{eq:appBConditionVariance} on the variance and regularization, and using the momentum parameter choice $a=2\lambda^{-1/2}$ in \labelcref{eq:appBMomentumChoiceOne}, the Gr\"onwall-type inequality holds:
\begin{align*}
    \dot \gE_t &\le -(\lambda^{-1/2} - 2Tk_+^{-1}\lambda^{-1/2}\tilde\sigma^{-1}) \gE_t\\
    &\quad + (\tilde\sigma - 2Tk_+^{-1}) b_-^2\frac{8Tk_+^{-1}\lambda^{-1}\tilde\sigma^{-1}- 16T^2 k_+^{-2} \tilde\sigma^{-3}}{4(-\lambda^{-1/2} - 2Tk_+^{-1} \lambda^{-1/2} \tilde\sigma^{-1} - 4Tk_+^{-1} \tilde\sigma^{-3/2})} \\
    &\le -(\lambda^{-1/2} - 2Tk_+^{-1}\lambda^{-1/2}\tilde\sigma^{-1}) \gE_t.
\end{align*}
In particular, close to the terminal distribution $\tilde\sigma \approx \lambda$, one has the asymptotic rate
\begin{equation}
    \gE_t = \mathcal{O}(e^{-rt}),
\end{equation}
where the rate is 
\begin{align*}
    r &= \lambda^{-1/2} - 2Tk_+^{-1} \lambda^{-1/2} \lambda^{-1}\\
    &= \left(1 - 2Tk_+^{-1}\lambda^{-1}\right)\lambda^{-1/2}\\
    &= \left(\frac{\lambda - T}{\lambda + T}\right) \lambda^{-1/2}.
\end{align*}


\subsubsection{Subcritical damping, removing the step-size condition}\label{app:appBMultiDim}
Consider $a \in(\lambda^{-1/2}, 2\lambda^{-1/2}]$. We wish to drop the condition that $\tilde\sigma$ is bounded below when deriving our rate, in order to get a global convergence. 
\begin{equation*}
    \gF_t = (\tilde\sigma_t-2Tk_+^{-1})[b_-+g_t]^2 + 2\chi \KL(\tilde\sigma_t, \lambda),
\end{equation*}
and $\chi = \chi(a,\lambda,T)>0$ is to be determined. Taking the time derivative of $\gF_t$, one gets
\begin{align*}
    \frac{\dot\gF_t}{2(\tilde\sigma - 2Tk_+^{-1})} &= g[b_- + g]^2 \notag \\
    & \quad +  (b_- + g) (-ag - g^2 - \lambda^{-1} + \tilde\sigma^{-1} + g\tilde\sigma^{-1/2} - 2Tk_+^{-1}g\tilde\sigma^{-3/2}) \notag\\
    & \quad + \chi g(\lambda^{-1} - \tilde\sigma^{-1}) \notag\\
    &= (b_-+g) \left(gb_- + g^2 - ag - g^2 - b_-b_+ + g\tilde\sigma^{-1/2}(1-2Tk_+^{-1}\tilde\sigma^{-1})\right)\\
    &\quad + \chi gb_-b_+\\
    &= (b_-+g)(g(b_--a) - b_-b_+ + g\tilde\sigma^{-1/2}(1-2Tk_+^{-1}\tilde\sigma^{-1})) + \chi gb_-b_+\\
    &= (b_-+g)(g(b_--a) + g\tilde\sigma^{-1/2}(1-2Tk_+^{-1}\tilde\sigma^{-1})) + (\chi-1) gb_-b_+\\
    &\quad - b_-^2 b_+.
\end{align*}
Note since $a\ge\lambda^{-1/2}$, that $b_--a < 0$. Moreover, since
\begin{equation}
    \KL(\tilde\sigma, \lambda) \le \tilde\sigma \sqrt{\lambda} b_-^2 b_+ - \frac{1}{2}\tilde\sigma b_-^2,
\end{equation}
we have
\begin{align*}
    &\quad -2\chi\tilde\sigma \sqrt{\lambda} b_-^2 b_+ \le -\gF_t + (\tilde\sigma - 2Tk_+^{-1}) [b_- + g]^2 - \tilde\sigma \chi b_-^2\\
    &\Leftrightarrow \tilde\sigma \chi(1-2\sqrt{\lambda}b_+)b_-^2\le -\gF_t + (\tilde\sigma - 2Tk_+^{-1})[b_-+g]^2 \\
    &\Leftrightarrow -b_-^2 \le \frac{-\gF_t + (\tilde\sigma-2Tk_+^{-1})[b_-+g]^2}{\chi \tilde\sigma(2\sqrt{\lambda}b_+-1)},
\end{align*}
where the equivalence is since $1-2\sqrt{\lambda}b_+ = -1-2\tilde\sigma^{-1/2}\lambda^{1/2}<0$. It remains to use the term $b_-^2 b_+$ to determine the rate. Let $r >0$ be some rate parameter. Then,
\begin{align}
    \frac{\dot\gF_t}{2(\tilde\sigma - 2Tk_+^{-1})}&\le  (b_-+g)(g(b_--a) + g\tilde\sigma^{-1/2}(1-2Tk_+^{-1}\tilde\sigma^{-1})) + (\chi-1) gb_-b_+ \notag\\
    &\quad + r b_+ \frac{-\gF_t + (\tilde\sigma-2Tk_+^{-1})[b_-+g]^2}{\chi \tilde\sigma(2\sqrt{\lambda}b_+-1)} - (1-r)b_-^2 b_+ \notag\\
    &= - \frac{rb_+}{{\chi \tilde\sigma(2\sqrt{\lambda}b_+-1)}} \gF_t \notag\\
    &\quad + g(b_-+g)(b_--a + \tilde\sigma^{-1/2}(1-2Tk_+^{-1}\tilde\sigma^{-1})) + (\chi-1) gb_-b_+ \notag\\
    &\quad + r b_+ \frac{(\tilde\sigma-2Tk_+^{-1})[b_-+g]^2}{\chi \tilde\sigma(2\sqrt{\lambda}b_+-1)}- (1-r)b_-^2 b_+. \label{eq:FUpdate}
\end{align}
It remains to find a maximal $r=r_t>0$, such that the sum of the last two terms is always negative. Considering it as a quadratic in $b_-+g$, one has
\begin{align*}
    &\quad g(b_-+g)(b_--a + \tilde\sigma^{-1/2}(1-2Tk_+^{-1}\tilde\sigma^{-1})) + (\chi-1) gb_-b_+\\
    &\quad + r b_+ \frac{(\tilde\sigma-2Tk_+^{-1})[b_-+g]^2}{\chi \tilde\sigma(2\sqrt{\lambda}b_+-1)}- (1-r)b_-^2 b_+\\
    &= (b_-+g)^2 (b_--a + \tilde\sigma^{-1/2}(1-2Tk_+^{-1}\tilde\sigma^{-1})) - {b_-(b_-+g)} (b_--a + \tilde\sigma^{-1/2}(1-2Tk_+^{-1}\tilde\sigma^{-1}))\\
    &\quad + (\chi-1)(b_-+g)b_-b_+- (\chi-1)b_-^2 b_+\\
    &\quad + r b_+ \frac{(\tilde\sigma-2Tk_+^{-1})}{\chi \tilde\sigma(2\sqrt{\lambda}b_+-1)} [b_-+g]^2 - (1-r)b_-^2 b_+.
\end{align*}

This is a quadratic $c_2(b_-+g)^2 + c_1(b_-+g) + c_0$, where
\begin{equation*}
    c_2 = (b_--a + \tilde\sigma^{-1/2}(1-2Tk_+^{-1}\tilde\sigma^{-1})) + r b_+ \frac{(\tilde\sigma-2Tk_+^{-1})}{\chi \tilde\sigma(2\sqrt{\lambda}b_+-1)},
\end{equation*}
\begin{equation*}
    c_1 = (\chi-1)b_-b_+ - b_-(b_- - a + \tilde\sigma^{-1/2}(1-2Tk_+^{-1}\tilde\sigma^{-1})),
\end{equation*}
\begin{equation*}
    c_0 = -(\chi-r)b_-^2 b_+.
\end{equation*}

We wish to show that $c_2<0$, and furthermore that the maximum $c_0 - \frac{c_1^2}{4c_2}$ is negative. Equivalently, $c_1^2 - 4c_0c_2<0$.

\textbf{Condition 1:} quadratic coefficient is negative. The equivalent condition for this to hold is that
\begin{align}
    &(b_--a + \tilde\sigma^{-1/2}(1-2Tk_+^{-1}\tilde\sigma^{-1})) + r b_+ \frac{(\tilde\sigma-2Tk_+^{-1})}{\chi \tilde\sigma(2\sqrt{\lambda}b_+-1)} <0 \notag\\
    & \Leftrightarrow r\chi^{-1} \le \frac{\tilde\sigma(a-b_- - \tilde\sigma^{-1/2}(1-2Tk_+^{-1}\tilde\sigma^{-1}))(2\sqrt{\lambda}b_+-1)}{b_+(\tilde\sigma-2Tk_+^{-1})}. \label{eq:rchiConditionOne}
\end{align}
The RHS is positive for all $\tilde\sigma$ if $a>\lambda^{-1/2}$.

\textbf{Condition 2:} quadratic is upper bounded by 0. Let $p=p_t\coloneqq a-b_- -\tilde\sigma^{-1/2}(1-2Tk_+^{-1}\tilde\sigma^{-1}) = a-\lambda^{-1/2} + 2Tk_+^{-1}\tilde\sigma^{-3/2}>0$. Rewriting the coefficients, we obtain
\begin{gather}
    c_2 = -p + rb_+ \frac{\tilde\sigma - 2Tk_+^{-1}}{\chi \tilde\sigma (2\sqrt{\lambda}b_+-1)},\\
    c_1 = (\chi-1) b_-b_+ + b_-p,\\
    c_0 = -(\chi-r)b_-^2 b_+.
\end{gather}
The equivalent condition is 
\begin{align}
    0&\ge c_1^2-4c_0c_2\\
    \Leftrightarrow  0&\ge ((\chi-1)b_+ + p)^2 - 4 (-p+rb_+ \frac{\tilde\sigma - 2Tk_+^{-1}}{\chi \tilde\sigma(2\sqrt{\lambda}b_+-1)})(-(\chi-r)b_+).\label{eq:QuadraticR1}
\end{align}
This inequality has to be strict at $r=0$ for a feasible rate to exist. One obtains the simplified quadratic inequality:
\begin{align*}
    0 &\ge ((\chi-1)b_+ +p)^2 - 4(-p)(-\chi b_+).
\end{align*}
One immediately observes now that if $a \in (\lambda^{-1/2}, 2\lambda^{-1/2}]$, then
\begin{equation}
    p = a-\lambda^{-1/2} + 2Tk_+^{-1}\tilde\sigma^{-3/2} \le \lambda^{-1/2} + 2Tk_+^{-1}\tilde\sigma^{-3/2} < b_+.
\end{equation}
Therefore, taking $\chi=1$ yields that the quadratic inequality holds strictly,
\begin{equation*}
    p^2 - 4pb_+ <0.
\end{equation*}
One may now solve for $r$ in \Cref{eq:QuadraticR1}. Substituting $\chi=1$, the quadratic inequality becomes
\begin{align*}
    (p)^2 - 4 (-p+rb_+ \frac{\tilde\sigma - 2Tk_+^{-1}}{ \tilde\sigma(2\sqrt{\lambda}b_+-1)})(-(1-r)b_+)
\end{align*}
This is a quadratic with positive coefficient in $r^2$ and is negative at $r=0$. Since 
\begin{equation}
 \frac{p \tilde\sigma (2\sqrt{\lambda}{b_+}-1)}{b_+(\tilde\sigma -2 Tk_+^{-1})}, \quad 1
\end{equation}
are positive, the inflection point is also positive. Therefore, the rate (that also satisfies \Cref{eq:rchiConditionOne}) is given by the smallest (positive) root. Since the quadratic is negative at $r=0$, by the intermediate value theorem (IVT), the positive root must be smaller than the (positive) quantity in \labelcref{eq:rchiConditionOne}.

\subsubsection{Multi-dimensional case: overdamping}
Let $\zeta>0$ be a constant to be chosen later, and define the Lyapunov function
\begin{equation*}
    \gF_t = \zeta^{-1}(\tilde\sigma_t-2Tk_+^{-1})[b_-+ \zeta g_t]^2 + 2\zeta  \KL(\tilde\sigma_t\lambda).
\end{equation*}
Differentiating,
 \begin{align*}
     \frac{\dot\gF_t}{2(\tilde\sigma - 2Tk_+^{-1})} &= \zeta^{-1}
     (g)(b_- + \zeta g)^2\\
     &\quad + \zeta^{-1}(b_-+\zeta g)(\zeta (-ag-g^2 - \lambda^{-1} + \tilde\sigma^{-1})+ g\tilde\sigma^{-3/2}(\tilde\sigma - 2Tk_+^{-1}))\\
     &\quad + \zeta g(\lambda^{-1} - \tilde\sigma^{-1})\\
     &= \zeta^{-1} (b_- + \zeta g)\\
     &\qquad \cdot \left[gb_- + \zeta g^2 + \zeta (-ag-g^2-\lambda^{-1}+\tilde\sigma^{-1}) + g\tilde\sigma^{-3/2}(\tilde\sigma - 2Tk_+^{-1})\right]\\
     &\quad + \zeta gb_-b_+\\
     &= \zeta^{-1} (b_- + \zeta g)\\
     &\quad \cdot \left[gb_- - a \zeta g - \zeta  b_-b_+ + g\tilde\sigma^{-3/2}(\tilde\sigma - 2Tk_+^{-1})\right]\\
     &\quad + \zeta gb_-b_+\\
    &= \zeta^{-1} g (b_- + \zeta g) \cdot \left[b_- - a\zeta  + \tilde\sigma^{-3/2}(\tilde\sigma - 2Tk_+^{-1})\right]\\
     &\quad - b_-^2 b_+.
 \end{align*}
Since
\begin{equation}
    \KL(\tilde\sigma, \lambda) \le \tilde\sigma \sqrt{\lambda} b_-^2 b_+ - \frac{1}{2}\tilde\sigma b_-^2 = b_-^2(\sqrt{\lambda}b_+- \frac{1}{2}) \tilde\sigma,
\end{equation}
we have that
\begin{align*}
    \gF \le \zeta^{-1} (\tilde\sigma_t-2Tk_+^{-1})[b_-+ \zeta g_t]^2 + \zeta b_-^2(2\sqrt{\lambda}b_+- 1) \tilde\sigma,
\end{align*}
and therefore we have the bound
\begin{align}
    -b_-^2 &\le \frac{-\gF + \zeta^{-1}(\tilde\sigma - 2Tk_+^{-1})(b_- + \zeta g)^2}{\zeta(2\sqrt{\lambda}{b_+}-1) \tilde\sigma}.
\end{align}
Introducing a rate parameter $r>0$, we again split the control term $- b_-^2 b_+ = -rb_-^2 b_+ - (1-r) b_-^2 b_+$, giving the derivative control
\begin{align*}
    \frac{\dot\gF_t}{2(\tilde\sigma - 2Tk_+^{-1})} &= \zeta^{-1} g (b_- + \zeta g) \cdot \left[b_- - a\zeta  + \tilde\sigma^{-3/2}(\tilde\sigma - 2Tk_+^{-1})\right]\\
     &\quad - r b_-^2 b_+- (1-r) b_-^2 b_+\\
     &\le \zeta^{-1} g (b_- + \zeta g) \cdot \left[b_- - a\zeta  + \tilde\sigma^{-3/2}(\tilde\sigma - 2Tk_+^{-1})\right]\\
     &\quad + r b_+ \frac{-\gF + \zeta^{-1} (\tilde\sigma - 2Tk_+^{-1}) (b_- + \zeta g)^2}{\zeta (2\sqrt{\lambda} b_+-1)\tilde\sigma}\\
     &\quad - (1-r) b_-^2 b_+\\
     &= -\frac{rb_+}{\zeta(2\sqrt{\lambda}b_+-1)\tilde\sigma} \gF + r b_+ \frac{\zeta^{-1} (\tilde\sigma - 2Tk_+^{-1}) (b_- + \zeta g)^2}{\zeta (2\sqrt{\lambda} b_+-1)\tilde\sigma}\\
     &\quad + \zeta^{-2} (b_-+\zeta g)^2 [\lambda^{-1/2} - a\zeta - 2Tk_+^{-1} \tilde\sigma^{-3/2}]\\
     &\quad - \zeta^{-2} b_- (b_- + \zeta g) [\lambda^{-1/2} - a\zeta - 2Tk_+^{-1} \tilde\sigma^{-3/2}]\\
     &\quad - (1-r) b_-^2 b_+.
\end{align*}
In the last step, we expand $b_-$, and write $g(b_-+\zeta g) = \zeta^{-1}(b_- + \zeta g)^2 - \zeta^{-1}b_-(b_-+\zeta g)$. Define the auxiliary variable similarly to the previous section, as
\begin{equation*}
    p = a\zeta - \lambda^{-1/2} + 2Tk_+^{-1}\tilde\sigma^{-3/2}.
\end{equation*}
In order to have $p>0$, we have the \textbf{necessary and sufficient condition: $a\zeta \ge \lambda^{-1/2}$}. 

Consider the quadratic that is added to the $\gF$ term. It can be written as $c_2 (b_- + \zeta g)^2 + c_1 (b_- + \zeta g) + c_0$, where
\begin{gather}
    c_2 = -\zeta^{-2} p + rb_+ \zeta^{-2}\frac{\tilde\sigma - 2Tk_+^{-1}}{(2\sqrt{\lambda}b_+ -1) \tilde\sigma},\\
    c_1 = \zeta^{-2} b_- p,\\
    c_0 = -(1-r)b_-^2b_+.
\end{gather}
We wish to show that the quadratic is always negative for all possible $g$. 

\noindent\textbf{Condition 1:} $c_2<0$. This is equivalent to the inequality
\begin{equation*}
    r \le \frac{\tilde\sigma p (2\sqrt{\lambda}b_+-1)}{b_+(\tilde\sigma - 2Tk_+^{-1}{)}}.
\end{equation*}
\textbf{Condition 2:} $c_1^2 - 4c_2c_0<0$ when $r=0$. This implies that the quadratic when $r=0$ is strictly less than 0, which guarantees the existence of a positive rate by continuity. This condition can be written as
\begin{align*}
    & \quad 0>\zeta^{-4} p^2 - 4(-\zeta^{-2}p)(-b_+)\\
    &\Leftrightarrow 0 > \zeta^{-2} p - 4b_+\\
    &\Leftrightarrow 4\zeta^2 b_+ > p.
\end{align*}
This holds for sufficiently large $\zeta$ since $p \sim a\zeta$ as $\zeta \rightarrow \infty$.

Let us take \begin{equation}\label{eq:kapp}
    \zeta = a\lambda^{1/2}/2.
\end{equation}
Since $a\ge 2\lambda^{-1/2}$, the necessary condition $a\zeta \ge \lambda^{-1/2}$ holds and we have that $p>0$. Therefore, the above equivalences hold. Moreover, we have that
\begin{align*}
    &\quad 4\zeta^2 b_+ - p \\
    &\ge a^2 \lambda b_+ - \frac{1}{2}a^2 \lambda^{1/2} + \lambda^{-1/2} - 2Tk_+^{-1} \tilde\sigma^{-3/2} \\
    &\ge \frac{1}{2}a^2 \lambda^{1/2} + \lambda^{-1/2} + a^2 \lambda \tilde\sigma^{-1/2} - \tilde\sigma^{-1/2}>0.
\end{align*}
Therefore, the quadratic is always negative for $r=0$. One obtains that the optimal rate is the smallest positive root of $c_1^2 - 4c_0c_2$, when written as a quadratic in $r$,
\begin{equation*}
    \zeta^{-4} p^2 - 4(-\zeta^{-2}p + rb_+ \zeta^{-2} \frac{\tilde\sigma - 2Tk_+^{-1}}{(2\sqrt\lambda b_+-1)\tilde\sigma})(-(1-r)b_+),
\end{equation*}
and the rate becomes
\begin{equation*}
    \dot \gF_t \le -\frac{2rb_+(\tilde\sigma - 2Tk_+^{-1})}{\zeta(2\sqrt{\lambda}b_+-1) \tilde\sigma} \gF_t.
\end{equation*}

\subsection{Proof of strengthened KL bound}\label{app:strongerKLBoundYifei}
It remains to show the strengthened inequality \labelcref{eq:tighterKLBound}. It states:
\begin{equation*}
    \KL(\tilde\sigma, \lambda) \le \tilde\sigma \sqrt{\lambda} b_-^2 b_+ - \frac{1}{2}\tilde\sigma b_-^2.
\end{equation*}
where $b_\pm = \lambda^{-1/2} \pm \tilde\sigma^{-1/2}$. To use this, we use \cite[Prop. 8]{wang2022accelerated}, stated as follows. We first need the concept of convexity over a probability space.

\begin{definition}
    Let $\Omega \subset \R^n$ be some domain, and $\gP(\Omega)$ the space of probability densities over $\Omega$. For a density $\rho \in \gP(\Omega)$, let $T_\rho \gP(\Omega)$ and $T_\rho^* \gP(\Omega)$ be the tangent and cotangent spaces at $\rho$ respectively. A metric tensor is a (pointwise) invertible mapping $G(\rho):T_\rho \gP(\Omega) \rightarrow T_\rho^*\gP(\Omega)$, which induces a metric inner product $g_\rho$ on $T_\rho\gP(\Omega)$ by
    \begin{equation*}
        g_\rho(\sigma_1, \sigma_2) = \int \sigma_1 G(\rho) \sigma_2 \dd{x} = \int \Phi_1 G(\rho)^{-1} \Phi_2 \dd{x},\quad \sigma_1, \sigma_2 \in T_\rho \gP(\Omega),
    \end{equation*}
    where $\Phi_i$ satisfies $\sigma_i = G(\rho)^{-1} \Phi_i$, $i=1,2$.

    Let $E(\rho)$ be defined some functional over a probability space, such as the KL divergence. We say that $E(\rho)$ is $\alpha$-strongly convex (with respect to a metric $g$) if for any $\rho \in \gP(\Omega)$, for any $\sigma \in T_\rho \gP(\Omega)$, we have
    \begin{equation*}
        g_\rho (\operatorname{Hess} E(\rho)\sigma, \sigma)\ge \beta g_\rho(\sigma,\sigma),
    \end{equation*}
    where $\operatorname{Hess} E(\rho)$ is the Hessian operator w.r.t.\! the metric $g_\rho$.
\end{definition}

The special case we consider is when the metric is given by the Wasserstein metric, with metric tensor
\begin{equation*}
    G(\rho)^{-1}(\Phi) = -\nabla\cdot(\rho \nabla \Phi),
\end{equation*}
and (tangent space) inner product (for $G(\rho)^{-1} \Phi_i \in T_\rho \gP(\Omega)$,
\begin{equation*}
    g_\rho(G(\rho)\Phi_1, G(\rho)\Phi_2) = \int \Phi_1 G(\rho)^{-1} \Phi_2 \dd{x}.
\end{equation*}
The enhanced KL property is given as follows.
\begin{proposition}\label{prop:yifeiprop}
    Let $E(\rho)$ be some potential energy functional in Wasserstein space, and suppose that $E$ satisfies $\mathrm{Hess}(\alpha)$ for some $\alpha \ge 0$. Let $\rho^*$ be the minimizer of $E$. Further let $T_t$ be the optimal transport map from $\rho_t$ to $\rho^*$. Then,
    \begin{equation}\label{eq:yifeiGeneralBound}
        E(\rho^*) \ge E(\rho) + \int \left\langle T_t(x)-x, \nabla \frac{\delta E}{\delta \rho_t} \right\rangle \rho_t \dd{x} + \frac{\alpha}{2} \int \|T_t(x)-x)\|^2\rho_t \dd{x}.
    \end{equation}
\end{proposition}
We have the following corollary for the special case where $E$ is KL divergence and $\rho_t, \rho^*$ are both Gaussians. Let $\rho^* = \gN(0, \Lambda)$. We have the following representation of the KL divergence and transport map:
\begin{lemma}\label{lem:KLandOTGaussians}
    The KL distance between two zero-mean Gaussians in $\R^d$ is 
    \begin{equation}
        \KL(\Sigma_1, \Sigma_2) = \frac{1}{2} \left(\log \det \Sigma_2\Sigma_1^{-1} - d + \Tr(\Sigma_2^{-1} \Sigma_1)\right).
    \end{equation}
    Moreover, the optimal transport map $T$ from $\gN(0, \Sigma_1)$ to $\gN(0, \Sigma_2)$ is linear, 
    \begin{equation}
        T(x) = \Sigma_1^{-1/2} (\Sigma_1^{1/2}\Sigma_2\Sigma_1^{1/2})^{1/2}\Sigma_1^{-1/2} x.
    \end{equation}
\end{lemma}

Substituting this into \Cref{prop:yifeiprop} yields the desired inequality. Let $\rho^* = \gN(0, \Lambda)$ be the desired target distribution $\rho^*\propto \exp(-V)$, where $V(x) = x^\top \Lambda^{-1}x/2$.
\begin{corollary}\label{cor:ImprovedKL}
    Let $\rho^* = \gN(0, \Lambda)$. Taking $E(\rho) = \KL(\rho\|\rho^*)$, the minimizer of $E$ is $\rho^*$. Moreover, if $\rho = \gN(0, \Sigma)$, then
    \begin{multline}
        E(\rho) \le -\Tr((\Sigma^{1/2}\Lambda \Sigma^{1/2})^{1/2} (\Sigma^{-1/2}\Lambda^{-1}\Sigma^{-1/2}-I) \Sigma)\\
        \quad - \frac{\alpha}{2} \Tr(\left(\Sigma^{-1/2} (\Sigma^{1/2}\Lambda\Sigma^{1/2})^{1/2}\Sigma^{-1/2} -I\right)^2 \Sigma).
    \end{multline}
    In the case where the covariance matrices $\Sigma$ and $\Lambda$ commute, one has the strengthened bound, where $B_\pm = \Lambda^{-1/2} \pm \Sigma^{-1/2}$,
    \begin{equation*}
        E(\rho) \le \Tr(\Sigma \Lambda^{1/2} B_-^2 B_+ - \Sigma B_-^2/2).
    \end{equation*}
\end{corollary}
\begin{proof}
The following two statements hold. 
\begin{itemize}
    \item $E(\rho)$ satisfies $\mathrm{Hess}(\alpha)$ where $\alpha = \lambda_{\max}^{-1}$. This follows from the Bakry--Emery criterion since $V(x)$ is $\lambda_{\max}^{-1}$-strongly convex.

    \item The optimal transport map is linear by \Cref{lem:KLandOTGaussians}.
\end{itemize}
    We can now substitute into \labelcref{eq:yifeiGeneralBound}. For a quadratic form $x^\top Ax$, we have that
    \begin{equation}
        \int x^\top A x \dd{\gN(0, \Sigma)}(x) = \Tr(A\Sigma).
    \end{equation}
    Further note that for the KL divergence,
    \begin{equation*}
        \frac{\delta E}{\delta \rho} = \frac{1}{2} x^\top (\Lambda^{-1}-\Sigma^{-1})x+1,\quad \text{therefore }\quad \nabla \frac{\delta E}{\delta \rho}(x) = (\Lambda^{-1}-\Sigma^{-1})x. 
    \end{equation*}
    Since $E(\rho^*)=0$, we have that
    \begin{align*}
        E(\rho) &\le -\int \left\langle T(x)-x, \nabla \frac{\delta E}{\delta \rho} \right\rangle \rho \dd{x} - \frac{\alpha}{2} \int \|T(x)-x\|^2\rho \dd{x}\\
        &= -\int x^\top \left(\Sigma^{-1/2} (\Sigma^{1/2}\Lambda\Sigma^{1/2})^{1/2}\Sigma^{-1/2} (\Lambda^{-1}-\Sigma^{-1})\right) x \dd{\rho(x)}\\
        &\quad - \frac{\alpha}{2}\int x^\top \left(\Sigma^{-1/2} (\Sigma^{1/2}\Lambda\Sigma^{1/2})^{1/2}\Sigma^{-1/2} -I\right)^2 x \dd{\rho(x)}\\
        &= -\Tr((\Sigma^{1/2}\Lambda \Sigma^{1/2})^{1/2} (\Sigma^{-1/2}\Lambda^{-1}\Sigma^{-1/2}-I) \Sigma)\\
        &\quad - \frac{\alpha}{2} \Tr(\left(\Sigma^{-1/2} (\Sigma^{1/2}\Lambda\Sigma^{1/2})^{1/2}\Sigma^{-1/2} -I\right)^2 \Sigma).
    \end{align*}

    In the case where $\Lambda$ and $\Sigma$ commute, one may work over an eigenbasis and sharpen the bound \Cref{eq:yifeiGeneralBound} to each dimension. In particular, the transport map simplifies to $T(x) = \Lambda^{1/2} \Sigma^{-1/2}x$, and the bound then becomes
    \begin{align*}
        E(\rho) &\le  -\Tr((\Lambda^{1/2}\Sigma^{-1/2}-I)(\Lambda^{-1}-\Sigma^{-1})\Sigma)\\
        &\quad - \frac{1}{2}\Tr(\Lambda^{-1}\left(\Lambda^{1/2}\Sigma^{-1/2}-I\right)^2\Sigma)\\
        &= \Tr(\Lambda^{1/2}(\Lambda^{-1/2} - \Sigma^{-1/2})(\Lambda^{-1} - \Sigma^{-1}))- \frac{1}{2}\Tr(\Sigma(\Lambda^{-1/2} - \Sigma^{-1/2})^2)\\
        &= \Tr(\Sigma \Lambda^{1/2} B_-^2 B_+ - \Sigma B_-^2/2),
    \end{align*}
    as desired.
\end{proof}

\section{Discrete Time ARWP Covariance Update}\label{app:discreteTimeGaussianUpd}
Let $V(x) = x^\top \Lambda^{-1} x/2$. Recall the original accelerated regularized kinetic system \labelcref{eqs:symplecticARWP}, updated as
\begin{equation}
    \begin{cases}
        P_{k+1} = (1-a\eta) P_k - \eta (\nabla V(X_k) + \nabla \log \wprox \rho_k(X_k)),\\
        X_{k+1} = X_k + \eta P_{k+1}.
    \end{cases}
\end{equation}
We recall the essential expressions relating the covariance of a Gaussian to its regularized Wasserstein proximal
\begin{equation}
    K_\pm = 1\pm T \lambda^{-1} ,\quad B_\pm = \lambda^{-1/2} \pm \tilde\sigma^{-1/2}.
\end{equation}
\begin{equation}
    \tilde\Sigma_t = 2TK_+^{-1} + K_+^{-1} \Sigma_t K_+^{-1},\quad \Sigma_t = K_+(\tilde\Sigma_t - 2TK_+^{-1}) K_+.
\end{equation}

Let us write $P_{k+1} = G_k X_k$, where $G_k$ is some matrix. Then, 
\begin{equation}
    X_{k+1} = X_k + \eta P_{k+1} = (I + \eta G_k) X_k,
\end{equation}
and the update of the momentum term
\begin{align*}
    G_k X_k = P_{k+1} &= (1-a\eta) P_k - \eta(\Lambda^{-1} - \tilde\Sigma_k^{-1})X_k\\
    &= (1-a\eta) G_{k-1}X_{k-1}-\eta(\Lambda^{-1} - \tilde\Sigma^{-1}_k)X_k\\
    &= (1-a\eta)G_{k-1}(I+\eta G_{k-1})^{-1}X_k - \eta (\Lambda^{-1} - \tilde\Sigma_k)X_k.
\end{align*}
Therefore, the regularized WProx kinetic equation in covariance form becomes
\begin{equation}\label{eq:SigmaEvoDiscrete}
    X_{k+1} = (I+\eta G_k) X_k \Rightarrow \Sigma_{k+1} = (I + \eta G_k) \Sigma_k(I+\eta G_k)^{\top},
\end{equation}
\begin{equation}\label{eq:gupd}
    G_k = (1-a\eta) G_{k-1}(1+\eta G_{k-1})^{-1} - \eta(\Lambda^{-1} - \tilde\Sigma_k^{-1}).
\end{equation}
\section{Ablation on regularization parameters: Gaussian case}
We consider a 2D Gaussian target distribution $\Lambda = \diag(1, 20)$. \Cref{appfig:GaussianTComparison} ablates against the Wasserstein proximal regularization parameter $T$, which has to be fixed in $T \in [0,1)$. We observe that the optimal damping is given by \labelcref{eq:optimalDampingMax}, which is consistent with the linearization analysis in \Cref{ssec:DiscreteTimeLinearization}. Moreover, as $T$ decreases, the update becomes (numerically) unstable for small damping $a$. This is due to the update step \labelcref{eq:gupd}: if $T$ is small, then the regularized Wasserstein proximal variance matrix $\tilde\Sigma_k$ may also be very small, leading to a large update in the momentum matrix $G$. This makes the ODE system stiff, and can cause blowup if the step-size is not chosen to be sufficiently small.

\begin{figure}
\centering

\renewcommand{\arraystretch}{0}
\noindent\makebox[\textwidth]{
\begin{tblr}{
    colspec={ccccc}, colsep=-2pt
    }
    $T=0$ & $T=0.01$ & $T=0.1$ & $T=(1+\sqrt{2})^{-1}$ & $T = 0.9$ \\\hline
 \adjincludegraphics[width=0.25\textwidth,trim={0cm 0cm 1.5cm 1cm},clip,valign=c]{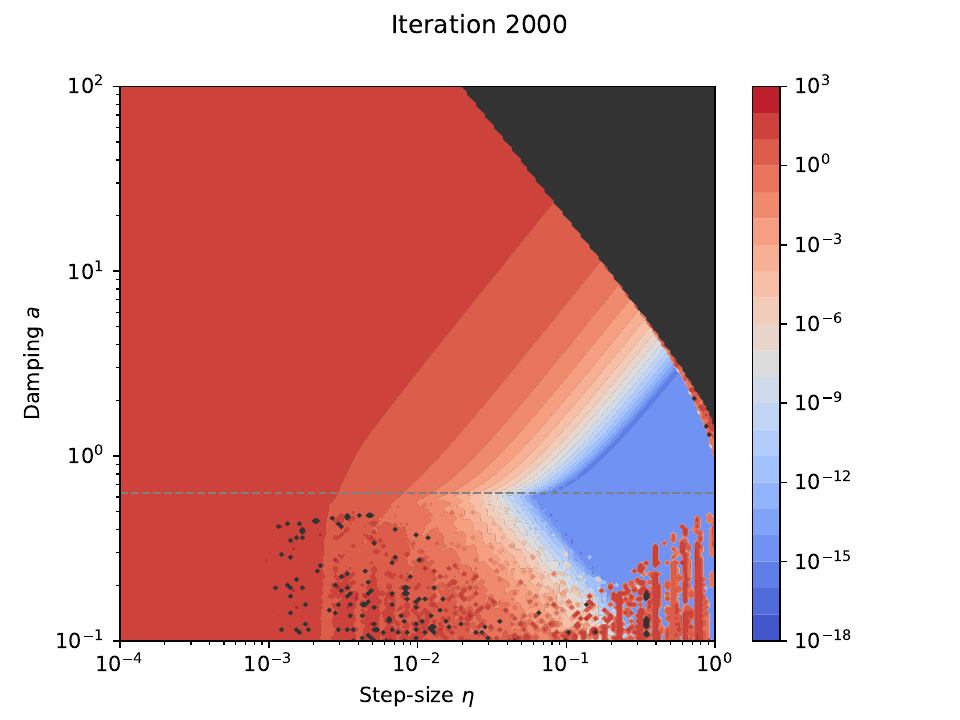}&
 \adjincludegraphics[width=0.25\textwidth,trim={0cm 0cm 1.5cm 1cm},clip,valign=c]{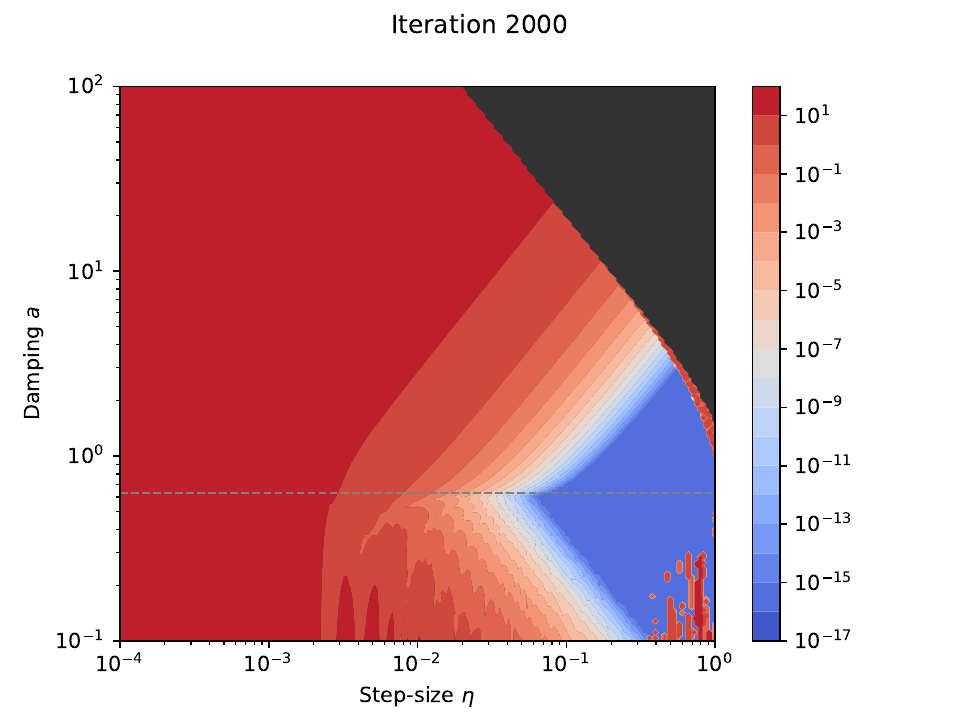}&
 \adjincludegraphics[width=0.25\textwidth,trim={0cm 0cm 1.5cm 1cm},clip,valign=c]{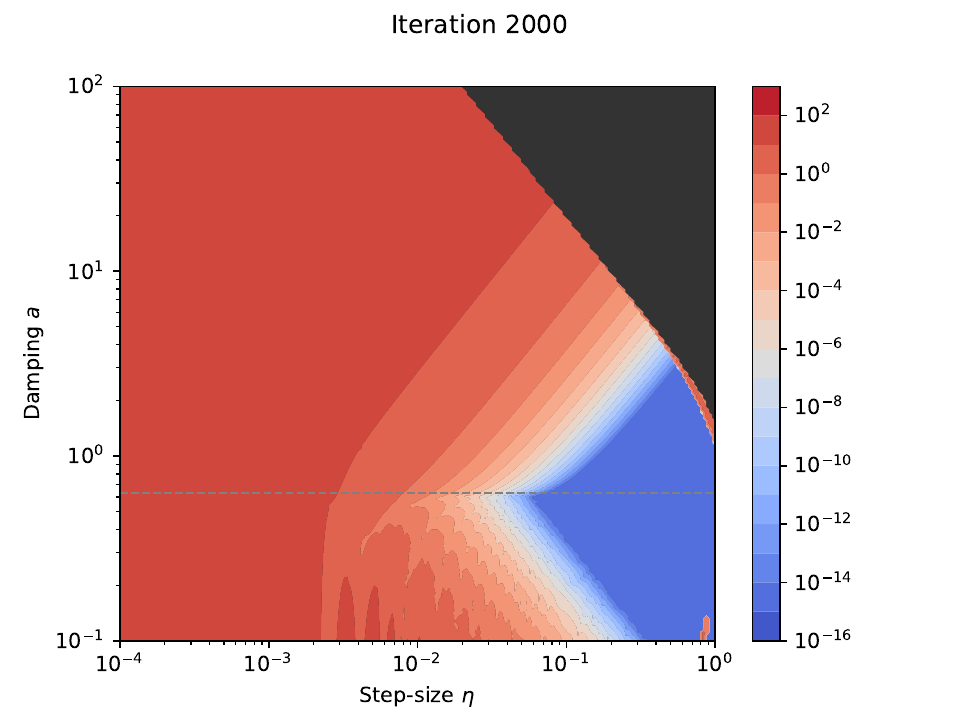}&
 \adjincludegraphics[width=0.25\textwidth,trim={0cm 0cm 1.5cm 1cm},clip,valign=c]{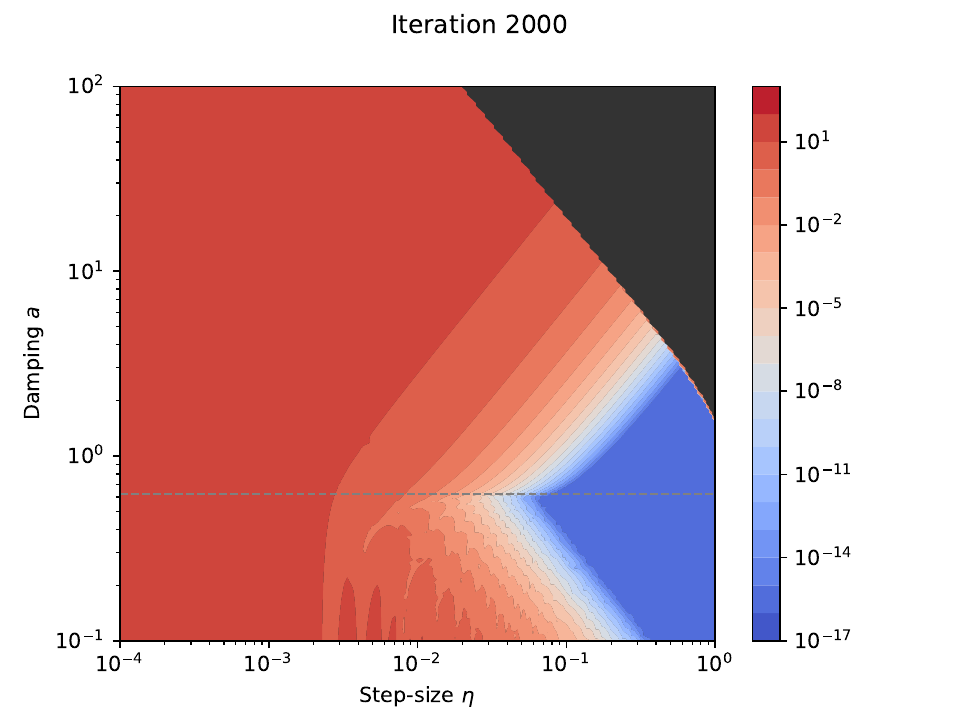}&
 \adjincludegraphics[width=0.25\textwidth,trim={0cm 0cm 1.5cm 1cm},clip,valign=c]{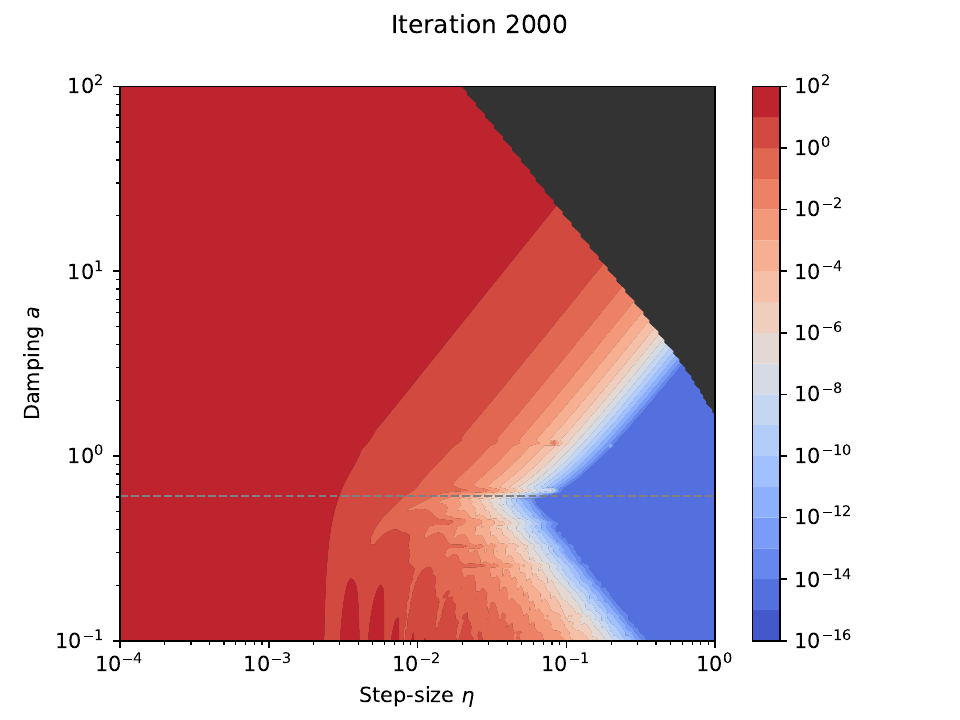}
\end{tblr}}
\caption{Contour plots of the covariance error of discrete-time ARWP \labelcref{eqs:symplecticARWP}, for target distribution $\Lambda = \gN(0, \diag(1,20))$, with initialization $\gN(0,40I)$, at 2000 iterations. We observe that for lower levels of $T$, numerical instability dominates in the highly underdamped setting $a<20^{-1/2}$, due to the covariance getting close to zero. }\label{appfig:GaussianTComparison}
\end{figure}

\section{Additional plots for Rosenbrock distribution}\label{appsec:rosenbrock}
\Cref{fig:rosenbrockAblation} plots the particle positions at iteration 200 for the Rosenbrock distribution. The particle count is $N=100$, step-size $\eta = 0.02$, and the parameters are varied as $T \in [0.02,0.05,0.1]$ and $a \in [2,5,15]$. We observe that as $T$ increases, the particles stay closer to the parabola $y=x^2$. This is consistent with the observation in \cite{tan2024noise} that the $T$ parameter ``shrinks'' the local variance by $T$. The evolution is very similar for the different damping parameters in this case, possibly indicating overdamping.
\begin{figure}
\centering
\renewcommand{\arraystretch}{0}
\noindent\makebox[\textwidth]{
\begin{tblr}{
    colspec={c|ccc}, colsep=1pt
    }
    & $T=0.02$ & $T=0.05$ & $T=0.1$\\ \hline
 $a=2$& \adjincludegraphics[width=0.35\textwidth,trim={1.3cm 0.7cm 1.7cm 1.2cm},clip,valign=c,decodearray={0 1 0 1 0.5 1}]{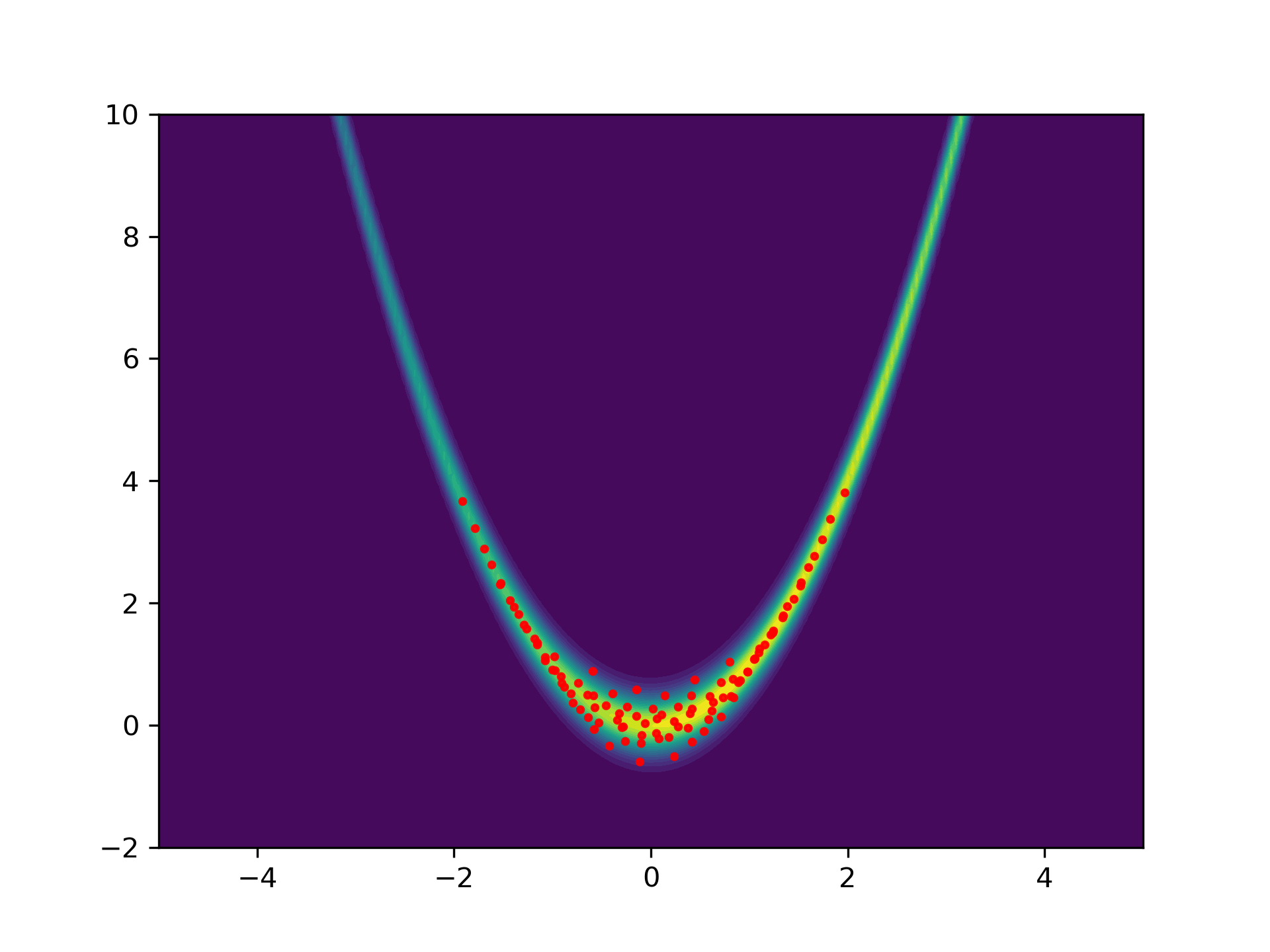}&
 \adjincludegraphics[width=0.35\textwidth,trim={1.3cm 0.7cm 1.7cm 1.2cm},clip,valign=c,decodearray={0 1 0 1 0.5 1}]{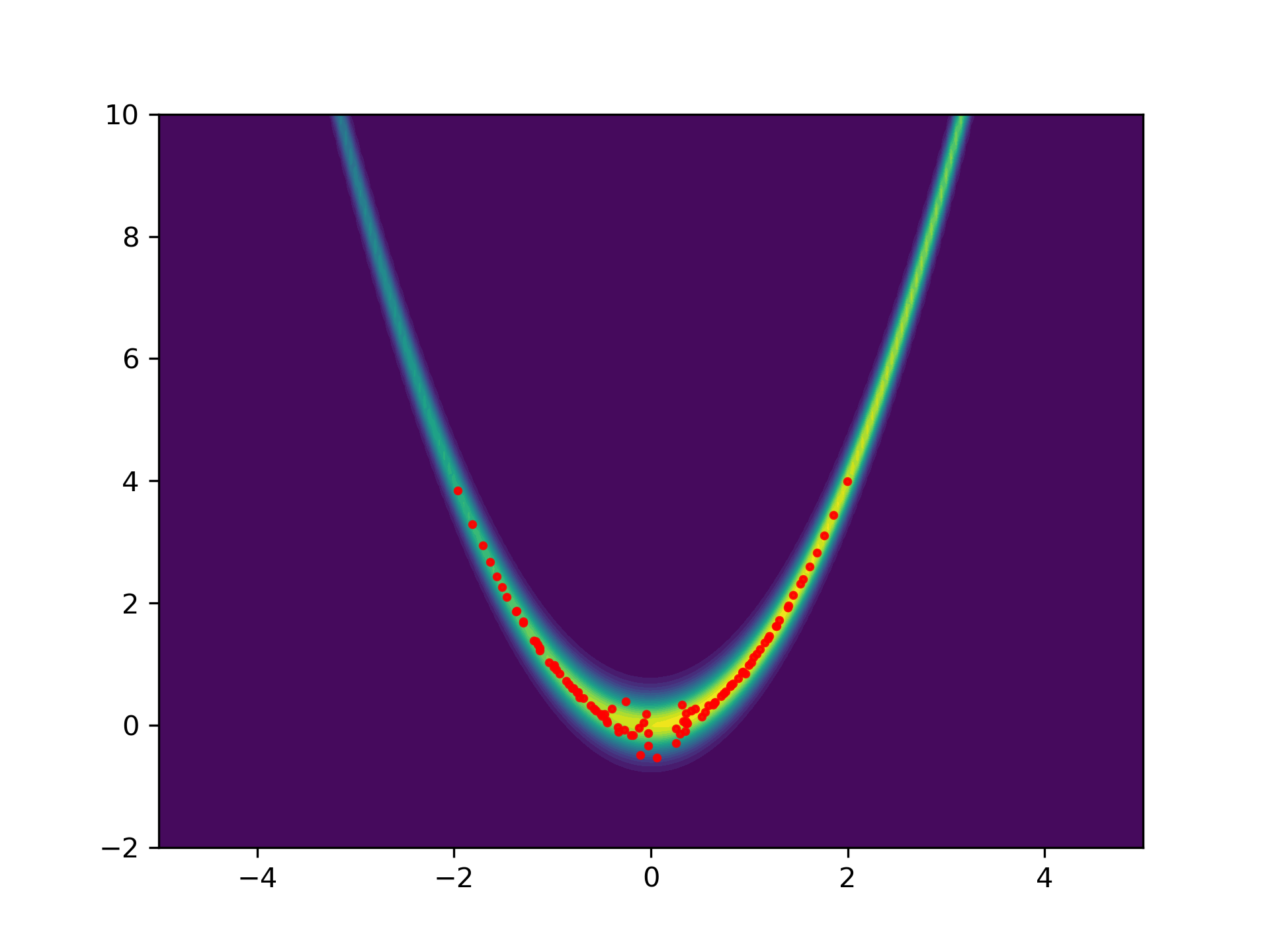}&
 \adjincludegraphics[width=0.35\textwidth,trim={1.3cm 0.7cm 1.7cm 1.2cm},clip,valign=c,decodearray={0 1 0 1 0.5 1}]{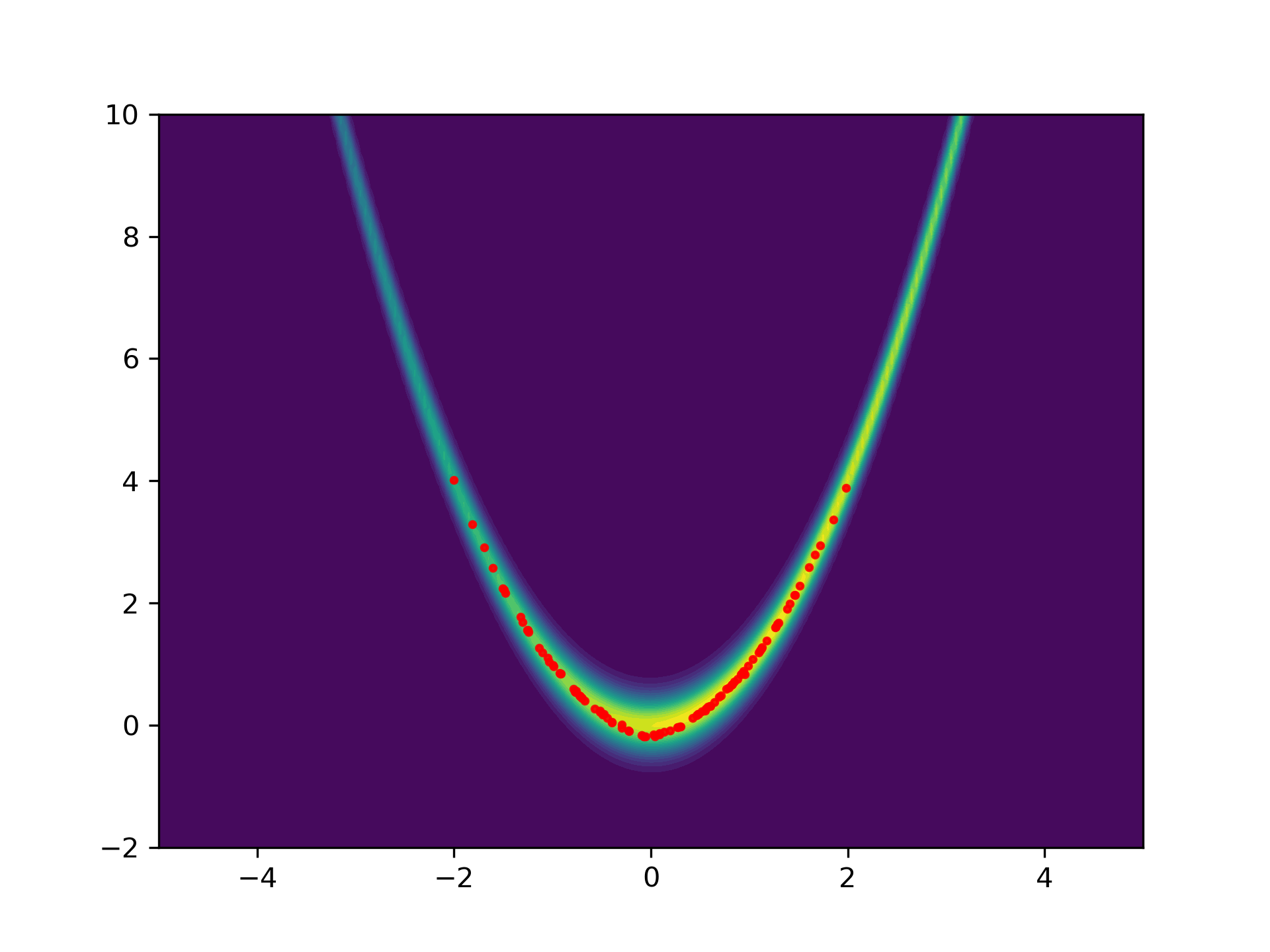}\\
 
 $a=5$& \adjincludegraphics[width=0.35\textwidth,trim={1.3cm 0.7cm 1.7cm 1.2cm},clip,valign=c,decodearray={0 1 0 1 0.5 1}]{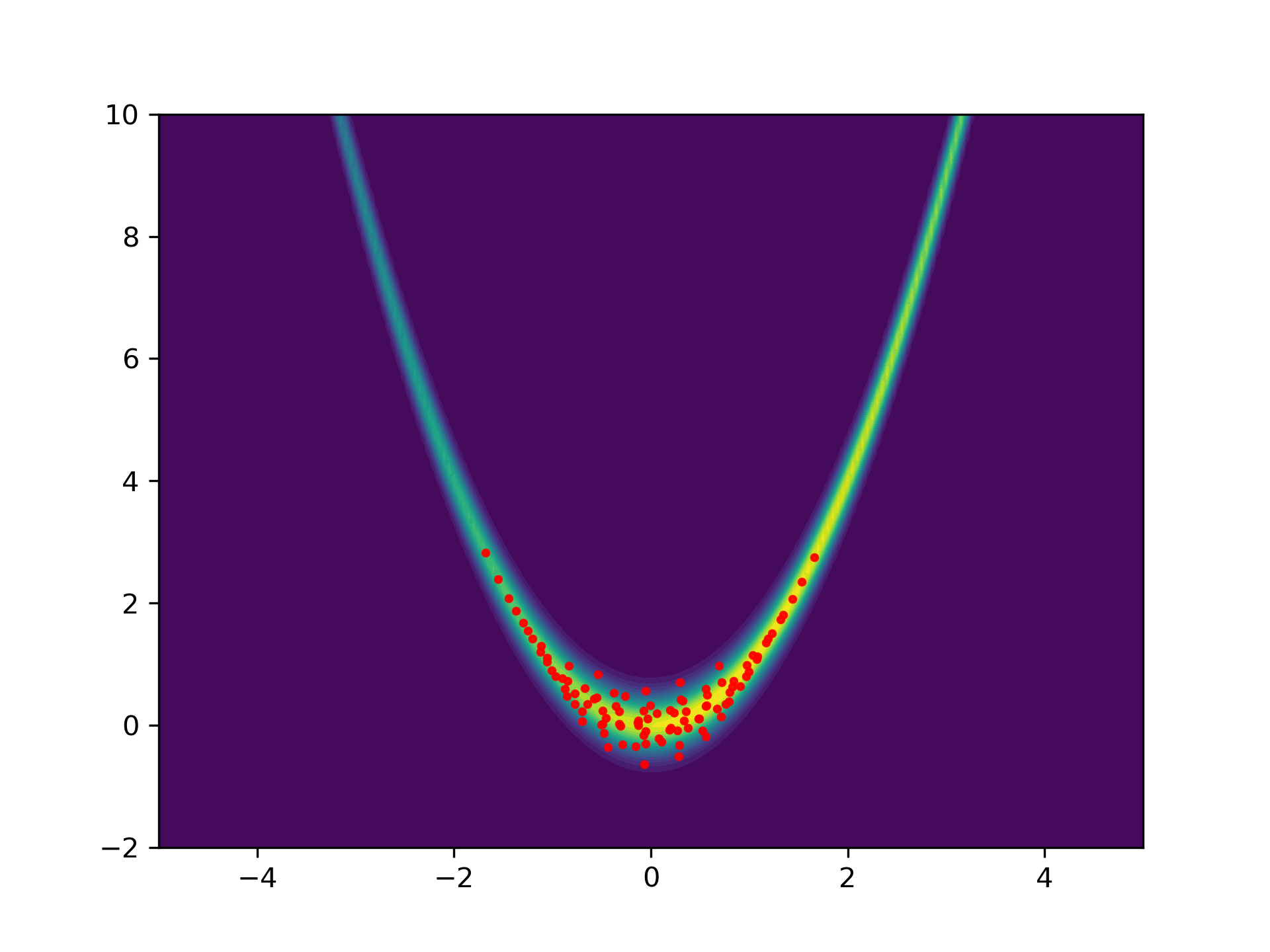}&
 \adjincludegraphics[width=0.35\textwidth,trim={1.3cm 0.7cm 1.7cm 1.2cm},clip,valign=c,decodearray={0 1 0 1 0.5 1}]{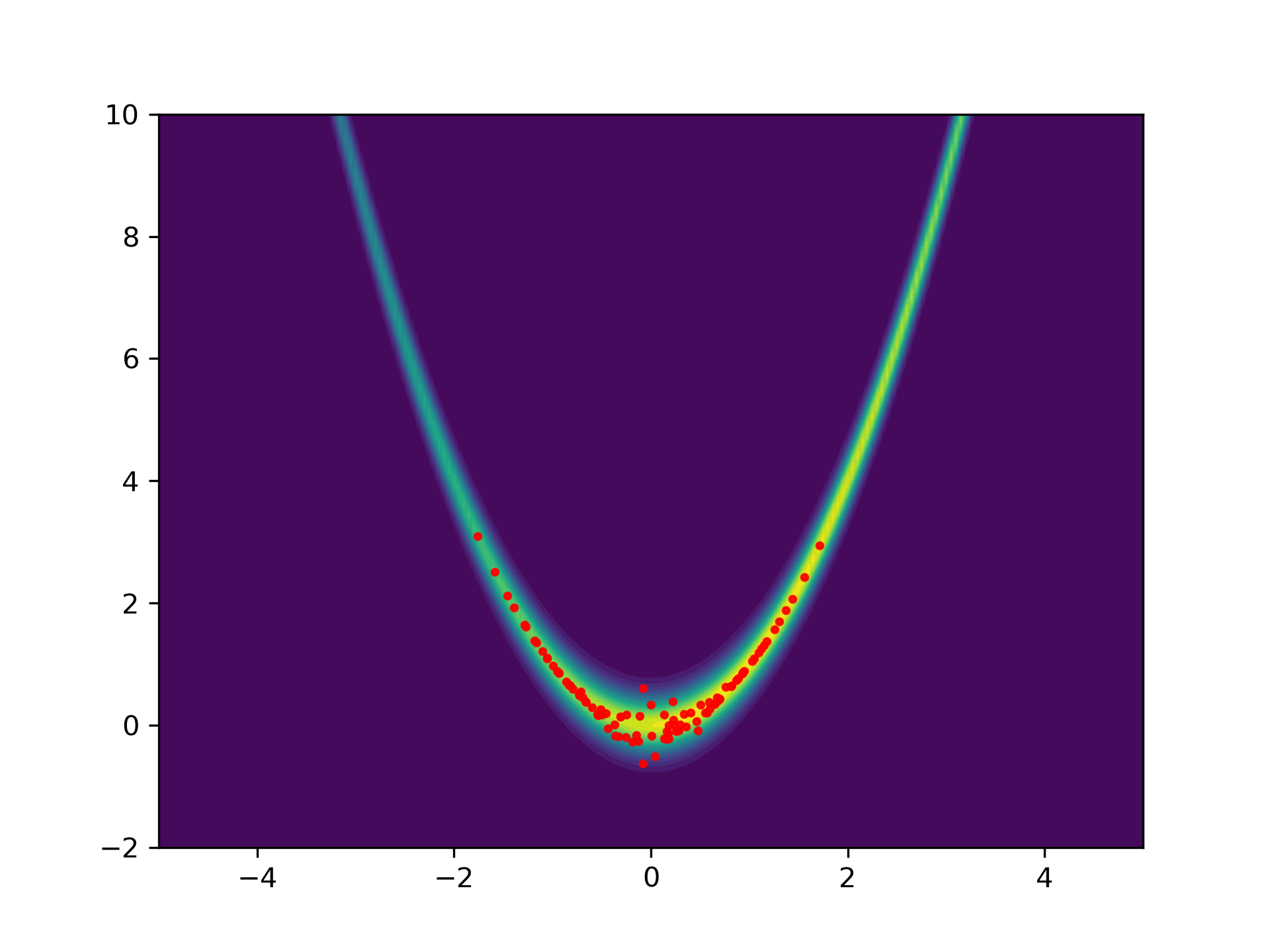}&
 \adjincludegraphics[width=0.35\textwidth,trim={1.3cm 0.7cm 1.7cm 1.2cm},clip,valign=c,decodearray={0 1 0 1 0.5 1}]{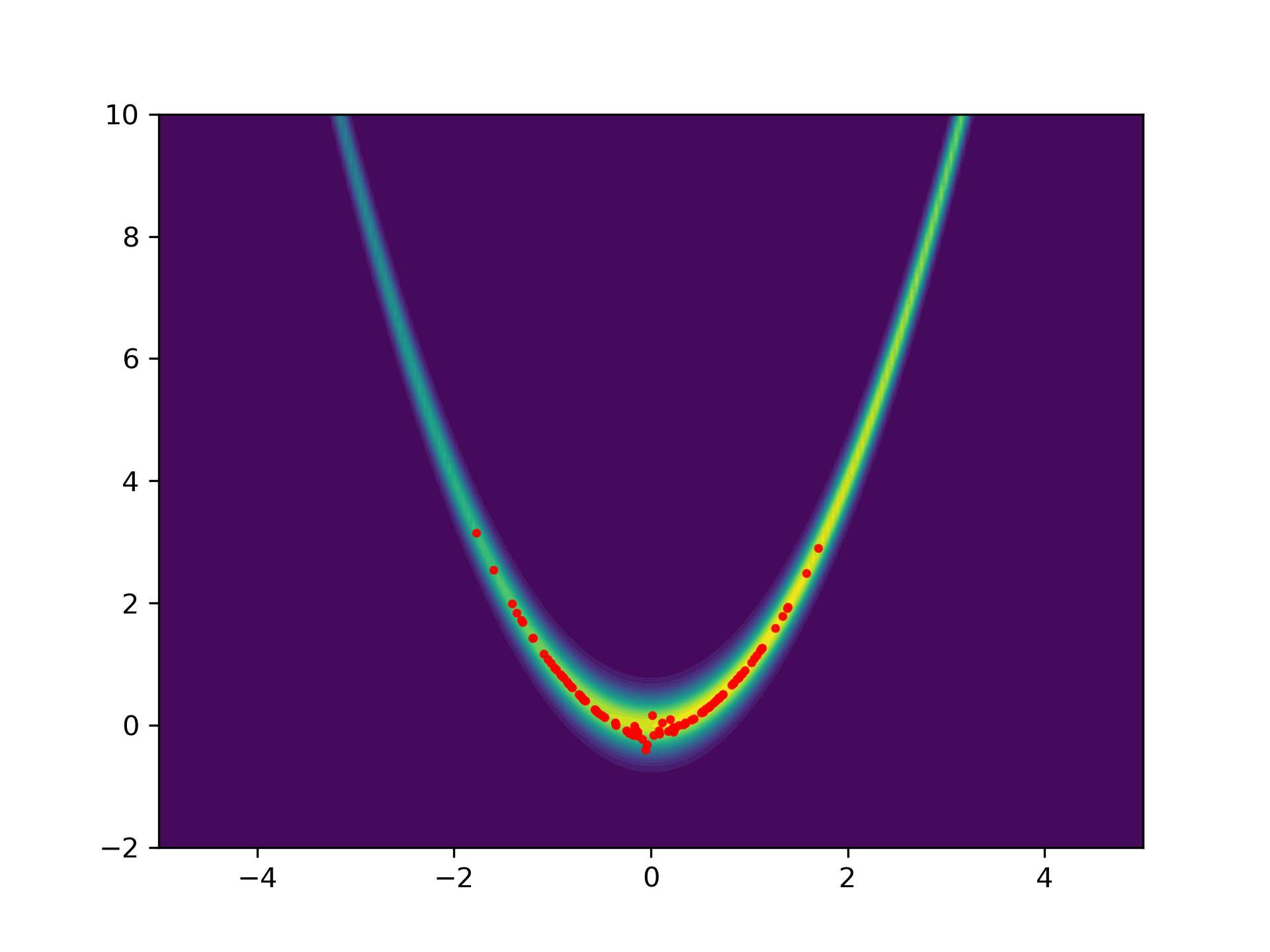}\\

 $a=15$& \adjincludegraphics[width=0.35\textwidth,trim={1.3cm 0.7cm 1.7cm 1.2cm},clip,valign=c,decodearray={0 1 0 1 0.5 1}]{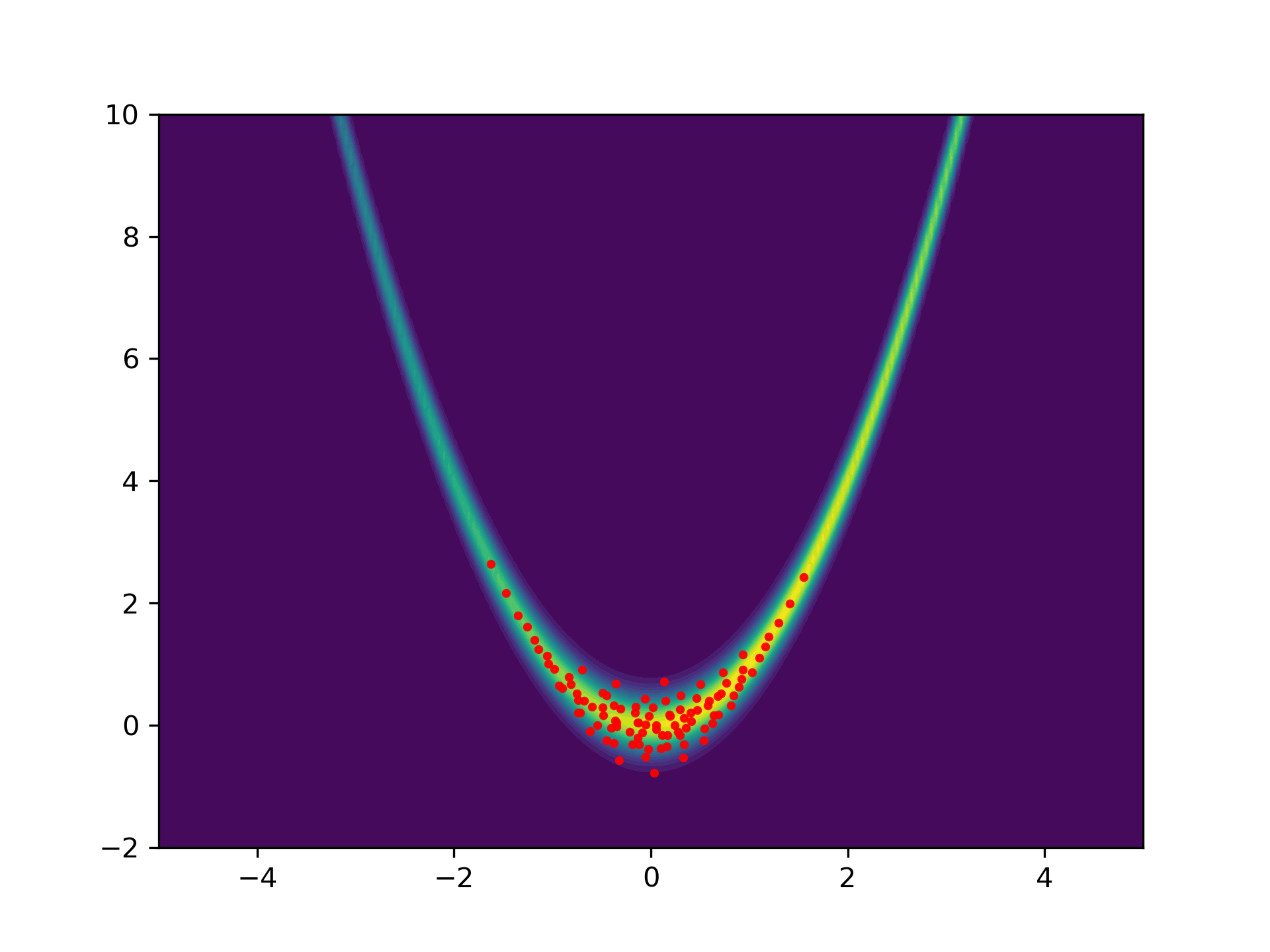}&
 \adjincludegraphics[width=0.35\textwidth,trim={1.3cm 0.7cm 1.7cm 1.2cm},clip,valign=c,decodearray={0 1 0 1 0.5 1}]{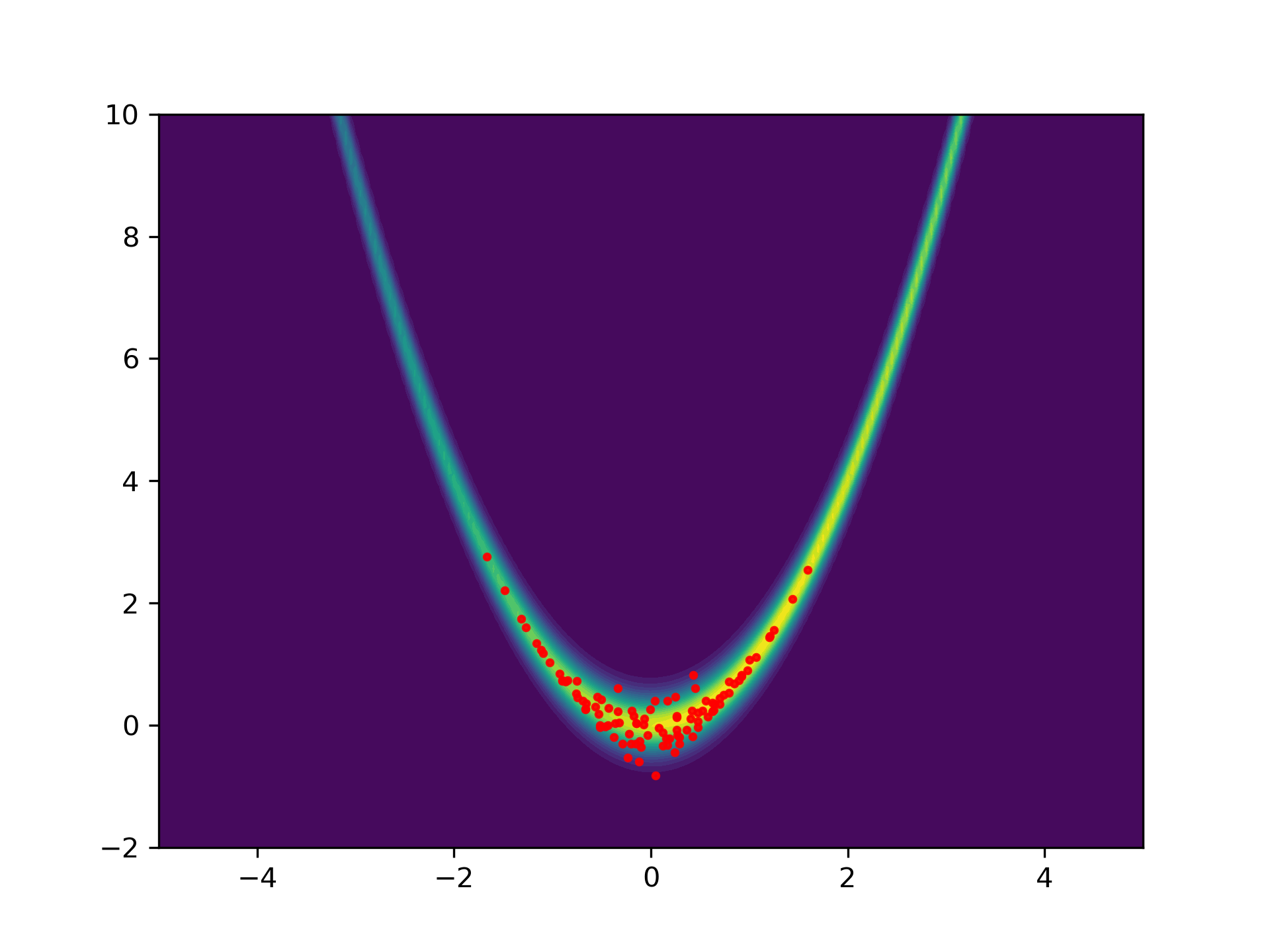}&
 \adjincludegraphics[width=0.35\textwidth,trim={1.3cm 0.7cm 1.7cm 1.2cm},clip,valign=c,decodearray={0 1 0 1 0.5 1}]{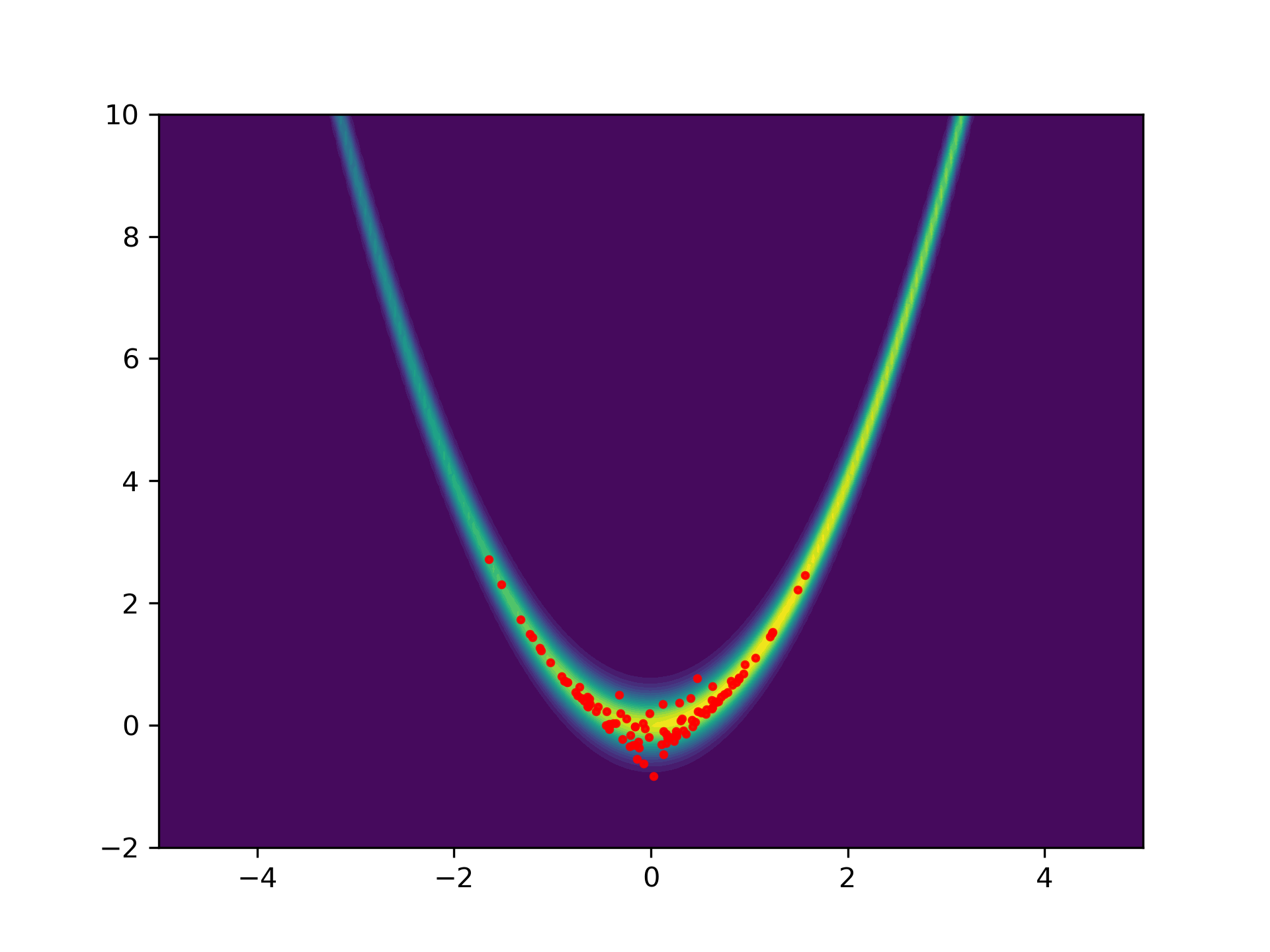}
\end{tblr}}
\caption{Ablation of ARWP-Heavy-ball across some different choices of $T \in \{0.02,0.05,0.1\}$ and $a \in \{2,5,15\}$. All evaluated at iteration 200 and step-size $\eta = 0.02$. As $T$ increases, the particles concentrate more around the parabola $y=x^2$.}\label{fig:rosenbrockAblation}
\end{figure}

To contrast this sensitivity, we compare with ILA, which has different behavior as the damping parameter changes as seen in \Cref{fig:rosenbrockILAComparison}. For the less damped case, corresponding to a small Lipschitz constant estimate, the particles are able to explore the tails. However, in the more damped case, the particles do not explore the tails. KLMC does not explore the tails at all, as seen in \Cref{fig:rosenbrockKLMCComparison}.

\begin{figure}
\centering

\renewcommand{\arraystretch}{0}
\noindent\makebox[\textwidth]{
\begin{tblr}{
    colspec={ccc}, colsep=1pt
    }
 Iter. 50 & 200 & 500\\\hline
 \adjincludegraphics[width=0.35\textwidth,trim={1.3cm 0.7cm 1.7cm 1.2cm},clip,valign=c,decodearray={0 1 0 1 0.5 1}]{figs/rosenbrock/ila0.05_2/50_zoom.png}&
 \adjincludegraphics[width=0.35\textwidth,trim={1.3cm 0.7cm 1.7cm 1.2cm},clip,valign=c,decodearray={0 1 0 1 0.5 1}]{figs/rosenbrock/ila0.05_2/200_zoom.png}&
 \adjincludegraphics[width=0.35\textwidth,trim={1.3cm 0.7cm 1.7cm 1.2cm},clip,valign=c,decodearray={0 1 0 1 0.5 1}]{figs/rosenbrock/ila0.05_2/500_zoom.png}\\

 \adjincludegraphics[width=0.35\textwidth,trim={1.3cm 0.7cm 1.7cm 1.2cm},clip,valign=c,decodearray={0 1 0 1 0.5 1}]{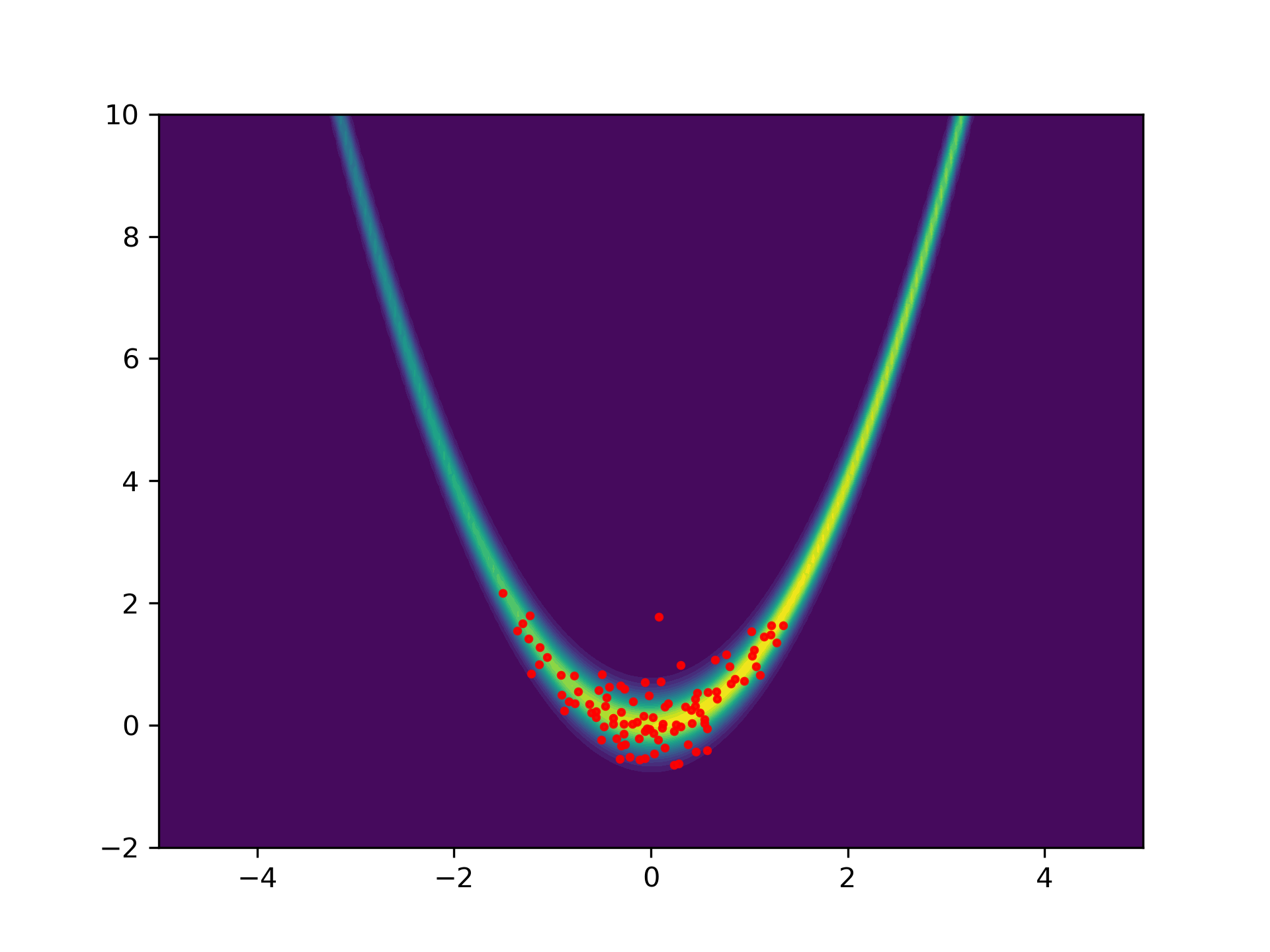}&
 \adjincludegraphics[width=0.35\textwidth,trim={1.3cm 0.7cm 1.7cm 1.2cm},clip,valign=c,decodearray={0 1 0 1 0.5 1}]{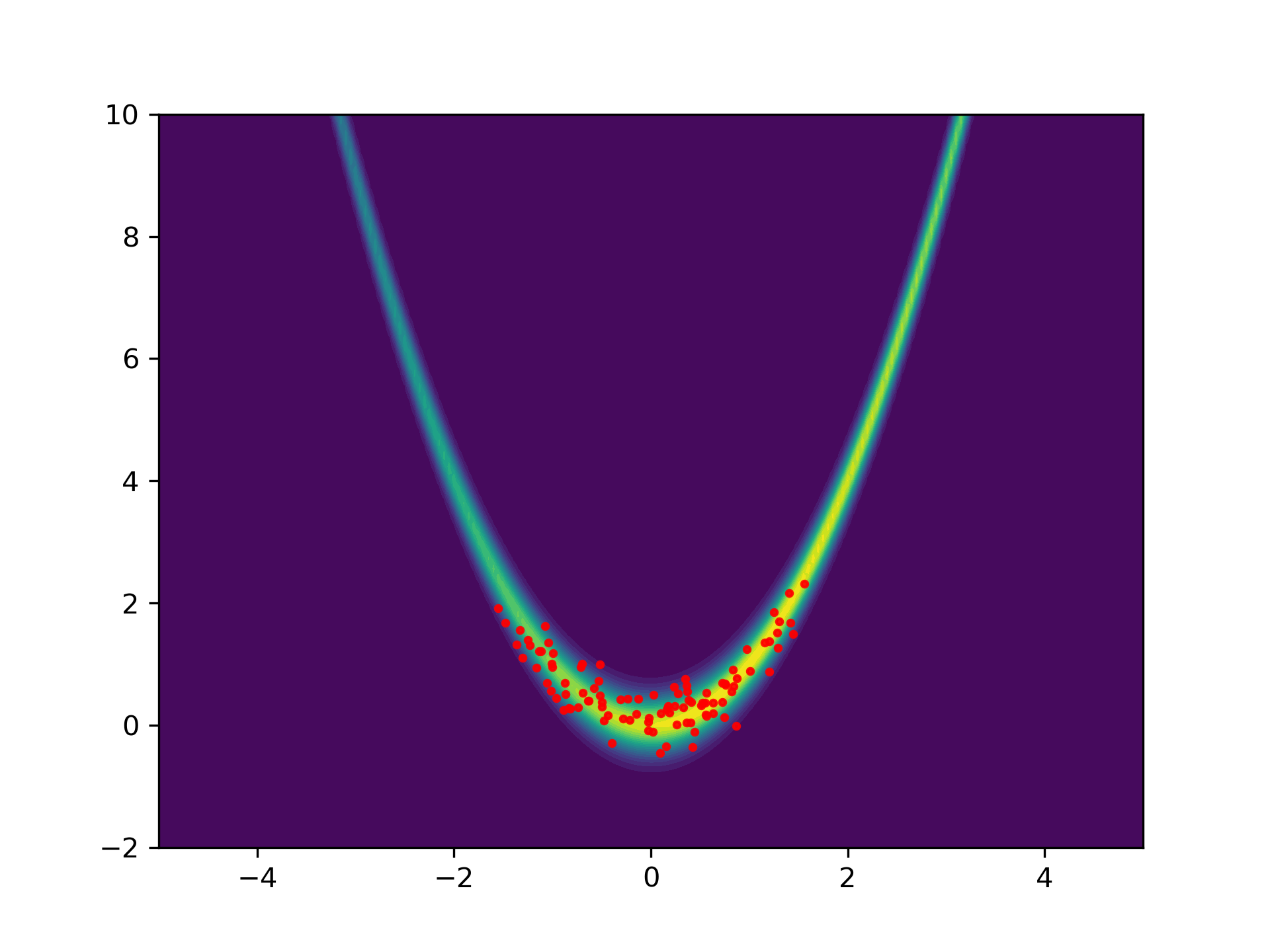}&
 \adjincludegraphics[width=0.35\textwidth,trim={1.3cm 0.7cm 1.7cm 1.2cm},clip,valign=c,decodearray={0 1 0 1 0.5 1}]{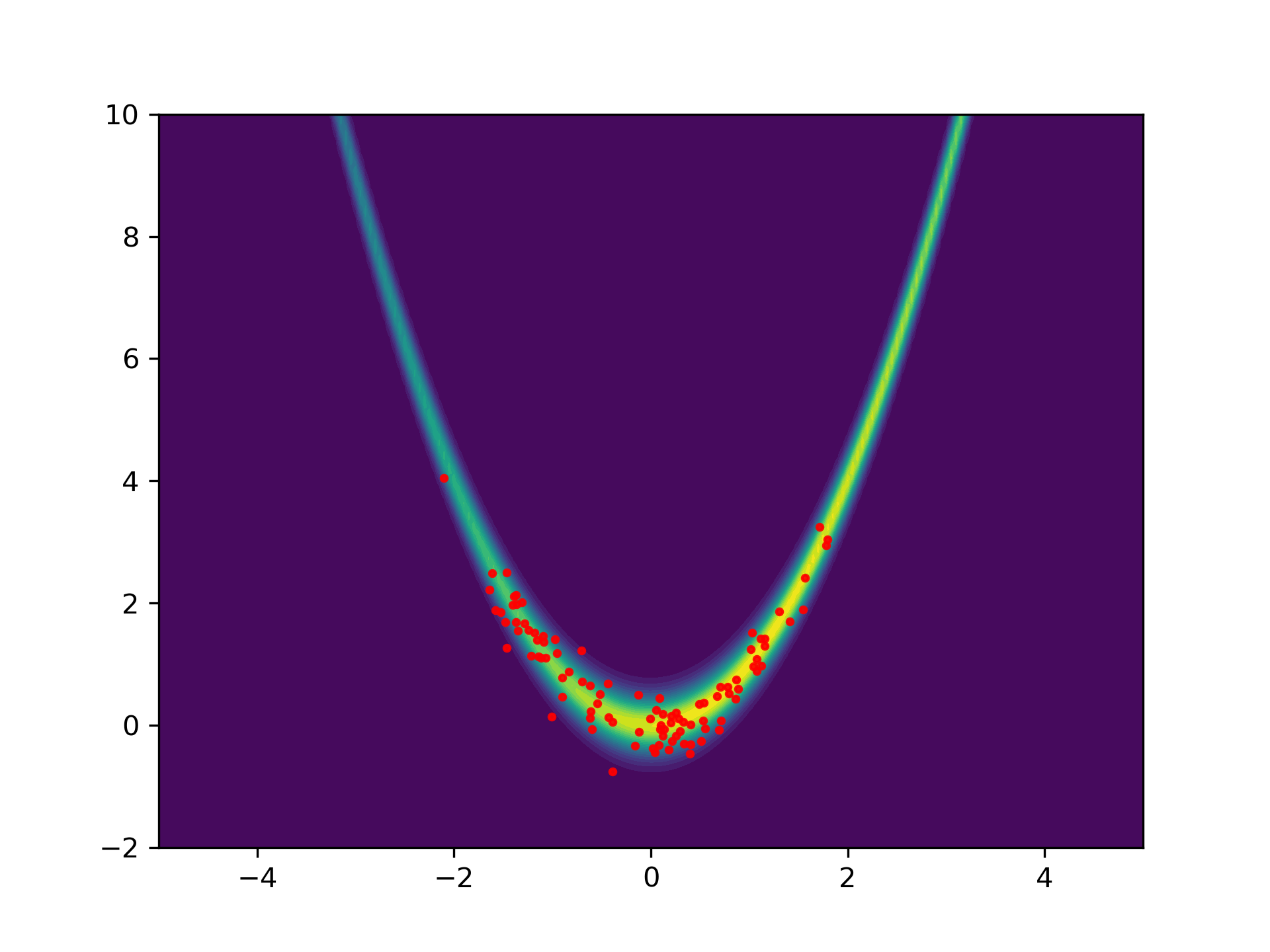}\\
\end{tblr}}
\caption{Top row: underdamped ILA ($\eta = 0.05,\ L=2$; bottom row: overdamped ILA ($\eta=0.05,\ L=10$). Evaluated at iterations 50, 200 and 500. Overdamped ILA does not explore the tails.}\label{fig:rosenbrockILAComparison}
\end{figure}

\begin{figure}
\centering

\renewcommand{\arraystretch}{0}
\noindent\makebox[\textwidth]{
\begin{tblr}{
    colspec={ccc}, colsep=1pt
    }
 Iter. 50 & 200 & 500\\\hline
 \adjincludegraphics[width=0.35\textwidth,trim={1.3cm 0.7cm 1.7cm 1.2cm},clip,valign=c,decodearray={0 1 0 1 0.5 1}]{figs/rosenbrock/klmc0.01_5/50_zoom.png}&
 \adjincludegraphics[width=0.35\textwidth,trim={1.3cm 0.7cm 1.7cm 1.2cm},clip,valign=c,decodearray={0 1 0 1 0.5 1}]{figs/rosenbrock/klmc0.01_5/200_zoom.png}&
 \adjincludegraphics[width=0.35\textwidth,trim={1.3cm 0.7cm 1.7cm 1.2cm},clip,valign=c,decodearray={0 1 0 1 0.5 1}]{figs/rosenbrock/klmc0.01_5/500_zoom.png}\\

 \adjincludegraphics[width=0.35\textwidth,trim={1.3cm 0.7cm 1.7cm 1.2cm},clip,valign=c,decodearray={0 1 0 1 0.5 1}]{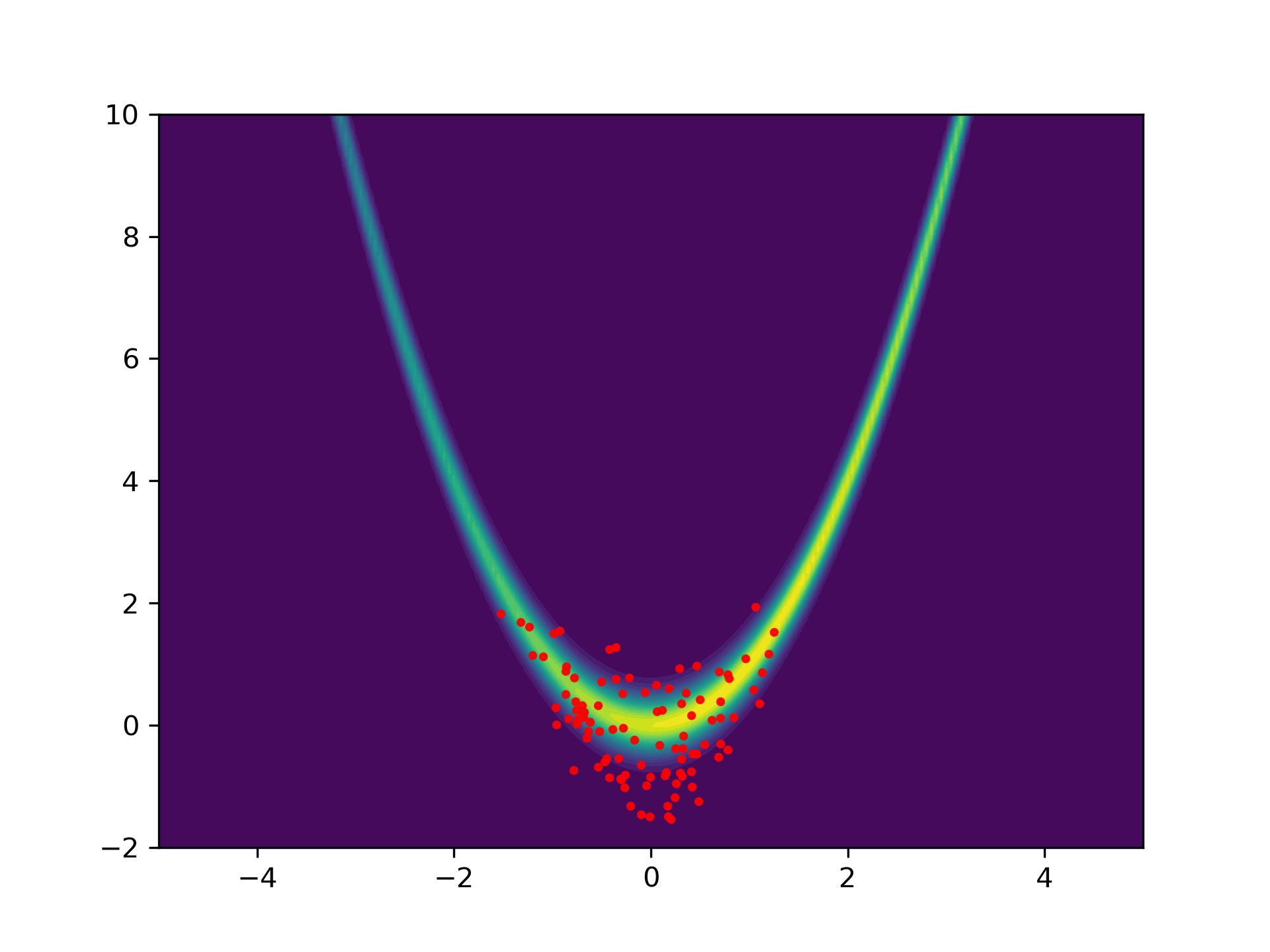}&
 \adjincludegraphics[width=0.35\textwidth,trim={1.3cm 0.7cm 1.7cm 1.2cm},clip,valign=c,decodearray={0 1 0 1 0.5 1}]{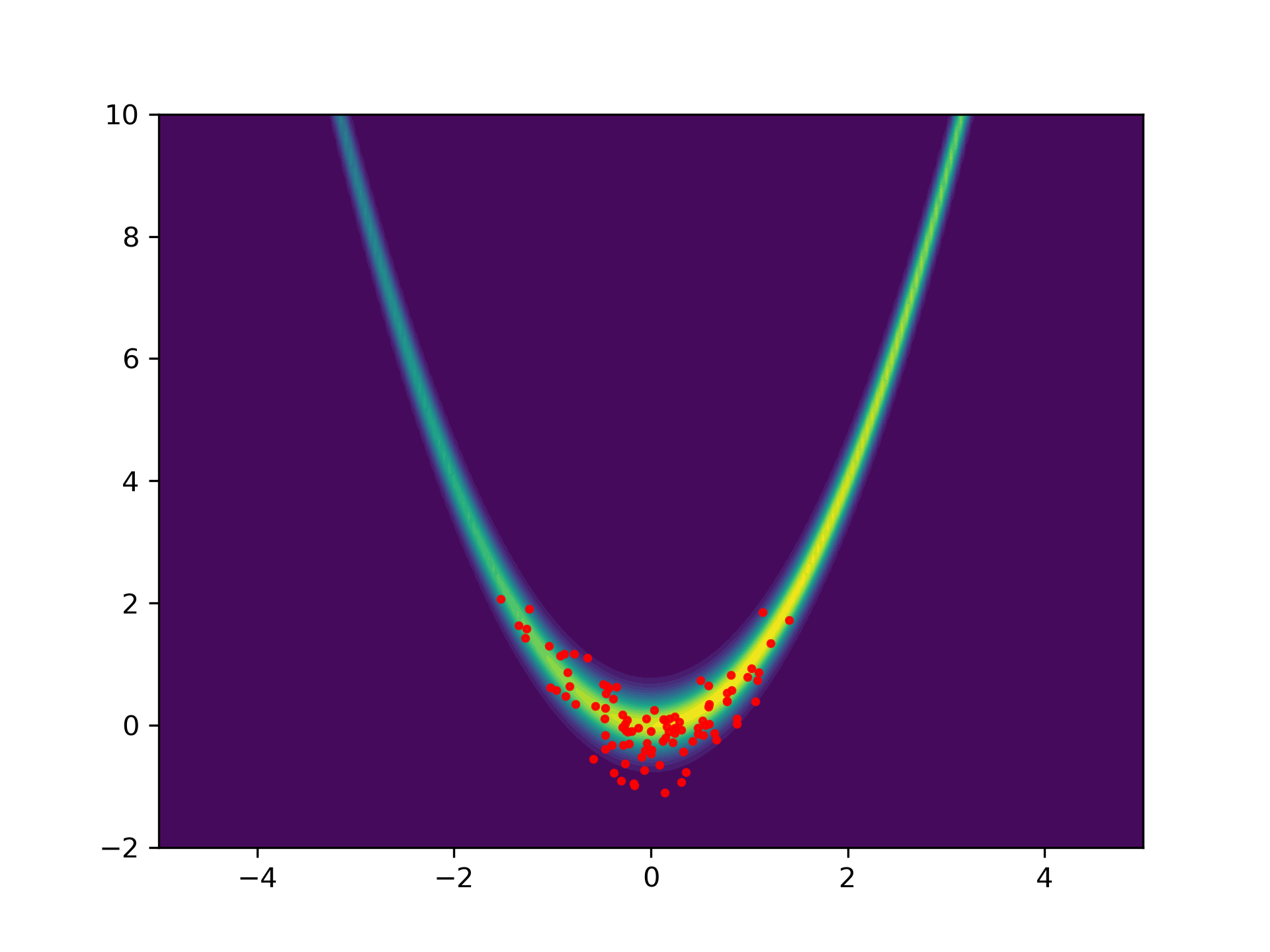}&
 \adjincludegraphics[width=0.35\textwidth,trim={1.3cm 0.7cm 1.7cm 1.2cm},clip,valign=c,decodearray={0 1 0 1 0.5 1}]{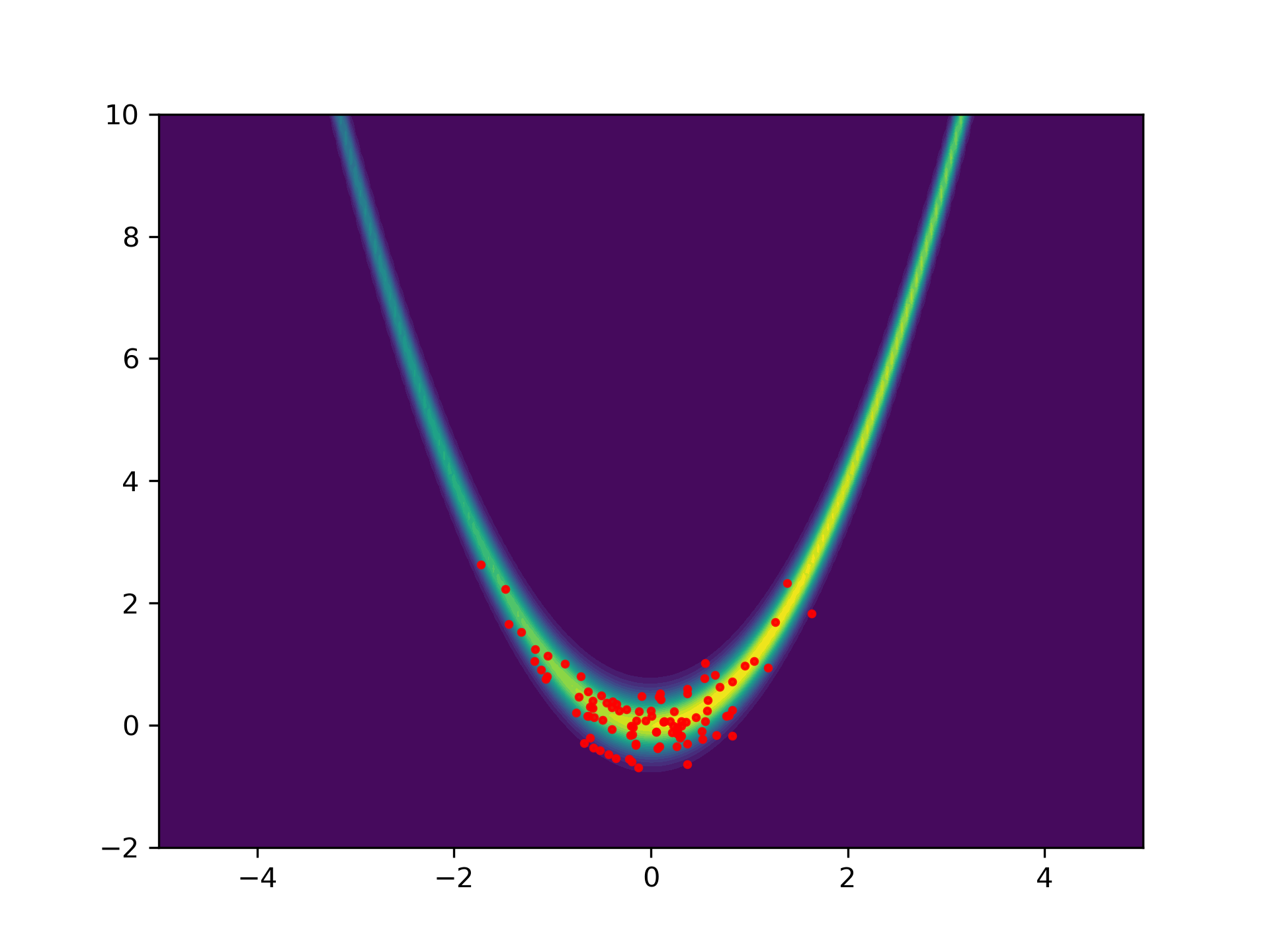}\\
\end{tblr}}
\caption{Top row: underdamped KLMC ($\eta = 0.01,\ a=5$; bottom row: overdamped KLMC ($\eta=0.01,\ a=20$). Evaluated at iterations 50, 200 and 500. There is no exploration.}\label{fig:rosenbrockKLMCComparison}
\end{figure}

ARWP has an additional advantage over BRWP in its stability with respect to step-size. BRWP is only able to diffuse up to a certain height of the parabola before some instabilities occur, and the outermost particles begin to oscillate perpendicularly to the parabola.

\begin{figure}
\centering

\renewcommand{\arraystretch}{0}
\noindent\makebox[\textwidth]{
\begin{tblr}{
    colspec={ccc}, colsep=1pt
    }
 Iter. 50 & 200 & 500\\\hline
 \adjincludegraphics[width=0.35\textwidth,trim={1.3cm 0.7cm 1.7cm 1.2cm},clip,valign=c,decodearray={0 1 0 1 0.5 1}]{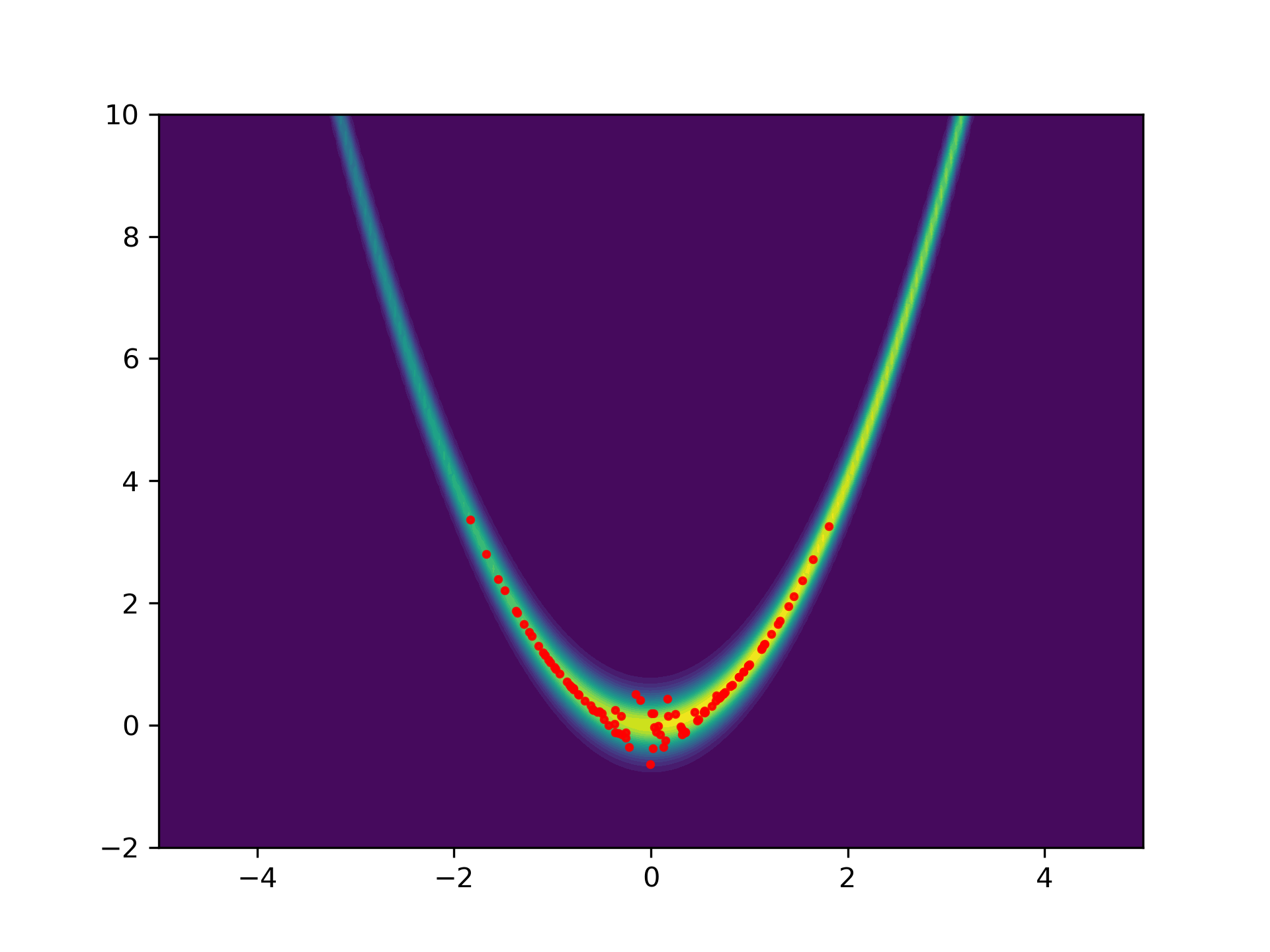}&
 \adjincludegraphics[width=0.35\textwidth,trim={1.3cm 0.7cm 1.7cm 1.2cm},clip,valign=c,decodearray={0 1 0 1 0.5 1}]{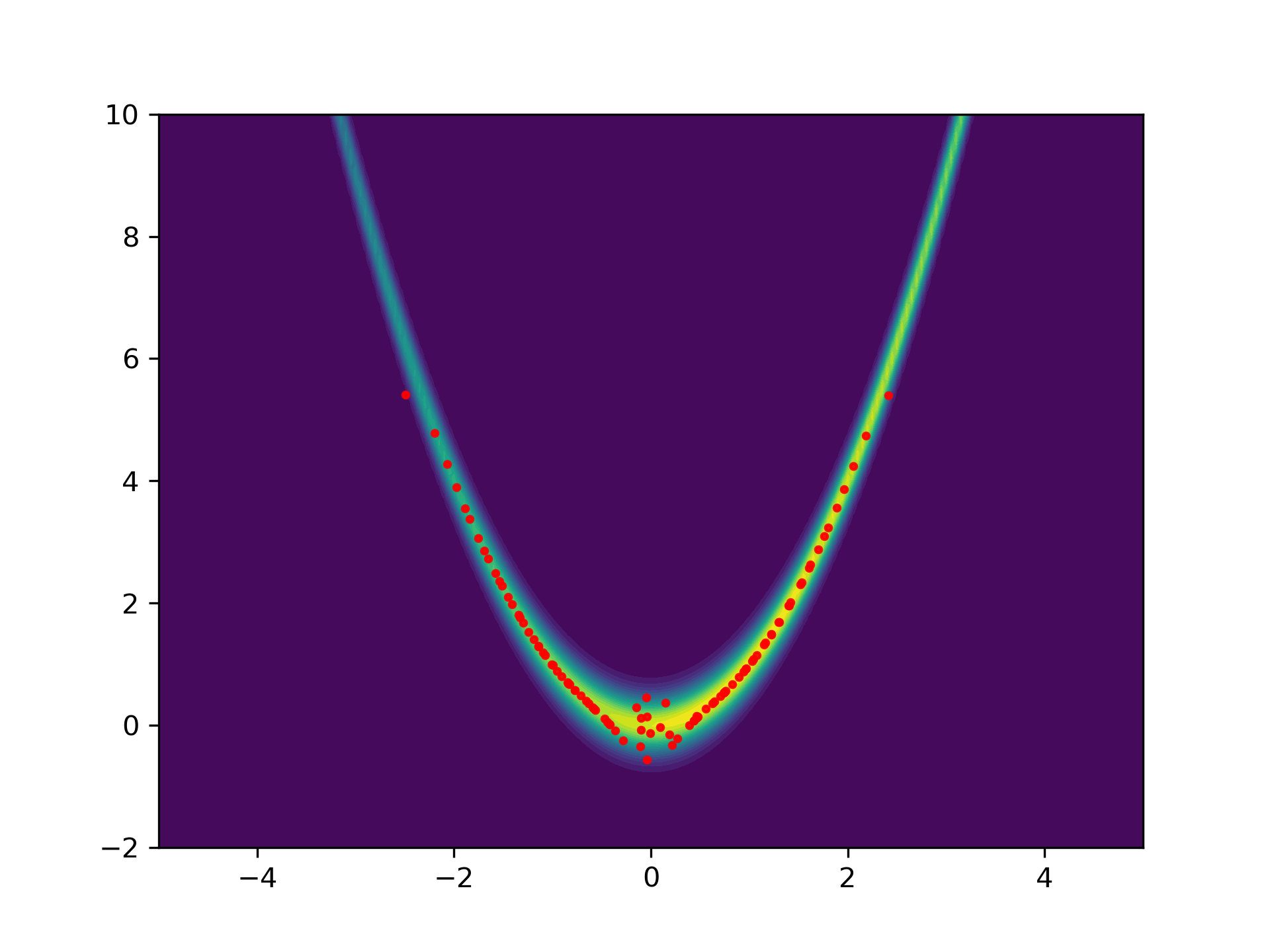}&
 \adjincludegraphics[width=0.35\textwidth,trim={1.3cm 0.7cm 1.7cm 1.2cm},clip,valign=c,decodearray={0 1 0 1 0.5 1}]{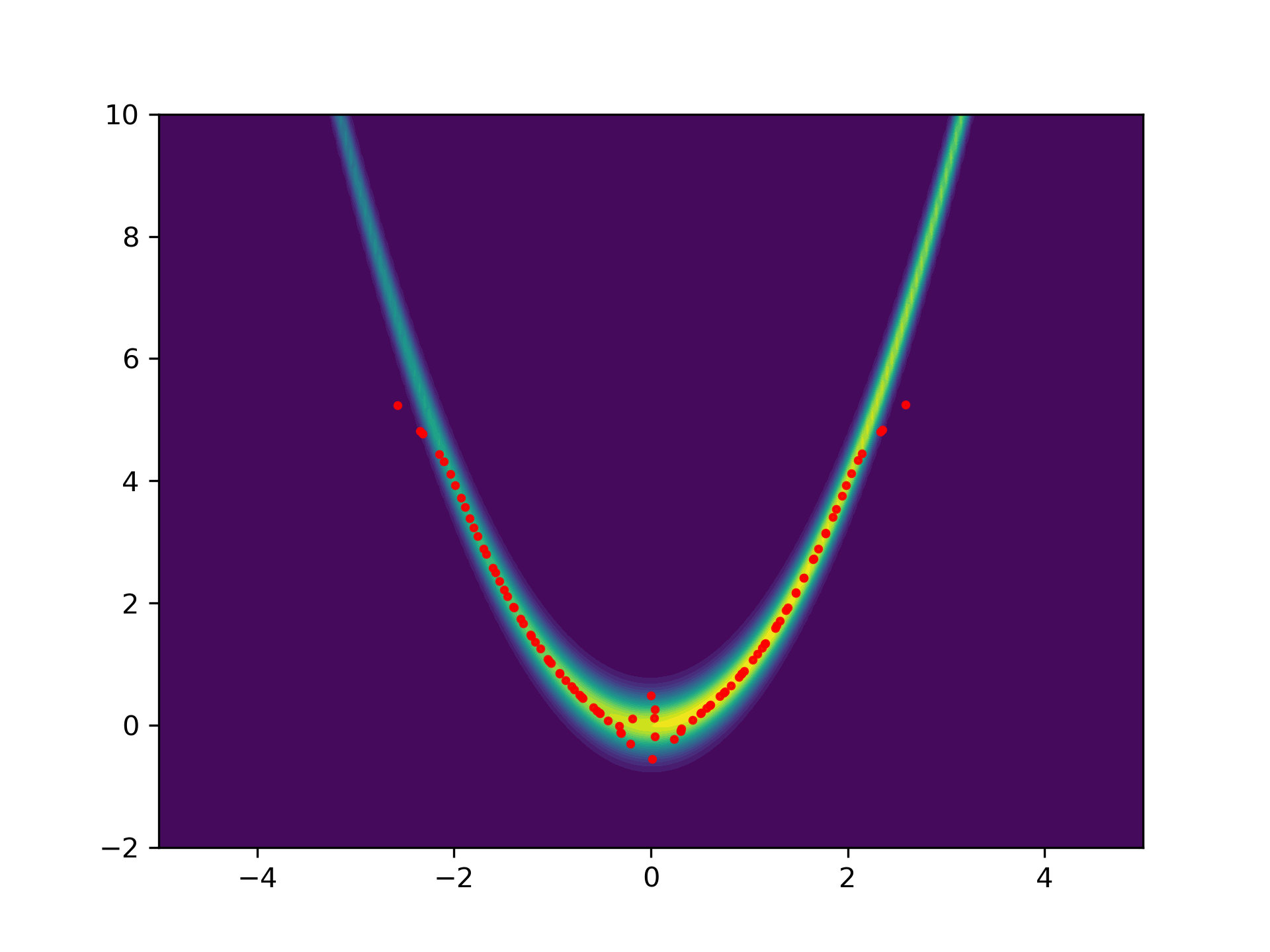}
\end{tblr}}
\caption{Rosenbrock distribution with BRWP ($\eta = 0.02,\ T=0.05$). Evaluated at iterations 50, 200 and 500. The convergence is fast but does not continue to explore the tails, as the outermost particles oscillate back and forth about the parabolic valley.}\label{fig:rosenbrockBRWPComparison}
\end{figure}

\section{Hyperparameters for experiments}\label{appsec:experimentParameters}
We detail the hyperparameters used in the plots for the ill-conditioned Gaussian \Cref{ssec:illCondGaussian}, Gaussian mixture \Cref{ssec:gmm}, and Bayesian neural network examples \Cref{ssec:bnn}. These are given respectively in \Cref{apptab:gaussian,apptab:gmm,apptab:bnn}. The parameters for Rosenbrock experiments are given in the previous section \Cref{appsec:rosenbrock}.

\begin{table}[ht]
    \centering
    \begin{tabular}{lc}
    \toprule 
    Method & Parameters \\
    \midrule
         ULA & $\eta = 0.01$ \\
         MALA & $\eta = 0.05$\\
         BRWP & $\eta = 0.2,\ T = 0.05$\\
         ARWP-Nesterov & $\eta = 0.3,\ T=0.05$\\
         ARWP-Heavy-ball & $\eta = 0.3,\ T = 0.05,\ a=1$\\
         ILA & $\eta = 0.2,\ L=5$\\
         KLMC & $\eta = 0.2,\ a=5$\\\bottomrule
    \end{tabular}
    \caption{Ill-conditioned Gaussian}
    \label{apptab:gaussian}
\end{table}

\begin{table}[ht]
    \centering
    \begin{tabular}{lc}
    \toprule 
    Method & Parameters \\
    \midrule
         ULA & $\eta =0.1$ \\
         MALA & $\eta = 0.3$\\
         BRWP & $\eta = 0.5,\ T = 0.1$\\
         ARWP-Nesterov & $ \eta = 0.6,\ T = 0.1$\\
         ARWP-Heavy-ball & $\eta = 0.6,\ T=0.1,\ a=1$\\
         ILA & $\eta = 0.3,\ L=1$\\
         KLMC & $\eta = 0.2,\ a=2$\\\bottomrule
    \end{tabular}
    \caption{GMM}
    \label{apptab:gmm}
\end{table}

\begin{table}[ht]
    \centering
    \begin{tabular}{lc}
    \toprule 
    Dataset & Parameters\\\midrule
    Boston & $\eta = 0.02,\ T = 0.01$\\
    Combined & $\eta = 0.02,\ T=0.01$\\
    Concrete & $\eta = 0.15,\ T=0.01$\\
    Kin8nm & $\eta = 0.02,\ T = 0.01$\\
    Wine & $\eta = 0.1,\ T=0.005$%
    \\\bottomrule
    \end{tabular}
    \caption{BNN}
    \label{apptab:bnn}
\end{table}

\section{Discretized Kinetic Langevin Algorithms}\label{appsec:DiscretizedKineticLan}
We detail the ILA and KLMC algorithms here, which were used as benchmarks for the low-dimensional examples.
\subsection{Inertial Langevin Algorithm}
The inertial Langevin algorithm (ILA) \cite[Alg. 1.1]{falk2025inertial} takes the following form. A particle is initialized with position $X_0$ and velocity $P_0 = 0$. The hyperparameters are a friction coefficient $\varepsilon \in [4/3, 7/4]$, a step-size $\Delta t$, and the Lipschitz constant of $\nabla V$ denoted as $L$. The ILA update then defines two new parameters:
\begin{equation*}
    \beta \coloneqq 1 - \varepsilon \Delta t,\quad \tau \coloneqq \Delta t^2/L,
\end{equation*}
and updates the particle position as
\begin{align*} 
    X_{k+1} &= X_k, \\
    P_{k+1} &= X_{k+1} - X_k.
\end{align*}
For simplicity, we take the friction $\varepsilon = 1.5$, and use the Lipschitz constant as a tunable damping coefficient.

\subsection{Kinetic Langevin Monte Carlo}
The KLMC algorithm is given in \cite{cheng2018underdamped,dalalyan2020sampling}. The algorithm is given by an exponential integrator, detailed as follows. For a damping parameter $a$ and a step-size $\eta$, the KLMC update is given by \cite[Sec. 3]{dalalyan2020sampling}:
\begin{equation}
    \begin{bmatrix}
        X_{k+1}\\
        P_{k+1}
    \end{bmatrix} = \begin{bmatrix}
        X_k + \psi_1(\eta) P_k - \psi_2(\eta) \nabla V(X_k)\\
        \psi_0(\eta) P_k - \psi_1(\eta) \nabla V(X_k)
    \end{bmatrix} + \sqrt{2a} \begin{bmatrix}
        \xi_{k+1}\\
        \xi_{k+1}'
    \end{bmatrix},
\end{equation}
where $\psi_0(t) = e^{-at}$, $\psi_{k+1}(t) = \int_0^t \psi_k(s) \dd{s}$, and $\begin{bmatrix}
        \xi_{k+1}\\
        \xi_{k+1}'
    \end{bmatrix}$ are $2d$-dimensional centered Gaussian vectors with covariance matrix given by 
\begin{equation*}
    \begin{bmatrix}
        \xi_{k+1}\\
        \xi_{k+1}'
    \end{bmatrix} \sim \gN(0, \mathbf{C}),\quad \mathbf{C} = \int_0^\eta \begin{bmatrix}
        \psi_0(t) \\ \psi_1(t)
    \end{bmatrix} \begin{bmatrix}
        \psi_0(t) & \psi_1(t)
    \end{bmatrix} \dd{t}.
\end{equation*}
The expressions are given in closed form as follows:
\begin{align*}
\psi_0(t) &= e^{-at},\\
\psi_1(t) &= a^{-1}(1- e^{-at}),\\
\psi_2(t) &= a^{-2}(at + e^{-at} - 1),\\
\mathbf{C} &= \begin{bmatrix}
    (2a)^{-1}(1-e^{-2a\eta})  & (2a^2)^{-1} (1 - e^{-a\eta})^2\\
    (2a^2)^{-1} (1 - e^{-a\eta})^2 & (2a^3)^{-1}(2a\eta - e^{-2a\eta} + 4 e^{-a\eta} - 3)\
\end{bmatrix}.
\end{align*}

\end{document}

%% file: math_commands.tex
\usepackage{amsmath,amsfonts,bm}

\def\eqref#1{equation~\ref{#1}}

\def\1{\bm{1}}

\def\rvp{{\mathbf{p}}}

\def\rvx{{\mathbf{x}}}

\DeclareMathAlphabet{\mathsfit}{\encodingdefault}{\sfdefault}{m}{sl}
\SetMathAlphabet{\mathsfit}{bold}{\encodingdefault}{\sfdefault}{bx}{n}
\newcommand{\tens}[1]{\bm{\mathsfit{#1}}}

\def\tP{{\tens{P}}}

\def\tX{{\tens{X}}}

\def\gE{{\mathcal{E}}}
\def\gF{{\mathcal{F}}}

\def\gN{{\mathcal{N}}}

\def\gP{{\mathcal{P}}}

\def\gW{{\mathcal{W}}}

\def\gZ{{\mathcal{Z}}}

\newcommand{\E}{\mathbb{E}}

\newcommand{\R}{\mathbb{R}}

\newcommand{\KL}{\mathrm{D}_{\mathrm{KL}}}

\DeclareMathOperator*{\argmin}{arg\,min}

\DeclareMathOperator{\wprox}{WProx}

\DeclareMathOperator{\diag}{diag}

\DeclareMathOperator{\softmax}{softmax}